%% file: main-arxiv.tex
\documentclass[11pt,table]{article}

\usepackage[left=1in,right=1in,top=1in,bottom=1in]{geometry}
\usepackage[colorlinks=true, urlcolor=black, linkcolor=blue,citecolor=black,pagebackref=true]{hyperref}
\usepackage{url}

\usepackage{booktabs}       
\usepackage{amsfonts}       
\usepackage[nice]{nicefrac} 
\usepackage{microtype}      

\usepackage{textcomp}
\usepackage{xcolor}

\usepackage{amsmath,amsthm,amstext,amsfonts,amssymb,mathrsfs}
\usepackage{graphicx,caption,subcaption,color,algorithm,algorithmic} 

\usepackage{wrapfig,dsfont,booktabs}

\usepackage[capitalize]{cleveref}
\crefname{assumption}{Assumption}{Assumptions}
\usepackage{changepage}
\usepackage{parskip}

\usepackage{fancyhdr}
\input{kishan-math-commands}
\usepackage[round]{natbib}
\renewcommand{\bibname}{References}
\renewcommand{\bibsection}{\subsubsection*{\bibname}}

\title{\Large Model-Free Robust $\phi$-Divergence Reinforcement Learning \\Using Both Offline and Online Data}

\author{Kishan Panaganti,\, Adam Wierman,\, Eric Mazumdar \\
  Computing + Mathematical Sciences Department, California Institute of Technology\\Emails:\texttt{\{kpb,\,adamw,\,mazumdar\}@caltech.edu}
}

\begin{document}

\maketitle

\begin{abstract}
The robust $\phi$-regularized Markov Decision Process (RRMDP) framework focuses on designing control policies that are robust against parameter uncertainties due to mismatches between the simulator (nominal) model and real-world settings. 
This work\,\footnote{To appear in the proceedings of the International Conference on Machine Learning (ICML) 2024.} makes \emph{two} important contributions.
First, we propose a \textit{model-free} algorithm called \textit{Robust $\phi$-regularized fitted Q-iteration} (RPQ) for learning an $\epsilon$-optimal robust policy that uses only the historical data collected by rolling out a behavior policy (with \textit{robust exploratory} requirement) on the nominal model. 
To the best of our knowledge, we provide the \textit{first} unified analysis for a class of $\phi$-divergences achieving robust optimal policies in high-dimensional systems with general function approximation.
Second, we introduce the \textit{hybrid robust $\phi$-regularized reinforcement learning} framework to learn an optimal robust policy using both historical data and online sampling. Towards this framework, we propose a model-free algorithm called \textit{Hybrid robust Total-variation-regularized Q-iteration} (HyTQ: pronounced \textit{height-Q}).
To the best of our knowledge, we provide the \textit{first} improved out-of-data-distribution assumption in large-scale problems with general function approximation under the hybrid robust $\phi$-regularized reinforcement learning framework.
Finally, we provide theoretical guarantees on the performance of the learned policies of our algorithms on systems with arbitrary large state space.
\end{abstract}

\custkw{Robust reinforcement learning, model uncertainty, general function approximation}

\tableofcontents

\input{01-introduction}
\input{03-robust-phi-fqi}
\input{04-FH-robust-tv-fqi}
\input{05-discussion-conclusion}

\bibliography{ref-hyqi} 
\addcontentsline{toc}{section}{\protect\numberline{}References}
\newpage
\input{06-appendix}

\end{document}

%% file: kishan-math-commands.tex
\usepackage{marvosym} 

\newcommand\numthis{\addtocounter{equation}{1}\tag{\theequation}}

\newcommand{\custkw}[1]
{	
  \textbf{{Keywords:}} #1
}

\renewcommand{\phi}{\varphi}
\renewcommand{\epsilon}{\varepsilon}

\newtheorem{lemma}{Lemma}

\newtheorem{assumption}{Assumption}

\newtheorem{proposition}{Proposition}
\newtheorem{theorem}{Theorem}

\newtheorem{corollary}{Corollary}

\newtheorem{remark}{Remark}

\newcommand{\dd}{\mathrm{d}}
\newcommand{\E}{\mathbb{E}}
\newcommand{\D}{\mathbb{D}}

\newcommand{\R}{\mathbb{R}}
\newcommand{\cS}{\mathcal{S}} 
\newcommand{\cA}{\mathcal{A}}
\newcommand{\ScA}{\cS\times\cA}

\newcommand{\cP}{\mathcal{P}}

\newcommand{\cF}{\mathcal{F}}
\newcommand{\cG}{\mathcal{G}}
\newcommand{\cD}{\mathcal{D}}
\newcommand{\cO}{\mathcal{O}}
\newcommand{\cX}{\mathcal{X}}
\newcommand{\cY}{\mathcal{Y}}
\newcommand{\cM}{\mathcal{M}}
\newcommand{\cH}{\mathcal{H}}

\newcommand{\cC}{\mathcal{C}}

\newcommand{\cT}{\mathcal{T}}

\newcommand{\cOtilde}{\widetilde{\cO}}

\newcommand{\cardS}{\ensuremath{|\cS|}}
\newcommand{\cardA}{\ensuremath{|\cA|}}
\newcommand{\cardSA}{\cardS\cardA}

\newcommand{\abs}[1]{\left\lvert #1 \right\rvert}
\newcommand{\tri}[1]{\left\langle #1 \right\rangle}
\newcommand{\indic}{\mathds{1}}

\newcommand{\nn}{\nonumber}

\newcommand{\Vmax}{V_{\max}}

\newcommand{\dphi}{D_{\phi}}

\newcommand{\dtv}{D_{\mathrm{TV}}}

\DeclareMathOperator*{\argmax}{arg\,max}
\DeclareMathOperator*{\argmin}{arg\,min}
\newcommand{\norm}[1]{\left\| #1 \right\|}%

\colorlet{lightgray}{gray!40}

\allowdisplaybreaks

%% file: 01-introduction.tex
\section{Introduction}
\label{sec:introduction}

Online Reinforcement Learning (RL) agents learn through online interactions and exploration in environments and have been shown to perform well in structured domains such as Chess and Go \citep{silver2018general}, fast chip placements in semiconductors \citep{mirhoseini2021graph}, fast transform computations in mathematics \citep{fawzi2022discovering}, and more. However, online RL agents \citep{botvinick2019reinforcement} are known to suffer sample inefficiency due to complex exploration strategies in sophisticated environments. 
To overcome this, learning from available historical data has been studied using offline RL protocols \citep{levine2020offline}. However, offline RL agents suffer from out-of-data-distribution \citep{yang2021generalized, robey2020model} due to the lack of online exploration.
Recent work \citet{song2023hybrid} proposes another learning setting called \textit{hybrid RL} that makes the best of both offline and online RL worlds. In particular, hybrid RL agents have access to both offline data (to reduce exploration overhead) and online interaction with the environment (to mitigate the out-of-data-distribution issue).

All three of these approaches (online, offline, and hybrid RL) require training environments (simulators) that closely represent real-world environments. However, time-varying real-world environments \citep{maraun2016bias}, sensor degradations \citep{chen1996design}, and other  adversarial disturbances in practice \citep{pioch2009adversarial} mean that
even high-fidelity simulators are not enough \citep{schmidt2015depth,shah2018airsim}. 
RL agents are known to fail due to these mismatches between training and testing environments \citep{sunderhauf2018limits,lesort2020continual}. As a result, robust RL \citep{Mankowitz2020Robust,panaganti2020robust} has received increasing attention due to the potential for it to alleviate the issue of mismatches between the simulator and real-world environments.

Robust RL agents are built using the robust Markov Decision Process (RMDP) \citep{iyengar2005robust,nilim2005robust} framework. In this framework, the goal is to find an optimal policy that is robust, i.e., performs uniformly well across a set of models (transition probability functions).  This is formulated via a max-min problem, and the set of models is typically constructed around a simulator model (transition probability function) with some notion of divergence or distance function. We refer to the simulator model as any \textit{nominal model} that is provided to RL agents.

The RMDP framework in RL is identical to the Distributionally Robust Optimization (DRO) framework in supervised learning \citep{duchi2018learning, chen2020distributionally}.
Similar to RMDP, DRO is a min-max problem aiming to minimize a loss function uniformly over the set of distributions constructed around the training distribution of the input space. 
However, developing model-free algorithms for DRO problems with general $\phi$-divergences (see \cref{eq:phi-divergence-defn}) is known to be hard \citep{namkoong2016stochastic} due to their inherent non-linear and multi-level optimization structure.
Additionally, developing model-free robust RL agents is also challenging \citep{iyengar2005robust,duchi2018learning} for high-dimensional sequential decision-making systems under general function approximation.

To overcome this issue, in this work, we develop robust RL agents for the RRMDP framework, which is an equivalent alternative form of RMDP. A natural $\phi$-divergence regularization extension to the problem of RMDP gives way for this new RRMDP framework introduced in \citet{yang2023avoiding,zhang2023regularized}, under different names. It is built upon the penalized DRO problem \citep{levy2020large,jin2021non}, that is, the $\phi$-divergence regularization version of the DRO problem.
In particular, we focus on developing an {\bf offline robust RL algorithm} for a class of $\phi$-divergences under the RRMDP framework with arbitrarily large state spaces, using only offline data with general function approximation. Towards this, as the \emph{first main contribution}, we propose the \textit{Robust $\phi$-regularized fitted Q-iteration} model-free algorithm and provide its performance guarantee for a class of $\phi$-divergences with a unified analysis. We refer to algorithms as \textit{model-free} if they do not explicitly estimate the underlying nominal model.
We address the following important (suboptimality and sample complexity) questions: \textit{What is the rate of suboptimality gap achieved between the optimal robust value and the value of RPQ policy? How many offline data samples from the nominal model are required to learn an $\epsilon$-optimal robust policy?}
We discuss challenges and present these results in \cref{sec:infinite-horizon-offline-robust-rl}.

\begin{table*}[t]
\scriptsize \begin{adjustwidth}{-1em}{}
\begin{center}
    \begin{tabular}{|l|c|c|c|c|c|}
\hline
Algorithm                 & Algorithm-type          & Data Coverage & Dataset Type & Robust & Suboptimality                                           \\ 
\hline
\citep[Alg.1]{panaganti-rfqi}           &  FQI                      & all-policy &  offline  & TV  & $ \frac{ \Vmax^3 \sqrt{\log(|\cF||\cG|)}} {\rho  N^{1/2}} $
\\ \hline
\citep[Alg.1]{zhang2023regularized}           &  FQI                      & all-policy$^*$ &  offline & KL   & $ \frac{ \lambda \Vmax^2 \sqrt{\log(|\cF|)}} { e^{-\Vmax/\lambda} N^{1/2}} $
\\ \hline
\citep[Alg.2]{yang2023avoiding}            &  QL                      & uniform-policy &  offline Markov & $\phi$   & $\frac{ \Vmax^3 \sqrt{\log(\cardS\cardA)}} { 
d_{\min}^3 c(\lambda)  N^{1/3}}$
\\ \hline
\cellcolor{gray!25} RPQ    (\textbf{ours:}\,\cref{alg:infinite-horizon-RPQ-Algorithm})           &  FQI                      & all-policy$^*$ &  offline & $\phi$  & $\frac{ \Vmax^3 \sqrt{\log(|\cF||\cG|)}} { c(\lambda)  N^{1/2}}$
\\ \hline
\cellcolor{gray!25} HyTQ   (\textbf{ours:}\,\cref{alg:HyTQ-Algorithm})$^\dagger$            &  FQI                      & single-policy &  offline +  & TV    & $\frac{ \Vmax (\lambda+\Vmax) \log(|\cF||\cG|)} {  N^{1/2}}$ \\
\cellcolor{gray!25}                &           &  & online non-Markov &  &       
\\ \hline
\end{tabular}
\end{center}
\caption{\emph{Comparison of model-free $\phi$-divergence robust RL algorithms.}
In the \emph{algorithm-type} column, Fitted Q-Iteration (FQI) uses least-squares regression and Q-Learning (QL) uses stochastic approximation updates.
In the  \emph{data coverage} column, 
\textit{uniform-policy} stipulates a data-generating policy to cover the entire state-action space. \textit{all-policy} is where the data-generating policy should cover the state-action space covered by all non-stationary policies, and \textit{single-policy} is where it covers the state-action space covered by the optimal robust policy, on the nominal model. $^*$ denotes the coverage should include all the models in robust sets designed by the divergences in the \textit{robust} column.
The \emph{dataset type} column mentions the type of dataset collected with a data-generating policy for training corresponding algorithms where \emph{offline} indicates i.i.d. historical dataset on the nominal model, \emph{offline Markov} indicates Markovian dataset induced on the nominal model, and \emph{online non-Markov} indicates a history dependent dataset as a collection of Markovian datasets induced on the nominal model by a set of learned policies.
Finally, the suboptimality column is the statistical upper bound for the difference between the optimal robust value and the robust value achieved by the algorithm.
Here $\Vmax$ is either $H$ or $(1-\gamma)^{-1}$ effective horizon factors. 
$\rho$ is the robustness radius parameter in RMDPs and $\lambda$ is the robustness penalization parameter in RRMDPs, which are inversely related \citep[Theorem 3.1]{yang2023avoiding}. $c(\lambda)$ is some function on $\lambda$ that varies according to different $\phi$-divergences.
$N$ is the dataset size used by algorithms. $^\dagger$ The bound of HyTQ is not directly comparable with others in terms of $\Vmax$ since the non-stationary finite-horizon setting requires $H$ multiplicity in dataset size. 
$d_{\min}$  is the minimal positive value of data generating stationary distribution $d$, i.e. $\min_{s,a}d(s,a)$.
$\cF$ and $\cG$ are two function representations, and $(\cS,\cA)$ is the state-action space.} 
\label{table:comparison-table}
\end{adjustwidth}
\end{table*}

In this work, we also develop and study a novel {\bf hybrid robust RL algorithm} under the RRMDP framework using both offline data and online interactions with the nominal model. 
We make this \emph{second main contribution} to this work since hybrid RL overcomes the out-of-data-distribution issue in offline RL.
Towards this, we propose the \textit{Hybrid robust Total-variation-regularized Q-iteration} algorithm and provide its performance guarantee under improved assumptions.
Notably, the offline data-generating distribution must only cover the distribution that the optimal robust policy samples out on the nominal model, whereas before we needed it to cover any distribution uniformly. This is how online interactions help mitigate the out-of-data-distribution issue of offline RL and offline robust RL.
We now address the cumulative suboptimality question in addition to sample complexity: \textit{What is the rate of cumulative suboptimality gap achieved between the optimal robust value and the value of HyTQ iteration policies?}
We discuss challenges and present these results in \cref{sec:finite-horizon-hybrid-robust-rl}.

\textbf{Related Work.} \label{intro-related-works} Among all the previous works that provide model-free methods, here we only mention the ones closest to ours. We discuss more related works in \cref{appen:related works}. 
\citet{panaganti-rfqi} proposed a Q-iteration offline robust RL algorithm in the RMDP framework only for the total variation $\phi$-divergence.
\citet{bruns2023robust} proposed a Q-iteration offline robust RL algorithm in the RMDP framework to solve causal inference under unobserved confounders.
\citet{zhou2023natural} proposed an actor-critic robust RL algorithm in RMDP for integral probability metric.
\citet{zhang2023regularized} proposed a Q-iteration offline robust RL algorithm in the RRMDP framework only for the Kullback-Leibler $\phi$-divergence.
\citet{blanchet2023double} proposed specialized robust RL algorithms for the total variation and Kullback-Leibler $\phi$-divergences offering unified analyses for  linear, kernels, and factored function approximation models under the finite state-action setting.
Other line of work \citep{liu2022distributionally,liang2023single,wang2023finite,wang2023sample,yang2023avoiding} provide model-free robust RL algorithms based on classical Q-learning methods in finite state-action spaces.
We provide more insightful comparisons in \cref{table:comparison-table}. \textit{To the best of our knowledge, this is the first work that addresses a wide class of robust RL problems (like the general $\phi$-divergence) with arbitrary large state space using general function approximation  under mild assumptions (like the robust Bellman error transfer coefficient).}

\textbf{Notation.} 
We use the equality sign (=) for pointwise equality in vectors and matrices. For any $x \in \R$, let $(x)_+ = \max\{x, 0\}$. 
For any vector $x$ and positive semidefinite matrix $A$, the squared matrix norm is $\|x\|_{A}^2 = x^{\top} A x$.
The set of probability distributions over $\cX$, with cardinality $|\cX|$, is denoted as $\Delta(\cX)$, and its power set sigma algebra as $\Sigma(\cX)$.
For any function $f$ that takes $(s,a,r,s')$ as input, define the expectation w.r.t. the dataset $\cD$ (or empirical expectation) as $\E_\cD[f(s_i,a_i,r_i,s'_i)] = \frac{1}{N} \sum_{(s_i,a_i,r_i,s'_i) \in \cD} f(s_i,a_i,r_i,s'_i)$. 
For any positive integer $H$, set $[H]$ denotes $\{0,1,\cdots,H-1\}$.
Define $\ell_2$ and $\ell_1$ norms as $\left\| x \right\|_{2,\mu} = \sqrt{\E_{\mu}[x^2]}$ and $\left\| x \right\|_{1,\mu} = \E_{\mu}[|x|]$.
$p\ll q$ denotes a probability distribution $p$ is absolutely continuous w.r.t a probability distribution $q$.
We use $\cO(\cdot)$ to ignore universal constants less than $300$ and $\cOtilde(\cdot)$ to ignore universal constants less than $300$ and the polylog terms depending on problem parameters.

%% file: 03-robust-phi-fqi.tex
\section{Offline Robust $\phi$-Regularized Reinforcement Learning}
\label{sec:infinite-horizon-offline-robust-rl}

We start with preliminaries and the problem formulation.

\textbf{Infinite-Horizon Markov Decision Process:}
An infinite-horizon discounted Markov Decision Process ($\gamma$MDP) is a tuple $(\cS,\cA,R,P,\gamma,d_0)$ where $\cS$ is a countably large state-space, $\cA$ is a finite set of actions, $R:\ScA\to[0,1]$ is a known stochastic reward function, $P\in\Delta(\cS)^{\cardSA}$ is a probability transition function describing an environment, $\gamma$ is a discount factor, and $d_0$ is the starting state distribution.
A stationary (stochastic) policy $\pi: \cS \to \Delta(\cA)$ specifies a distribution over actions in each state. We denote the transition dynamic distribution at state-action $(s,a)$ as $P_{s,a}\in\Delta(\cS)$. 
For convenience, we write $r(s,a)=\E_{r\sim R(s,a)}[r]$ and assume it is deterministic as in RL literature \citep{agarwal2019reinforcement} since the performance guarantee will be identical up to a constant factor. 

The value function of a policy $\pi$ is $V^\pi_{P,r}(s) = \E_{P,\pi} [\sum_{t=0}^\infty \gamma^t r(s_t, a_t) \; | \; s_0 = s ]$ starting at state $s_0 = s$ and $a_t \sim \pi(s_t), s_{t+1} \sim P_{s_t, a_t}$ for all $t\geq 0$. Similarly, we define an action-value function of a policy $\pi$ as
$Q^\pi_{P,r}(s,a) = \E_{P,\pi} [\sum_{t=0}^\infty \gamma^t r(s_t, a_t) \; | \; s_0 = s , a_0 = a ].$
Each policy $\pi$ induces a discounted occupancy density over state-action pairs $d^{\pi}_{P}: \ScA \to [0,1]$ defined as $d^{\pi}_{P}(s,a) = (1-\gamma)\sum_{t=0}^\infty \gamma^t P_t(s_t = s, a_t = a; \pi)$, where $P_t(s_t = s, a_t = a; \pi)$ denotes the visitation probability of state-action pair $(s,a)$ at time step $t$, starting at $s_0 \sim d_0(\cdot)$ and following $\pi$ on the model $P$. The optimal policy $\pi^*_{P}$ achieves the maximum value of any policy $ V^\pi_{P,r}$.

\textbf{Offline Reinforcement Learning:}
The goal of offline RL on $\gamma$MDP $(P^o,r)$ is to learn a \textit{good} policy $\hat{\pi}$ (a policy with a high $V^{\hat{\pi}}_{P^o,r}$) based only on the \textit{offline dataset}. 
An offline dataset is a historical and fixed dataset of interactions $\cD_{P^o} = \{(s_i,a_i,s_i')\}_{i=1}^N$, where $s_i' \sim P^o_{s_i,a_i}$ and the $(s_i,a_i)$ pairs are independently and identically generated  according to a data distribution $\mu \in \Delta(\ScA)$. For convenience, $\mu$ also denotes the \textit{offline/behavior policy} that generates $\cD_{P^o}$. 
One classical offline  RL algorithm with general function approximation capabilities with provable performance guarantees is \textit{Fitted Q-Iteration (FQI)} \citep{szepesvari2005finite, chen2019information, liu2020provably}. A function class $\cF=\{f:\ScA\to [0, 1/(1-\gamma)] \}$ (e.g., neural networks, kernel functions, linear functions, etc) represents $Q$-value functions of $\gamma$MDP $(P^o,r)$. At each iteration, given $f_{k}\in\cF$ and $\cD_{P^o}$, FQI does the following least-square regression for the approximate squared Bellman error: $ f_{k+1} = \argmin_{f \in \cF}  \E_{\cD_{P^o}}  [(y_{f_{k}} - f )^{2}]$, where $y_{f_{k}}(s, a, s')= r(s, a) + \gamma \max_{b} f_{k}(s', b)$.
In this regression step, FQI aims to find the optimal action-value $Q^{\pi^*}_{P^o,r}$ by approximating the  non-robust squared Bellman error ($\|r + \gamma \E_{P^o} V^{\pi^*}_{P^o,r}(\cdot)-Q^{\pi^*}_{P^o,r}\|_{2,\mu}^2$) using offline data $\cD_{P^o}$ with function approximation $\cF$.
Finally, for some starting state $s_0\sim d_0$, the performance guarantee of an algorithm policy $\hat{\pi}$ is given by bounding the \textit{suboptimality} quantity $0 \leq V^{\pi^*}_{P^o,r}(s_0) - V^{\hat{\pi}}_{P^o,r}(s_0)$.

\textbf{Infinite-Horizon Robust $\phi$-Regularized Markov Decision Process:} Let $P^o$ be the nominal model, that is, a probability transition function describing a training environment.  An infinite-horizon discounted Robust $\phi$-Regularized Markov Decision Process ($\gamma$RRMDP) tuple $(\cS,\cA,r,P^o,\lambda,\gamma,\phi,d_0)$ where $\lambda>0$ is a robustness parameter and $\phi:\R\to\R$ is a convex function. The \textit{robust regularized reward function} is defined as $r^\lambda_{P}(s,a)=r(s,a)+\lambda \gamma \dphi(P_{s,a},P^o_{s,a})$ for any state-action pairs and any $P$ such that $P_{s,a},P^o_{s,a}$. Here $\dphi$ is the $\phi$-divergence \citep{csiszar1967information} defined  as \begin{align}
    \label{eq:phi-divergence-defn}
    \dphi(p,q)= \int \phi\left(\frac{\dd p}{\dd q}\right)\dd q
\end{align}
 for two probability distributions $p$ and $q$ with $p\ll q$,
where $\phi$ is convex on $\R$ and differentiable on $\R_+$ satisfying $\phi(1)=0$ and $\phi(t)=+\infty$ for $t<0$. Examples of $\phi$-divergence include Total Variation (TV), Kullback-Leibler (KL), chi-square, Conditional Value at Risk (CVaR), and more (c.f. \cref{prop:divergence-loss-bounds}). 
The \textit{robust regularized value function} of a policy $\pi$ is defined as 
\begin{align}
    \label{eq:robust-regularized-value} V^{\pi}_{\lambda}=\inf_{P\in\cP}V^\pi_{P,r^\lambda_{P}},
\end{align} where $\cP=\otimes_{s,a}\cP_{s,a}$ and $\cP_{s,a}=\{P_{s,a}\in\Delta(\cS):P_{s,a}\ll P^o_{s,a}, \forall (s,a)\in\ScA\}$. By definition, for any $\pi$, it follows that $V^{\pi}_{\lambda}\leq V^\pi_{P^o,r}\leq 1/(1-\gamma)$.
The \textit{optimal robust regularized value function} is $V^{*}_{\lambda}=\max_{\pi} V^\pi_{\lambda}$ (similarly we can design $Q^{*}_{\lambda}$), and $\pi^*$ is the \textit{robust regularized optimal policy} that achieves this optimal value.
For convenience, we denote $V^{*}_{\lambda}$($Q^{*}_{\lambda}$) as $V^*$($Q^*$).
We note that $\cP$ satisfies the $(s,a)$-rectangularity condition \citep{iyengar2005robust} by definition. This is a sufficient condition for the optimization problem in \eqref{eq:robust-regularized-value} to be tractable. It also enables the existence of a \textit{deterministic policy} for $\pi^*$ \citep{yang2023avoiding}. We formally mention this in \cref{prop:infinite-horizon-opt-deterministic-policy}.
For any policy $\pi$, denote $V^{\pi}=\E_{s\sim d_0}[V^{\pi}(s)]$ as the expected total reward with $d_0$ as initial state distribution.

Denote the robust regularized Bellman operator $\cT: \R^{\ScA}\to \R^{\ScA}$ as 
 \begin{align}
 \label{eq:robust-regularized-bellman-eq-primal}
     &(\cT Q)(s,a) = r(s,a) +   \gamma \inf_{P_{s,a}\in\cP_{s,a}}\big(\E_{s'\sim P_{s,a}}[\max_{a'}Q(s',a')] + \lambda \dphi(P_{s,a},P^o_{s,a})\big).
 \end{align}
Since $\cT$ is a contraction \citep{yang2023avoiding}, the \textit{robust Q-iteration} (RQI)  $Q_{k+1} = \cT Q_{k}$ converges to $Q^*$. We get the robust optimal policy as $\pi^*(s)=\argmax_{a}Q^*(s,a)$.

\subsection{Problem Conceptualization}

In this section, we study the offline infinite-horizon robust $\phi$-regularized RL ($\gamma$R$^3$L) problem, acquiring useful insights to construct our algorithm (\cref{alg:infinite-horizon-RPQ-Algorithm}) in next section. 
The goal here is to learn a \textit{good} robust policy $\hat{\pi}$ (a policy with a high $V^{\hat{\pi}}_{\lambda}$) based on the offline dataset. 
We start by noting one key challenge in the estimation of the robust regularized Bellman operator $\cT$ \eqref{eq:robust-regularized-bellman-eq-primal}: One may require many offline datasets from each $P\in\cP$ to achieve our offline $\gamma$R$^3$L goal.
In this work, we use the penalized Distributionally Robust Optimization (DRO) tool \citep{sinha2017certifiable,levy2020large,jin2021non} to not require such unrealistic existence of offline datasets. In particular, as in non-robust offline RL, we only rely on the offline dataset $\cD_{P^o}$  generated on the nominal model $P^o$ by an offline policy $\mu$. 
This statement is justified via the following proposition.

\begin{proposition}
\label{prop:robust-bellman-dual}
Consider a robust $\phi$-regularized MDP. For any $Q: \ScA\to [0, 1/(1-\gamma)]$, the robust regularized Bellman operator $\cT$ \eqref{eq:robust-regularized-bellman-eq-primal} can be equivalently written as
\begin{align}
\label{eq:robust-regularized-bellman-eq-dual}
    (\cT Q)(s, a) &= r(s, a) - \gamma \inf_{\eta\in \Theta} (\lambda \E_{s' \sim P^{o}_{s,a}}[\phi^*\left({(\eta-V(s'))}/{\lambda}\right)]  - \eta ), 
\end{align} 
where $V(s)=\max_{a\in\cA} Q(s,a)$ and $\Theta\subset\R$ is some bounded real line which depends on $\phi^*$. 
\end{proposition}
A proof of this proposition is given in \cref{appendix:infinite-horizon} and follows from \citet[Section A.1.2]{levy2020large}. We refer to \eqref{eq:robust-regularized-bellman-eq-dual} as the \textit{robust regularized Bellman dual operator}. Observing the sole dependence on the nominal model $P^o$ in \eqref{eq:robust-regularized-bellman-eq-dual}, one can come up with estimators for data-driven approaches that naturally depend only on the dataset $\cD_{P^o}$.
We remark that we consider a class of $\phi$-divergences satisfying the conditions in \cref{prop:divergence-loss-bounds} for all the results in this paper.

We now remark on a natural first attempt at performing the squared Bellman error least-square regression, like FQI, on the robust regularized Bellman dual operator \eqref{eq:robust-regularized-bellman-eq-dual}. Observe that the true Bellman error $\E_{s,a\sim\mu}[|\cT Q^{*}(s,a)-Q^{*}(s,a)|]$ involves solving an inner convex minimization problem in $\cT Q^{*}(s,a)$ \eqref{eq:robust-regularized-bellman-eq-dual} for every $(s,a)$. Since we are in a countably large state space regime, it is infeasible to devise approximations to this true squared Bellman error. In addition, we have to also enable general function architecture for action-values.
To alleviate this challenging task, we now turn our attention to the inner convex minimization problem in the robust regularized Bellman dual operator \eqref{eq:robust-regularized-bellman-eq-dual}. Due to the $(s,a)$-rectangularity assumption, we note that the $\eta$'s  are not correlated across all $(s,a)$. With this note, for every $(s,a)$, we can replace $\eta$ in $(\cT Q)(s, a)$ \eqref{eq:robust-regularized-bellman-eq-dual} with a \textit{dual-variable function} $g(s,a)$. Thus, intuitively, multiple point-wise minimizations can be replaced by a single dual-variable functional minimization over the function space of $g$. We formalize this intuition using \textit{variational functional analysis} \citep{rockafellar2009variational} for a countably large state space regime in the following.

We denote $L^1(\mu)$ as the set of all absolutely integrable functions defined on the probability (measure) space $(\ScA,\Sigma(\ScA),\mu)$ with $\mu$, the data generating distribution, as the $\sigma$-finite probability measure. To elucidate, $L^1(\mu)$ is the set of all functions $g : \ScA  \rightarrow  \cC \subset \R$ such that $\norm{g}_{1,\mu}$ is finite. We set $\cC=\Theta$ considering the inner minimization in  \eqref{eq:robust-regularized-bellman-eq-dual}.
Fixing any given function $f : \ScA \rightarrow [0,1/(1-\gamma)]$, we define the loss function $L_{\mathrm{dual}}(g; f)$, for all $g\in L^1(\mu)$, as
\begin{align}
    \label{L-f-eta-loss}
    L_{\mathrm{dual}}(g; f,\mu)&=\E_{s,a\sim\mu,s'\sim P^o_{s,a}} [  \lambda \phi^*((g(s,a)-\max_{a'} f(s',a'))/{\lambda})  - g(s,a)]. 
\end{align}
We state the result for single dual-variable functional minimization intuition we developed in the previous paragraph.  We also note one variant of this result appears in the distributionally robust RL work \citep{panaganti-rfqi}.
\begin{proposition}
\label{prop:functional-opt-phi-divergence}
Let $L_{\mathrm{dual}}$ be the loss function defined in \eqref{L-f-eta-loss}. Then, for any function $f : \ScA \rightarrow [0,1/(1-\gamma)]$, we have
\begin{align}
\label{eq:functional-opt-phi-divergence}
    &\inf_{g \in L^1(\mu)}  L_{\mathrm{dual}}(g; f,\mu)  = \E_{s,a\sim\mu} \Big[\inf_{\eta\in \Theta} (\lambda \E_{s' \sim P^{o}_{s,a}}[\phi^*({(\eta-\max_{a'}f(s',a'))}/{\lambda})]  - \eta )\Big]. 
\end{align}
\end{proposition}
We provide a proof in \cref{appendix:infinite-horizon}, which relies on \citet[Theorem 14.60]{rockafellar2009variational}.

For any given  $f: \ScA \rightarrow [0,1/(1-\gamma)]$ and $(s,a)\in\ScA$, we define an operator $\cT_g$, for all $g\in L^1(\mu)$, as 
\begin{align}
\label{eq:Tg}
    &(\cT_g f) (s, a) = r(s, a) - \gamma (\lambda \E_{s' \sim P^{o}_{s,a}}[\phi^*\left({(g(s,a)-V(s'))}/{\lambda}\right)]  - g(s,a)  ). 
\end{align}  
This operator is useful in view of \cref{prop:robust-bellman-dual,prop:functional-opt-phi-divergence}. To see this, we first define $g^*(Q) \in \argmin_{g \in L^1(\mu)}  L_{\mathrm{dual}}(g; Q, \mu)$ for any action-value function $Q$. Now, by taking an expectation w.r.t  the data generating distribution $\mu$ on \eqref{eq:robust-regularized-bellman-eq-dual}, we observe $\cT Q = \cT_{g^*(Q)} Q$ by   utilizing \eqref{eq:functional-opt-phi-divergence}. Due to this observation, in the following subsection, we develop an algorithm by approximating both the optimal dual-variable function of optimal robust value $g^*(Q^*)$ and the robust squared Bellman error ($\|\cT_{g^*(Q^*)} Q^*-Q^{*}\|_{2,\mu}^2$) using offline data $\cD_{P^o}$. 
\citet{panaganti-rfqi} similarly conceptualized their total variation $\phi$-divergence robust RL algorithm. Here, \cref{prop:robust-bellman-dual} enables us to conceptualize for general $\phi$-divergence.

\subsection{Robust $\phi$-regularized fitted Q-iteration}

In this section, we formally propose our algorithm based on the tools developed so far. Our proposed algorithm is called Robust $\phi$-regularized fitted Q-iteration (RPQ) Algorithm and is summarized in \cref{alg:infinite-horizon-RPQ-Algorithm}. We first discuss the inputs to our algorithm. As mentioned above, we only use the offline dataset $\cD_{P^o} = \{(s_i,a_i,s_i')\}_{i=1}^N$, generated  according to a data distribution $\mu$ on the nominal model $P^o$. We also consider two general function classes $\cF\subset(f:\ScA \rightarrow [0,1/(1-\gamma)])$ and $\cG\subset(g:\ScA \rightarrow \Theta)$ representing action-value functions and dual-variable functions, respectively.
We now define useful approximation quantities for $g\in\cG$ and $f\in\cF$. For given $f$, the empirical loss function of the true loss $L_{\mathrm{dual}}$ \cref{L-f-eta-loss} on $\cD_{P^o}$ is
\begin{align} \label{eq:L-dual-empirical-loss}
    \widehat{L}_{\mathrm{dual}}&(g;f) =\E_{\cD_{P^o}} [  \lambda \phi^*((g(s_i,a_i)-\max_{a'} f(s_i',a'))/{\lambda})  - g(s_i,a_i)]. 
\end{align}
For given $f,g$, the empirical squared robust regularized Bellman error  on $\cD_{P^o}$ is
\begin{align} \label{eq:L-rob-Q-empirical-loss}
    &\widehat{L}_{\mathrm{robQ}}(Q;f, g) =\E_{\cD_{P^o}} [( [ r(s_i, a_i)  - \gamma\lambda \phi^*((g(s_i,a_i)-\max_{a'} f(s_i',a'))/{\lambda})  + \gamma g(s_i,a_i) - Q(s_i,a_i) )^{2}]. 
\end{align}

We start with an initial action-value function $Q_{0}(s,a)= 0$ and execute the following two steps for $K$ iterations.
At iteration $k$ of the algorithm with input $Q_k$, as a first step, we  compute a dual-variable function $g_k\in\cG$ through the \textit{empirical risk minimization} approach, that is, we solve  $\argmin_{g \in \cG} \widehat{L}_{\mathrm{dual}}(g;Q_k)$ \textit{(Line 4 of \cref{alg:infinite-horizon-RPQ-Algorithm})}. 
As a second step, given inputs $Q_k$ and $g_k$, we  compute the next iterate $Q_{k+1}\in\cF$ through the \textit{least-squares regression} method, that is, we solve  $\argmin_{f \in \cF} \widehat{L}_{\mathrm{robQ}}(f;Q_k, g_k) $ \textit{(Line 5 of \cref{alg:infinite-horizon-RPQ-Algorithm})}.
After $K$ iterations, we extract the greedy policy from $Q_K$ \textit{(Line 7 of \cref{alg:infinite-horizon-RPQ-Algorithm})}.

\begin{algorithm}[h]
	\caption{Robust $\phi$-regularized fitted Q-iteration (RPQ) Algorithm}	
	\label{alg:infinite-horizon-RPQ-Algorithm}
	\begin{algorithmic}[1]
		\STATE \textbf{Input:} Regularization $\phi$, offline dataset $\cD_{P^o} =(s_i,a_i,r_{i}, s'_{i})_{i=1}^N$, general function classes $\cF$ and $\cG$
		\STATE \textbf{Initialize:} $Q_{0}\equiv 0 \in \cF$.
		\FOR {$k=0,\cdots,K-1$ } 
        \STATE \textbf{Dual variable function minimization:} $g_{k} = \widehat{g}_{Q_{k}} = \argmin_{g \in \cG} \widehat{L}_{\mathrm{dual}}(g; Q_{k})\;$ (c.f. \eqref{eq:L-dual-empirical-loss})
        \STATE \textbf{Robust $\phi$-regularized Q-update:} $Q_{k+1}=\argmin_{Q\in\cF}\widehat{L}_{\mathrm{robQ}}(Q;Q_{k}, g_{k})\;$ (c.f. \eqref{eq:L-rob-Q-empirical-loss})
		\ENDFOR
		
		\STATE \textbf{Output:} $\pi_{K}  = \argmax_a Q_{K}(s,a)$
	\end{algorithmic}
\end{algorithm}

\subsection{Performance Guarantee: Suboptimality}

We now discuss the performance guarantee of our RPQ Algorithm. In particular, we characterize how close the robust regularized value function of our RPQ Algorithm is to the optimal robust regularized value function. We first mention all the assumptions about the data generating distribution $\mu$ and  the representation power of $\cF$ and $\cG$ before we present our main results.

\begin{assumption}[Concentrability] 
\label{assum-concentra-condition}
There exists a finite constant $C>0$ such that for any $\nu\in \{ d_{\pi,P} ~|$ any policy $\pi$ and  $P \in \cP$ satisfying $\dphi(P_{s,a},P^o_{s,a})\leq 1/(\lambda(1-\gamma))$ for all  $s,a$ (both can be non-stationary)$\} \subseteq \Delta(\ScA)$, we have $\norm{\nu/\mu}_{\infty} \leq \sqrt{C}$.
\end{assumption}

Assumption \ref{assum-concentra-condition} stipulates the support set of the data generating distribution $\mu$, i.e. $\{(s,a)\in\ScA:\mu(s,a)>0\}$, to cover the union of all support sets of the distributions $\nu$, leading to a \textit{robust exploratory} behavior. 
This assumption is widely used in the offline RL literature \citep{munos2003error,agarwal2019reinforcement,chen2019information,wang2021what,xie2021bellman} in different forms. We adapt this assumption from the robust offline RL \citep{panaganti-rfqi,zhang2023regularized}.

\begin{assumption}[Approximate Robust Bellman Completeness] 
\label{assum-bellman-completion}
Let $\epsilon_{\cF}$ be some small positive constant. For any $g\in\cG$, we have $ \sup_{f\in\cF} \inf_{f'\in\cF} \| f' - \cT_g f \|_{2,\mu}^2 \leq \epsilon_{\cF}$ for the data generating  distribution $\mu$.
\end{assumption}
We note that \cref{assum-bellman-completion} holds trivially if $\cT_g$ is closed under $\cF$, that is, for any $f\in\cF$ and $g\in\cG$, if it holds that $\cT_g f\in\cF$, then $\epsilon_{\cF}=0$. 
This assumption has been widely used in different forms in the non-robust offline RL literature~\citep{agarwal2019reinforcement,wang2021what,xie2021bellman} and robust offline RL literature~\citep{panaganti-rfqi,bruns2023robust,zhang2023regularized}.

\begin{assumption}[Approximate Dual Realizability] \label{assum-dual-realizability}
For all $f \in \cF$, there exists a uniform constant $\epsilon_{\cG}$ such that $ \inf_{g \in \cG}  L_{\mathrm{dual}}(g; f) - \inf_{g \in L^1(\mu)}   L_{\mathrm{dual}}(g; f) \leq \epsilon_{\cG}$.
\end{assumption}
\cref{assum-dual-realizability} holds trivially if $g^*(f)\in\cG$ for any $f\in\cF$ (since $\epsilon_{\cG}=0$). This assumption has been used in earlier robust offline RL literature~\citep{panaganti-rfqi,bruns2023robust}.

Now we state our main theoretical result on the performance of the RPQ algorithm. In \cref{appendix:infinite-horizon} we restate the result including the constant factors.
\begin{theorem}
\label{thm:infinite-horizon-phi-divergence-guarantee}
Let \cref{assum-concentra-condition,assum-bellman-completion,assum-dual-realizability} hold. 
Let $c_{\phi}(\lambda,\gamma)$ be problem-dependent constants for $\phi$. Let $\pi_{K}$ be the RPQ algorithm policy after $K$ iterations. Then, for any $\delta\in(0,1)$, with probability at least  $1 - \delta$,  we have 
\begin{align*}
	V^{\pi^*} - V^{\pi_{K}} \leq &\frac{\sqrt{C}(\gamma^K+\sqrt{6\epsilon_{\cF}} + \gamma \epsilon_{\cG}) }{(1-\gamma)^2}  +   \frac{c_{\phi}(\lambda,\gamma)}{(1-\gamma)^3} \cO(\sqrt{{C \log(|\cF||\cG|/\delta)}/{N}}). \end{align*}
\end{theorem}

\cref{thm:infinite-horizon-phi-divergence-guarantee} states that the RPQ algorithm is approximately optimal. 
This theorem also gives the sample complexity guarantee for finding an $\epsilon$-suboptimal policy w.r.t. the optimal policy $\pi^*$. 
To see this, by neglecting the first term due to inevitable function class approximation errors, for $N\geq\cO(\frac{(c_{\phi}(\lambda,\gamma))^2}{\epsilon^{2} (1-\gamma)^{4}} \log \frac{|\cF| |\cG|}{\delta })$ we get $V^{\pi^*} - V^{\pi_{K}} \leq {\epsilon}/{(1-\gamma)}$ with probability at least $1-\delta$ for any fixed $\epsilon, \delta \in (0,1)$.

\begin{remark}\label{rem:tv-offline-guarantee}
Note that the guarantee for the TV case in \cref{thm:infinite-horizon-phi-divergence-guarantee}  requires making another assumption on the existence of a \textit{fail-state} \citep[Lemma 3]{panaganti-rfqi}, \cref{fail-state-assump} replacing $H$ with $1/(1-\gamma)$.
However, we specialize \cref{thm:infinite-horizon-phi-divergence-guarantee} for the TV case by relaxing \cref{assum-concentra-condition} to get the same guarantee, which we present in \cref{appendix:infinite-horizon}. In particular, we relax \cref{assum-concentra-condition} to the non-robust offline RL concentrability assumption \citep{foster2022offline}, i.e. we only need the distribution $\nu$ to be in the collection of discounted state-action occupancies on the nominal model $P^o$. 
\end{remark}

%% file: 04-FH-robust-tv-fqi.tex
\section{Hybrid Robust $\phi$-Regularized Reinforcement Learning}
\label{sec:finite-horizon-hybrid-robust-rl}

In this section, we provide a \textit{hybrid} robust $\phi$-Regularized RL protocol to overcome the out-of-data-distribution issue in offline robust RL.
As in \citet{song2023hybrid}, we reformulate the problem in the finite-horizon setting to use its backward induction feature that enables RPQ iterates to run in each episode.
We again start by discussing preliminaries and the problem formulation.

\textbf{Finite-Horizon Markov Decision Process:}
A finite-horizon Markov Decision Process ($h$MDP) is
$(\cS,\cA,P=(P_h)_{h=0}^{H-1},r=(r_h)_{h=0}^{H-1}, {H})$, where $H$ is the horizon length,  for any $h \in [H]$, $r_h:\ScA\to[0,1]$ is a known deterministic reward function and $P_h\in\Delta(\cS)^{\cardSA}$ is the transition probability function at time $h$.
A non-stationary (stochastic) policy $\pi = (\pi_h)_{h=0}^{H-1}$ where $\pi_h:\cS \to \Delta(\cA)$. We denote the transition dynamic distribution at time $h$ and state-action $(s,a)$ as $P_{h,s,a}\in\Delta(\cS)$.
Given $\pi$, we define the state and action value functions in the usual manner: $V^{h,\pi}_{P,r}(s) = \E[ \sum_{t = h}^{H-1} r_t(s_t,a_t) | s_h = s ]$ starting at state $s_h = s$ and $a_t \sim \pi_t(s_t), s_{t+1} \sim P_{t+1,s_t, a_t}$, and $Q^{h,\pi}_{P,r}(s,a) = \E[ \sum_{t = h}^{H-1} r_t(s_t,a_t) | s_h = s, a_h=a]$ starting at state-action $s_h = s, a_h=a$ and $s_{t+1} \sim P_{t+1,s_t, a_t}, a_{t+1} \sim \pi_{t+1}(s_{t+1})$. 
Given $\pi$, occupancy measure over state-action pairs $d^{h,\pi}_{P}(s,a) = P_h(s_h = s, a_h = a; \pi)$.
We write $\pi^*_P = (\pi^*_h)_{h=0}^{H-1}$ to denote an optimal non-stationary deterministic policy, which maximizes $V^{\pi}_{P,r} = (V^{h,\pi}_{P,r})_{h=0}^{H-1}$.

\textbf{Hybrid Reinforcement Learning:}
The goal of hybrid RL on $h$MDP $(P^o,r)$ is to learn a good policy $\hat{\pi}$ based on adaptive datasets consisting of \textit{both} offline datasets and on-policy datasets. 
Given timestep $h\in[H]$, offline dataset $\cD^{\mu}_{h,P^o} = \{(s_i,a_i,s_i')_{i=1}^{m_{\mathrm{off}}}\}$ is generated by $s_i' \sim P^o_{h,s_i,a_i}$ with the $(s_i,a_i)$ pairs i.i.d. sampled by $\mu_{h} \in \Delta(\ScA)$ offline data distribution. For convenience, $\mu=(\mu_{h})_{h=0}^{H-1}$ also denotes the \textit{offline policy} that generates $\cD^{\mu}_{P^o}$.
Given timestep $h\in[H]$, on-policy dataset $\cD^{\pi}_{h,P^o} = \{(s_i,a_i,s_i')_{i=1}^{m_{\mathrm{on}}}\}$ is generated by $(s_i,a_i)\sim d^{h,\pi}_{P^o}$ and $s_i' \sim P^o_{h,s_i,a_i}$ for all the previously learned policies $\pi$ by the algorithm.
\citet{song2023hybrid} proposes \textit{Hybrid Q-learning} (HyQ) algorithm with general function approximation capabilities and provable guarantees for hybrid RL. 
The HyQ algorithm (c.f. \citet[Algorithm 1]{song2023hybrid}) is quite straightforward: For each iteration $k\in[K]$, do backward induction of the FQI algorithm on timesteps $h\in[H]$ using the adaptive datasets described above.
Finally, for some starting state $s_0\sim d_0$, the performance guarantee of algorithm policies $\{\pi_k\}_{k\in[K]}$ is given by bounding the \textit{cumulative suboptimality} quantity $0 \leq \sum_{k=[K]}[V^{0,\pi^*}_{P^o,r}(s_0) - V^{0,\pi_k}_{P^o,r}(s_0)]$.
We note the total adaptive dataset size is $N$ to provide comparable results with offline RL.

\textbf{Finite-Horizon Robust $\phi$-Regularized Markov Decision Process:} Again, let $P^o$ be the nominal model.  A finite-horizon discounted Robust $\phi$-Regularized Markov Decision Process ($h$RRMDP) tuple $(\cS,\cA,P^o=(P^o_h)_{h=0}^{H-1},r=(r_h)_{h=0}^{H-1},\lambda,H,\phi,d_0)$ where $\lambda>0$ is a robustness parameter and $\phi:\R\to\R$ is as before. For $h\in[H]$, the \textit{robust regularized reward function} is $r^\lambda_{h}(s,a)=r_h(s,a)+\lambda \dphi(P_{h,s,a},P^o_{h,s,a})$.
For $h\in[H]$, the \textit{robust regularized value function} of a policy $\pi$ is defined as $V^{\pi}_{h,\lambda}=\inf_{P\in\cP}V^{h,\pi}_{P,r^\lambda_{h}},$
where $\cP=\otimes_{h,s,a}\cP_{h,s,a}$ and $\cP_{h,s,a}=\{P_{h,s,a}\in\Delta(\cS):P_{h,s,a}\ll P^o_{h,s,a}, \forall (s,a)\in\ScA \text{ and } h\in[H]\}$. 
By definition, for any $\pi$, it follows that $V^{\pi}_{h,\lambda}\leq V^{h,\pi}_{P^o,r}\leq H$.
For $h\in[H]$, the \textit{optimal robust regularized value function} is $V^{*}_{h,\lambda}=\max_{\pi} V^{\pi}_{h,\lambda}$, and $\pi^*$ is the \textit{robust regularized optimal policy} that achieves this optimal value.
For convenience, we denote $V^{*}_{h,\lambda}$($Q^{*}_{h,\lambda}$) as $V^*_{h}$($Q^*_{h}$) for all $h\in[H]$.
We again note that, for each $h\in[H]$, $\cP$ satisfies the $(s,a)$-rectangularity condition \citep{iyengar2005robust} by definition. It enables the existence of a \textit{non-stationary deterministic policy} for $\pi^*$ \citep{zhang2023regularized}. We formalize this in \cref{prop:FH-opt-deterministic-policy}.
We denote $V^\pi = \E_{s\sim d_0} [V^{\pi}_0(s)]$ as the expected total reward. 

For convenience, we let $Q_{H,\lambda}^{\pi}=0$ for any $\pi$. For any $h\in[H]$, denote the robust regularized Bellman operator $\cT: \R^{\ScA}\to \R^{\ScA}$ as 
 \begin{align}
 \label{eq:FH-robust-regularized-bellman-eq-primal}
     &(\cT Q_{h+1}) (s, a) = r_{h}(s, a) +   \inf_{P_{h,s,a}\in\cP_{h,s,a}}\big(\E_{s'\sim P_{h,s,a}}[\max_{a'}Q_{h+1}(s',a')] + \lambda \dphi(P_{h,s,a},P^o_{h,s,a})\big). 
 \end{align}
As $Q^{*}_{H}=0$, doing backward iteration of $\cT$, i.e., the \textit{robust dynamic programming} $Q^*_{h} = \cT Q^*_{h+1}$, we get $Q^*_{h}$ for all $h\in[H]$. For each timestep $h\in[H]$, we also get the robust optimal policy as $\pi^*_h(s)=\argmax_{a}Q^*_{h}(s,a)$.

\subsection{Problem Conceptualization}
In this section, we study the hybrid finite-horizon robust TV-regularized RL problem, acquiring the necessary insights to construct our algorithm (\cref{alg:HyTQ-Algorithm}) in the next section.
We conceptualize for general $\phi$-divergence, but only propose our algorithm for total variation $\phi$-divergence.
The goal here is to learn a \textit{good} robust policy $\hat{\pi}$ based on adaptive datasets consisting of \textit{both} offline datasets and on-policy datasets. We start by noting a direct consequence of \cref{prop:robust-bellman-dual} due to similar inner minimization problems in both infinite horizon \eqref{eq:robust-regularized-bellman-eq-primal} and finite horizon \eqref{eq:FH-robust-regularized-bellman-eq-primal} operators. 
\begin{corollary}
\label{cor:fh-robust-bellman-dual}
For any $Q_h: \ScA\to [0, H]$ and $h\in[H]$, the robust regularized Bellman operator $\cT$ \eqref{eq:FH-robust-regularized-bellman-eq-primal} can be equivalently written as
\begin{align}
\label{eq:FH-robust-regularized-bellman-eq-dual}
    (\cT &Q_{h+1})(s, a) = r_h(s, a) - \gamma \inf_{\eta\in \Theta} (\lambda \E_{s' \sim P^{o}_{h,s,a}}[\phi^*\left({(\eta-V_{h+1}(s'))}/{\lambda}\right)]  - \eta ), 
\end{align} 
where $V_{h+1}(s)=\max_{a\in\cA} Q_{h+1}(s,a)$ and $\Theta\subset\R$ is some bounded real line that depends on $\phi^*$.
\end{corollary}
As in \cref{sec:infinite-horizon-offline-robust-rl}, this dual reformulation enables us to use the datasets from only the nominal model $P^o$ for estimating the robust regularized operator in its primal form \eqref{eq:FH-robust-regularized-bellman-eq-primal}.

We start by recalling the philosophy of the HyQ algorithm \citep{song2023hybrid} to use the FQI algorithm for adaptive datasets.
We do the same for our hybrid finite-horizon robust $\phi$-regularized RL problem here.
For each $h\in[H]$, we need to estimate the true Bellman error $\E_{s,a\sim\mu_h}[|\cT Q^{*}_{h+1}(s,a)-Q^{*}_{h}(s,a)|]+\sum_{t=0}^{k-1} \E_{s,a\sim d^{\pi_t}_{h,P^o}}[|\cT Q^{*}_{h+1}(s,a)-Q^{*}_{h}(s,a)|]$ using offline dataset from $\mu_h$ and the on-policy dataset from $d^{\pi_t}_{h,P^o}$ by the learned policies from the algorithm. We remark that the out-of-data-distribution issue appears when we only have access to the offline dataset to estimate the summation term above, which depends on $d^{\pi_t}_{h,P^o}$.

As discussed in \cref{sec:infinite-horizon-offline-robust-rl}, the true Bellman error itself involves solving an inner convex minimization problem in $\cT Q^{*}_{h+1}(s,a)$ \eqref{eq:FH-robust-regularized-bellman-eq-dual} for every $(s,a)$ and $h$ that is challenging for countably large state setting. 
To alleviate this challenging task, we again utilize the functional minimization \cref{prop:functional-opt-phi-divergence} developed in \cref{sec:infinite-horizon-offline-robust-rl}. For any $h$, we denote the set of \textit{admissible distributions} of nominal model $P^o$ as $\D_{h}=\{\mu_h\} \cup \{ d^{\pi}_{h,P^o}\,|\text{\,for any policy (including non-stationary)\,}\pi\}$. Now we redefine dual loss for any $f_{h+1}\in\cF_{h+1},\nu_h\in\D_h$, as
\begin{align}
    \label{FH-L-f-eta-loss}
    L_{\mathrm{dual}}&(g; f_{h+1},\nu_h)=\E_{s,a\sim\nu_h,s'\sim P^o_{h,s,a}} [  \lambda \phi^*((g(s,a)-\max_{a'} f_{h+1}(s',a'))/{\lambda})  - g(s,a)].
\end{align}
We state a direct consequence of \cref{prop:functional-opt-phi-divergence} here. 
\begin{corollary}
\label{cor:FH-functional-opt-phi-divergence}
Let $L_{\mathrm{dual}}$ be the loss function defined in \eqref{FH-L-f-eta-loss}. Fix $h\in[H]$ and consider any policy $\pi$. Then, for any function $f_{h+1} : \ScA \rightarrow [0,H]$ and any $\nu_h\in\D_h$, we have
\begin{align}
\label{eq:FH-functional-opt-phi-divergence}
    &\inf_{g \in L^1(\nu_h)}  L_{\mathrm{dual}}(g; f_{h+1},\nu_h)  = \E_{s,a\sim\nu_h} \Big[ \inf_{\eta\in \Theta} (\lambda \E_{s' \sim P^{o}_{h,s,a}}[\phi^*({(\eta-\max_{a'}f_{h+1}(s',a'))}/{\lambda})]  - \eta )\Big]. 
\end{align}
\end{corollary}
For any given  $f_h: \ScA \rightarrow [0,H]$ and $h$, we redefine operator $\cT_g$ for all $g\in\cG_h$, as 
\begin{align}
\label{eq:FH-Tg}
    &(\cT_g f_{h+1}) (s, a) = r_h(s, a) - \lambda \E_{s' \sim P^{o}_{h,s,a}}[\phi^*({(g(s,a)-\max_{a'}f_{h+1}(s',a'))}/{\lambda})]  + g(s,a)  . 
\end{align} 
We have all the necessary tools now. 
In the following subsection, we develop an algorithm that naturally extends our RPQ algorithm using adaptive datasets.

\subsection{Hybrid Robust regularized  Q-iteration}

In this section, we propose our algorithm based on the tools developed so far. Our proposed algorithm is called Hybrid robust Total-variation-regularized  Q-iteration (HyTQ: pronounced \textit{height-Q}) Algorithm, summarized in \cref{alg:HyTQ-Algorithm}. The total variation $\dtv$ $\phi$-divergence \eqref{eq:phi-divergence-defn} is defined with $\phi(t)=|t-1|/2$. The inputs to this algorithm are the offline dataset, and two general function classes $\cF=\otimes_{h\in[H]} \cF_h, \cG=\otimes_{h\in[H]} \cG_h$. 
For any $h\in[H]$, $\cF_h\subset(f:\ScA \rightarrow [0,H])$ and $\cG_h\subset(g:\ScA \rightarrow [0,\lambda])$ represent action-value functions and dual-variable functions at $h$, respectively.
We redefine, using \eqref{eq:tv-inner-optimization-dual-form}, the empirical dual loss and the robust empirical squared robust regularized Bellman error  for dataset $\cD$ as
\begin{align} 
 \label{eq:FH-L-dual-empirical-loss}
    &\widehat{L}_{\mathrm{dual}}(g;f,\cD) =\E_{\cD} [  (g(s_i,a_i)-\max_{a'} f(s_i',a'))_+  - g(s_i,a_i)] \quad\text{and} \\ \label{eq:FH-L-rob-Q-empirical-loss}
    &\widehat{L}_{\mathrm{robQ}}(Q;f, g,\cD) =\E_{\cD} [( [ r_h(s_i, a_i)  - (g(s_i,a_i)-\max_{a'} f(s_i',a'))_+  +  g(s_i,a_i) - Q(s_i,a_i) )^{2}]. 
\end{align}

\begin{algorithm}[h]
	\caption{HyTQ Algorithm}	
	\label{alg:HyTQ-Algorithm}
	\begin{algorithmic}[1]
		\STATE \textbf{Input:} Offline dataset $\cD^{\mu}_h\sim \mu_h$ of size $m_\mathrm{off}=T$ for $h \in [H]$, general function classes $\cF$ and $\cG$. 
		\STATE \textbf{Initialize:} $Q^{0}_{h}\equiv 0 \in \cF_h$.
		\FOR {$k=0,\cdots,K-1$ } 
    \STATE Compute $\pi_{k}$ as $\pi_{k,h}(s)  = \argmax_a Q^{k}_h(s,a)$
    \STATE $\forall h$, collect $m_{\mathrm{on}}$$=$$1$ online dataset $\cD^k_h \sim d_{h,P^o}^{\pi_k}$
    \STATE \textbf{Initialize:} $Q^{k+1}_{H}\equiv 0 \in \cF_H$ 
    \FOR{$h = H-1, \cdots, 0$}  
    \STATE Aggregate adaptive dataset $\cD^k_h = \cD^\mu_{h} + \sum_{\tau = 0}^{k} \cD^\tau_{h}$
        \STATE \textbf{Dual variable function minimization:} (c.f. \eqref{eq:FH-L-dual-empirical-loss})\\\hspace{-0.5cm}$g^{k+1}_{h} = \argmin_{g \in \cG_{h}} \widehat{L}_{\mathrm{dual}}(g; Q^{k+1}_{h+1},\cD^k_h)$ 
        \STATE \textbf{Robust $\phi$-regularized Q-update:} (c.f. \eqref{eq:FH-L-rob-Q-empirical-loss})\\\hspace{-0.5cm}$Q^{k+1}_{h}=\argmin_{Q\in\cF_h}\widehat{L}_{\mathrm{robQ}}(Q;Q^{k+1}_{h+1}, g^{k+1}_{h},\cD^k_h)$ 
		\ENDFOR
		\ENDFOR
	\end{algorithmic}
\end{algorithm}

\subsection{Cumulative Suboptimality Guarantee}

We now discuss the performance guarantee in terms of the cumulative suboptimality of our HyTQ Algorithm. We first mention all the assumptions before we present our main result and add a brief discussion. We provide detailed discussion in \cref{sec:discussion}.

\begin{assumption}[Robust Bellman Error Transfer Coefficient] 
\label{assum-bellman-transfer-coefficient}
Let $\mu_h\in\Delta(\ScA)$ be the offline data generating distribution. For any $f\in\cF$, there exists a small positive constant $C(\pi^*)$ for the optimal policy $\pi^*$ that satisfies 
\[  \frac{\sum_{h=0}^{H-1} \E_{s,a\sim d^{h,\pi^*}_{P^o}}[\cT f_{h+1}(s,a) -  f_{h}(s,a)]}{ \sum_{h=0}^{H-1} \E_{s,a\sim \mu_{h}}[|\cT f_{h+1}(s,a) -  f_{h}(s,a)|] } \leq C(\pi^*). \]
\end{assumption}

We develop this assumption from non-robust offline RL work \citep{song2023hybrid}.

\begin{assumption}[Approximate  Value Realizability and Robust Bellman Completeness] 
\label{assum-realizability-with-bellman-completion}
Let $\epsilon_{\cF,\mathrm{r}}$$\geq$$0$ be small constant. For any  $h\in[H]$ and  $g_h\in\cG_h$, we have $\inf_{f\in\cF_h} \sup_{\nu_h}  \| f - \cT_{g_h} f_{h+1} \|_{2,\nu_h}^2 \leq \epsilon_{\cF,r}$ for all $\nu_h \in \D_{h}$. Furthermore, for any $f_{h+1} \in \cF_{h+1}$, we have $\cT_{g_h} f_{h+1} \in \cF_h$.
\end{assumption}

\begin{assumption}[Approximate Dual Realizability] \label{assum-dual-realizability-hybrid}
Let $\epsilon_{\cG}$ be some small positive constant. For any  $h\in[H]$ and $f_{h+1} \in \cF_{h+1}$, we have $ \inf_{g \in \cG_{h}}  L_{\mathrm{dual}}(g; f_{h+1}, \nu_h) - \inf_{g \in L^1(\nu_h)}   L_{\mathrm{dual}}(g; f_{h+1}, \nu_h) \leq \epsilon_{\cG}$, for all $\nu_h \in \D_h$.
\end{assumption}
We adapt these two enhanced realizability assumptions from the non-robust offline RL literature~\citep{xie2021bellman,foster2022offline,song2023hybrid} to our problem. 
The assumptions in \cref{sec:infinite-horizon-offline-robust-rl} are not directly comparable, but for the sake of exposition, let $\cF_h,\cG_h$ be the same across $h$.
First, note that \cref{assum-dual-realizability} with all-policy concentrability (\cref{assum-concentra-condition}) is equivalent to \cref{assum-dual-realizability-hybrid}. Second, \cref{assum-bellman-completion} implies $\inf_{f\in\cF} \| f - \cT_g f \|_{2,\mu}^2 \leq \epsilon_{\cF}$. Now again, with all-policy concentrability (\cref{assum-concentra-condition}), it is the approximate value realizability (\cref{assum-realizability-with-bellman-completion}). We know  non-robust offline RL is hard \citep{foster2022offline} with just realizability and all-policy concentrability. As robust RL is at least as hard as its non-robust counterpart \citep{panaganti22a}, we also assume Bellman completeness in \cref{assum-realizability-with-bellman-completion}.

\begin{assumption}[Bilinear Models] 
\label{assum-bilinear-model}
Consider any $f\in\cF,g\in\cG$ and $h\in[H]$. Let $\pi^f$ be greedy policy w.r.t $f$. 
There exists an unknown feature mapping $X_h:\cF\mapsto \mathbb{R}^d$ and  two unknown weight mappings $W^{\mathrm{q}}_h,W^{\mathrm{d}}_h:\cF\times\cG\mapsto \mathbb{R}^d$ with $\max_{f}\| X_h(f)\|_2 \leq B_X$ and $\max_{f,g} \max\{\|W^{\mathrm{q}}_h(f,g)\|_2,\|W^{\mathrm{d}}_h(f,g)\|_2\} \leq B_W$
such that both
$\E_{d^{\pi^f}_{h}}[( f_{h}(s,a) - T_{g_{h}} f_{h+1})_+]
=  \abs{\tri{X_h(f), W^{\mathrm{q}}_h(f,g)}}$ and $\E_{ d^{\pi^f}_{h}} [ (T_{g_{h}} f_{h+1} - T f_{h+1} )_+] =  \abs{\tri{X_h(f), W^{\mathrm{d}}_h(f,g)}}$ holds.
\end{assumption}
We adapt this problem architecture assumption on $P^o$ with $\cF$ and $\cG$  for our setting from a series of non-robust online RL works \citep{jin2021bellman,du2021bilinear}.

\begin{assumption}[Fail-state] \label{fail-state-assump}
    There is a fail state $s_{f,h}$ for all $h\in[H]$, such that $r_h(s_f, a) = 0$ and $P_{h,s_f, a}(s_{f,h}) = 1$, for all $a \in \cA$ and $P \in \cP$ satisfying $\dtv(P_{h',s',a'},P^o_{h',s',a'})\leq \max\{1,H/\lambda\}$  for all  $h',s',a'$.
\end{assumption}
This assumption enables us to ground the value of such $P$'s at $s_{f,h}$ to zero, which helps us to get a tight duality (c.f. \eqref{eq:tv-inner-optimization-dual-form}) without having to know the minimum value across large $\cS$. There are approximations to this in the literature \citep{wang2022policy}. But we adopt this less restrictive assumption from \citet{panaganti-rfqi} for convenience. 

Now we state our main theoretical result on the performance of the HyTQ algorithm. The proof is presented in \cref{appendix:hybrid-robust-results}.
\begin{theorem}
\label{thm:hybrid-tv-guarantee}
Let \cref{assum-bellman-transfer-coefficient,assum-realizability-with-bellman-completion,assum-dual-realizability-hybrid,assum-bilinear-model,fail-state-assump} hold. Fix any $\delta\in(0,1)$. Then, HyTQ algorithm policies $\{\pi_k\}_{k\in[K]}$ satisfy
$  \sum_{k = 0}^{K-1} (V^{\pi^*} - V^{\pi_{k}}) \leq \cOtilde(\sqrt{\epsilon_{\cF,\mathrm{r}}} + \epsilon_{\cG}) + \cOtilde ( \max\{C(\pi^*), 1\} \sqrt{d H^2 K} (\lambda+H) \log(|\cF||\cG|/\delta) )
$
with probability at least  $1 - \delta$.
\end{theorem}

\begin{remark} \label{remark:hytq-theorem}
We specialize this result for bilinear model examples, \emph{linear occupancy complexity model} \citep[Definition 4.7]{du2021bilinear} and \emph{low-rank feature selection model} \citep[Definition A.1]{du2021bilinear}, in \cref{appen:hytq-specialized-results}.
We also specialize this result using standard online-to-batch conversion \citep{shalev2014understanding} for uniform policy over HyTQ policies $\{\pi_k\}_{k\in[K]}$ to provide sample complexity 
$\cOtilde ( \max\{(C(\pi^*))^2, 1\} d H^3 (\lambda+H)^2  (\log(|\cF||\cG|/\delta))^2 )/{\epsilon^2}$ in the \cref{appen:hytq-specialized-results}.
\end{remark}

%% file: 05-discussion-conclusion.tex
\section{Theoretical Discussions and Final Remarks}\label{sec:discussion}

In this section, we first discuss the proof ideas for our results, focusing on discussions of the assumptions and their improvements. Next, we compare our results with the most relevant ones from the robust RL literature. Our \cref{table:comparison-table} should be used as a reference. Finally, we discuss the bilinear model architecture in detail, as ours is the first work to consider it in the robust RL setting under the general function architecture for the value and dual functions approximations.

\textbf{Discussions on Proof Sketch:}
We first discuss our RPQ algorithm (\cref{alg:infinite-horizon-RPQ-Algorithm}) result. We note that the concentrability (\cref{assum-concentra-condition}) assumption requires the data-generating policy to be robust exploratory. That is, it covers the state-action occupancy induced by any policy and any $\phi$-divergence set transition model. 
We reiterate the proof idea of the suboptimality result \citep[Theorem 1]{panaganti-rfqi} of the RFQI algorithm \citep[Algorithm 1]{panaganti-rfqi}. We highlight the most important differences with \citet{panaganti-rfqi,zhang2023regularized} here. Firsty, we generalize the robust performance lemma ($\E_{s_0\sim d_0}[{V}^{\pi^*}] - \E_{s_0\sim d_0}[V^{\pi_K}] \leq 2\|Q^{\pi^*} -  Q_{K}\|_{1,\nu }/(1-\gamma)$ at \cref{eq:thm-bound-part-1}) for any general $\phi$-divergence problem.
Secondly, we identify that it is hard to come up with a unified analysis for general $\phi$-divergences in robust RL setting via the dual reformulation of the distributionally robust optimization problem \citep[Proposition 1]{duchi2018learning}. Thus, a direct extension of the results in \citet{panaganti-rfqi} is hard for general $\phi$-divergences. By RPQ analyses, we showcase that it is indeed possible to get a unified analysis for the robust RL problem using the RRMDP framework. 
Thirdly, we show the generalization bounds for the empirical risk minimization (\cref{prop:erm-high-prob-bound}) and least squares (\cref{prop:least-squares-generalization-bound}) estimators for general $\phi$-divergences with unified results.
By these three points, equipped with the more general robust exploratory concentrability (\cref{assum-concentra-condition}), we have a unified general $\phi$-divergences  suboptimality result (\cref{thm:infinite-horizon-phi-divergence-guarantee}) for the RPQ algorithm.

We now discuss our HyTQ algorithm (\cref{alg:HyTQ-Algorithm}) result. We immediately make an important note here. The concentrability assumption improvement is two-fold: all-policy concentrability (\cref{tv-assum-concentra-condition}) to single concentrability, and then to the robust Bellman error transfer coefficient (\cref{assum-bellman-transfer-coefficient}) via \cref{lem:transfer-coeff-concentrability}. 
We refer to \citet{foster2022offline,song2023hybrid} for further discussion on such concentrability assumption improvements and tightness in the non-robust offline RL. We leave it to future work for more tightness of these assumptions in the robust RL setting.
We  execute a tighter analysis in our HyTQ algorithm  result (\cref{thm:hybrid-tv-guarantee}) compared to our RPQ algorithm TV $\phi$-divergence specialized result (\cref{thm:infinite-horizon-tv-guarantee}). We summarize the steps as follows:\\
Step $(a)$: We meticulously arrive at  the following robust performance lemma (c.f. \cref{eq:tv-part1,eq:tv-part2}) for each algorithm iteration $k$:  $\E_{s_0\sim d_0}[V^{\pi^*}_{0}(s_0) - V^{\pi_k}_{0}(s_0)] \leq \textstyle\sum_{h=0}^{H-1} \E_{s,a\sim d^{\pi^*}_{h}}[(\cT Q^{k}_{h+1}(s,a) -  Q^{k}_{h}(s,a))_+] +  \sum_{h=0}^{H-1} \E_{s,a\sim d^{\pi_k}_{h}}[( Q^{k}_{h}(s,a) - \cT Q^{k}_{h+1}(s,a) )_+].$ We highlight that the first summand here depends on the samples from state-action occupancy of the optimal robust policy and for the second summand it is the w.r.t. the learned HyTQ policies. It is now intuitive to connect the first summand with the offline samples and the second with the online samples.   \\
Finally, step $(b)$: With the above gathered intuition, firstly, the history dependent dataset collected by different offline data-generating policy and the learned HyTQ policies on the nominal model $P^o$ warrants more sophisticated generalization bounds for the empirical risk minimization and least squares estimators. We prove a generalization bound for empirical risk minimization when the data are not necessarily i.i.d. but adapted to a stochastic process in \cref{appen:foundation-results}. This result is applicable to more machine learning problems outside of the scope of this paper as well. 
Finally, equipped with the transfer coefficient (\cref{assum-bellman-transfer-coefficient}) and bilinear model (\cref{assum-bilinear-model}) assumptions for the nominal model $P^o$, we formally show generalization bounds for the empirical risk minimization and least squares estimators in \cref{prop:erm-high-prob-bound-offline+online,prop:least-squares-generalization-bound-online+offline} respectively.\\
We complete the proof by combining these two steps.

\begin{remark}
    We offer computational tractability in our RPQ and HyTQ algorithms due to the usage of empirical risk minimization (Steps 4 \& 9 resp.), over the general function class $\cG$, and least-squares (Steps 5 \& 10 resp.), over the general function class $\cF$, \emph{computationally tractable} estimators. This two-step estimator update avoids the complexity of solving the inner problem for each state-action pair (leading to scaling issues for high-dimensional problems) in the original robust Bellman operators (\cref{eq:robust-regularized-bellman-eq-primal,eq:FH-robust-regularized-bellman-eq-primal}).
    To the best of our knowledge, no purely online or purely offline robust RL algorithms are known to be tractable in this sense, except other robust Q-iteration and actor-critic methods (discussed in \cref{intro-related-works}) and except under much stronger coverage conditions (like single-policy and uniform) in the tabular setting.
\end{remark}

\textbf{Theoretical Guarantee Discussions:}
In the suboptimality result (\cref{thm:infinite-horizon-phi-divergence-guarantee}) for the RPQ algorithm (\cref{alg:infinite-horizon-RPQ-Algorithm}), we only mention the leading statistical bound with a problem-dependent (on $\phi$-divergence) constant $c_{\phi}(\lambda,\gamma)$. We provide the exact constants pertaining to different $\phi$-divergences in a restated statement of \cref{thm:infinite-horizon-phi-divergence-guarantee} in \cref{thm:restated:infinite-horizon-phi-divergence-guarantee}. Furthermore, the constants $c_1, c_2, c_3$ in \cref{thm:restated:infinite-horizon-phi-divergence-guarantee} take different values for different $\phi$-divergences provided in \cref{prop:divergence-loss-bounds}. Similarly, for the suboptimality result (\cref{thm:hybrid-tv-guarantee}) of the HyTQ algorithm (\cref{alg:HyTQ-Algorithm}), we provide a more detailed bound in a restated statement in \cref{thm:restated:hybrid-tv-guarantee}.

In the following we provide comparisons of suboptimality results with relevant prior works. But first, we make an important note here on $\rho$, the robustness radius parameter in RMDPs, and $\lambda$, the robustness penalization parameter in RRMDPs, mentioned briefly in \cref{table:comparison-table}.
\citep{levy2020large,yang2023avoiding} establish the regularized and constrained versions of DRO and robust MDP problems, respectively, are equivalent by connecting their respective ($\lambda$ and $\rho$) robustness parameters. Moreover, both observe rigorously that $\lambda$ and $\rho$ are inversely related. This is intuitively true, as $\lambda\to \infty$ and $\rho\to 0$ both yield the non-robust solutions on the nominal model $P^o$ and as $\lambda\to 0$ and $\rho\to \infty$ both yield the conservative solutions considering the entire probability simplex for the transition dynamics. 
However, it is an interesting open problem to establish an exact analytical relation between the  robustness parameters $\lambda$ and $\rho$. We leave this to future research as it is out of the scope of this work.

Here we specialize our result (\cref{thm:restated:infinite-horizon-phi-divergence-guarantee}) for the chi-square $\phi$-divergence $\gamma$R$^3$L problem. We get the suboptimality for the RPQ algorithm as $\cOtilde\left(\frac{ \max\{ \frac{1}{\lambda(1-\gamma)^2},\lambda\} \sqrt{ C \log({|\cF||\cG|}) } }{ (1-\gamma)^{2} \sqrt{N} }\right)$, where we only have presented the higher-order terms. The suboptimality of Algorithm 2 in \citet[Theorem 5.1]{yang2023avoiding} for chi-square $\phi$-divergence is stated  for $\lambda=1/(1-\gamma)$ as $\cOtilde\left(\frac{ \max\{ \frac{1}{(1-\gamma)^2},\sqrt{\log(|\cS||\cA|)}\} }{ d^3_{\min} (1-\gamma)^{3} N^{1/3} }\right)$ where $d_{\min}$  is described in \cref{table:comparison-table}. We use the typical equivalence from RL literature for comparison between these two results in the tabular setting with generative/simulator modeling assumption: function approximation classes with full dimension yields $\log(\mathcal{|F||G|})=O(|\cS||\cA|)$ \citep{panaganti-rfqi} and uniform support data sampling yields $\mu_{\min}=1/(|\cS||\cA|)$ and $C\leq |\cS||\cA|$ \citep{shi2023curious}. Now our result with $\lambda=1/(1-\gamma)$  reduces to $\cOtilde\left(\frac{   |\cS||\cA|  }{ (1-\gamma)^{3} \sqrt{N} }\right)$ and  their result \citep{yang2023avoiding} reduces to $\cOtilde\left(\frac{ |\cS|^3 |\cA|^3 \max\{ \frac{1}{(1-\gamma)^2},\sqrt{\log(|\cS||\cA|)}\} }{ (1-\gamma)^{3} N^{1/3} }\right)$. Two comments warrant attention here. Firstly, compared to a model-based robust regularized algorithm (robust value iteration using empirical estimates of the nominal model $P^o$) \citep[Theorem 3.2]{yang2023avoiding}, our suboptimality bound is worse off by the factors $\sqrt{|\cS||\cA|}$ and $1/(1-\gamma)$. We leave it to future work to fine-tune and get optimal rates.  Secondly, their result \citet[Theorem 5.1]{yang2023avoiding} exhibit inferior performance compared to ours in all parameters, but we do want to note that they make a first attempt to give suboptimality bounds for the stochastic approximation-based algorithm. The dependence on $|\cS||\cA|$ is typically known  to be bad using the stochastic approximation technical tool \citep{chen2022finite}, and \citet[Discussion on Page 16]{yang2023avoiding} conjectures using the Polyak-averaging technique to improve their suboptimality bound rate to $N^{-1/2}$.


Here we discuss and compare our result for the total variation $\phi$-divergence setting. As mentioned in \cref{rem:tv-offline-guarantee}, we have a specialized result in \cref{appen:infinite-horizon-tv-guarantee} for the total variation $\phi$-divergence.
We get the suboptimality result (\cref{thm:infinite-horizon-tv-guarantee}) for the RPQ algorithm as $\cOtilde\left(\frac{ \lambda  \sqrt{ C_{\mathrm{tv}} \log({|\cF||\cG|}) } }{ (1-\gamma)^{3} \sqrt{N} }\right)$, where we again only have presented the higher-order terms. \citet[Theorem 1]{panaganti-rfqi} mentioned in \cref{table:comparison-table} also exhibits same suboptimality guarantee replacing $\lambda$ with $\rho^{-1}$. As we noted before, $\rho$ (the robustness radius parameter in RMDPs) and $\lambda$ (the robustness penalization parameter in RRMDPs) are inversely related, and for the TV $\phi$-divergence we observe a straightforward relation between the two as $\lambda=\rho^{-1}$.
Using the earlier arguments for a tabular setting bound, our result further reduces to $\cOtilde\left(\frac{  \lambda |\cS||\cA|  }{ (1-\gamma)^{3} \sqrt{N} }\right)$. Now comparing this to the minimax lower bound \citep[Theorem 2]{shi2023curious}, our suboptimality bound is worse off by the factors $\sqrt{|\cS||\cA|}$ and $1/(1-\gamma)$.
Nevertheless, we push the boundaries by providing novel suboptimality guarantee studying the robust RL problem in the hybrid RL setting. Furthermore, as mentioned earlier in \cref{remark:hytq-theorem}, we provide the offline+online robust RL suboptimality guarantee $\cOtilde \left( {\max\{C(\pi^*), 1\} \sqrt{d H^3} (\lambda+H) \log(|\cF||\cG|/\delta) }/{\sqrt{N}}\right)$ in the \cref{appendix:hybrid-robust-results}.
We also remark that the HyTQ algorithm can be proposed under the RMDP setting with a similar suboptimality guarantee due to the similarity of the dual Bellman equations under the TV $\phi$-divergence for RMDPs and RRMDPs (c.f. \cref{eq:FH-tv-dual-bellman-eq} and \citet[Lemma 8]{xu-panaganti-2023samplecomplexity}). For the sake of consistency and novelty, we present our results solely for the RRMDP setting.
As mentioned earlier, the concentrability assumption improvement is two-fold (\cref{lem:transfer-coeff-concentrability}): all-policy concentrability (\cref{tv-assum-concentra-condition}) to single concentrability to transfer coefficient. 
This is the first of its kind result that does not yet have any existing lower bounds to compare in the robust RL setting.  
Under similar  transfer coefficient, Bellman completeness, and bilinear model assumptions, the HyTQ algorithm sample complexity (\cref{cor:hytq-sample-complexity-rl}) is comparable to that of a non-robust RL algorithm \citep{song2023hybrid}, i.e., $\cOtilde ( {\max\{(C(\pi^*))^2, 1\} d H^5}  \log(H|\cF|/\delta) / {\epsilon^2} )$.
We leave it to future work for developing  minimax rates and getting optimal algorithm guarantees.

Here we specialize our result (\cref{thm:restated:infinite-horizon-phi-divergence-guarantee}) for the KL $\phi$-divergence $\gamma$R$^3$L problem. We get the suboptimality for RPQ as $\cOtilde\left(\frac{ (\lambda + (1-\gamma)^{-1}) \exp\{ (\lambda(1-\gamma))^{-1}\} \sqrt{ C \log({|\cF||\cG|}) } }{ (1-\gamma)^{2} \sqrt{N} }\right)$, where we only have presented the higher-order terms. Using the earlier arguments for a tabular setting bound, our result with $\lambda=1/(1-\gamma)$ again reduces to $\cOtilde\left(\frac{   |\cS||\cA|  }{ (1-\gamma)^{3} \sqrt{N} }\right)$. \citet[Theorem 5]{zhang2023regularized} mentioned in \cref{table:comparison-table} also exhibits same suboptimality guarantee. 
Two remarks are in order here.
Firstly, we remark that our RPQ algorithm and its theoretical guarantee unifies for a class of $\phi$-divergence classes, whereas \citet[Algorithm 1]{zhang2023regularized} is specialized for the KL $\phi$-divergence. This steers towards our first main contribution discussed in \cref{sec:introduction}.
Secondly, we remark the robust regularized Bellman operator \cref{eq:robust-regularized-bellman-eq-primal} for the KL $\phi$-divergence has a special form due to the existence of an analytical worse-case transition model.  This arrives at a special structure of the form of an exponential robust Bellman operator in a Q-value-variant space. This special structure helps avoid the dual variable function update (Step 4) in the RPQ algorithm and the $\log(|\cG|)$ factor in the suboptimal guarantee.  We choose not to include this specialized result in this work (like we did for the TV $\phi$-divergence in \cref{appen:infinite-horizon-tv-guarantee})  and directly point to \citet{zhang2023regularized}. We do highlight here an important note for such a choice in our paper. The abovementioned special structure forces us to get online samples \emph{from all the transition kernels} (c.f. \cref{assum-concentra-condition}), which is unrealistic in practice, to achieve an improvement in the hybrid robust RL setting. 
We leave it to future work for developing  such improved algorithm guarantees in the hybrid robust RL setting for other $\phi$-divergences.

\textbf{Discussion of Bilinear Models in the Hybrid Robust RL setting:}
We emphasize that while our bilinear model for the HyTQ algorithm is specialized to low occupancy complexity (i.e. the occupancy measures themselves have a low-rank structure) and low-rank feature selection model (i.e. the nominal model $P^o$ has a low-rank structure) in \cref{appen:hytq-specialized-results}, the function classes $\cF$ (Q-value representations) and $\cG$ (dual-value representations) can be arbitrary, potentially nonlinear function classes (neural tangent kernels, neural networks, etc).
Thus, even in the tabular setting with large state space (e.g. $|\cS|>O(10^5)$) for the bilinear model, our suboptimality bounds only scale with the complexity of the function classes $\cF$ and $\cG$, which can considerably be low compared to $|\cS|$. For example, linear function approximators (e.g. linear feature dimension $d=\log(\mathcal{|F||G|})\ll |\cS||\cA|$), RKHS approximators with low dimension features, neural tangent kernels with low effective neural net dimension, and more function approximators. Moreover, our work solves the robust RL problem with more nuances, which is at least as hard as the non-robust RL problem. Thus, due to the new upcoming research status of robust RL in the general function approximation setting, we believe it is currently out of scope for this work to satisfy more general bilinear model classes \citep{du2021bilinear}. Nevertheless, our initial findings for robust RL by the HyTQ algorithm in the hybrid learning setting reveal the hardness of finding larger model classes for RRMDPs with general $\phi$-divergences.


We conclude this section with an exciting future research direction that remains unsolved in this paper. To solve the hybrid robust RL problem for general $\phi$-divergence. In this work, we noticed while building hybrid learning for robust RL that one would require online samples from the worse-case model (c.f. the model that solves the inner problem in robust Bellman operator \cref{eq:FH-robust-regularized-bellman-eq-primal}) for general $\phi$-divergences due to the current analyses dependent on the bilinear models. We use the dual reformulation for the total variation $\phi$-divergence and provide current results supporting the HyTQ algorithm. We remark that using the same approach for other general $\phi$-divergences, we get exponential dependence on the horizon factor. This warrants more sophisticated algorithm designs for the hybrid robust RL problem under general $\phi$-divergences.

\section{Conclusion}\label{sec:conclusion}

In this work, we presented two robust RL algorithms. We proposed Robust $\phi$-divergence-fitted Q-iteration algorithm for general $\phi$-divergence in the offline RL setting. We provided performance guarantees with unified analysis for all $\phi$-divergences with arbitrarily large state space using function approximation. To mitigate the out-of-data-distribution issue by improving the assumptions on data generation, we proposed a novel framework called hybrid robust RL that uses both offline and online interactions. We proposed the Total-variation-divergence Q-iteration algorithm in this framework with an accompanying guarantee. We have provided our theoretical guarantees in terms of suboptimality and sample complexity for both offline and offline+online robust RL settings. We also rigorously specialized our results to different $\phi$-divergences and different bilinear modeling assumptions. We have provided detailed comparisons with relevant prior works while also discussing interesting future directions in the field of robust reinforcement learning.

\section*{Acknowledgment}\label{sec:ack}

\addcontentsline{toc}{section}{\protect\numberline{}Acknowledgment}

KP acknowledges support from the `PIMCO Postdoctoral Fellow in Data Science' fellowship at the California Institute of Technology.
This work acknowledges support from NSF CNS-2146814, CPS-2136197, CNS-2106403, NGSDI-2105648, and funding from the Resnick Institute.
EM acknowledges support from NSF award 2240110.
We thank several anonymous ICML 2024 reviewers for their constructive comments on an earlier draft of this paper.

%% file: 06-appendix.tex
\section*{ \centering \Large \bf \underline{\Coffeecup ~\Coffeecup ~Supplementary Materials \Coffeecup ~\Coffeecup}}

\addcontentsline{toc}{section}{\protect\numberline{}Supplementary Materials\hfill \Coffeecup}

\appendix

\section[Related Works]{Related Works \hfill \Coffeecup} \label{appen:related works}

\textbf{Offline RL:}
Offline RL tackles the problem of learning optimal policy using minimal amount of  offline/historical data collected according to a behavior policy \citep{lange2012batch, levine2020offline}. Due to offline data quality and no access to simulators or any world models for exploration, the offline RL problem suffers from the out-of-distribution \citep{robey2020model,yang2021generalized} challenge. 
Many works \citep{fujimoto2019off, kumar2019stabilizing, kumar2020conservative, fujimoto2021minimalist, kostrikov2021offline} have introduced deep offline RL algorithms aimed at alleviating the out-of-distribution issue by some variants of trust-region optimization \citep{schulman2015trust,schulman2017proximal}. The earliest and most promising theoretical investigations into model-free offline RL methodologies relied on the assumption of \emph{uniformly bounded concentrability} such as the approximate modified policy iteration (AMPI) algorithm \citep{scherrer2015approximate} and fitted Q-iteration (FQI) \citep{munos08a} algorithm. This assumption mandates that the ratio of the state-action occupancy distribution induced by \emph{any} policy to the data generating distribution remains uniformly bounded across all states and actions \citep{munos2007performance, antos2008learning, munos08a, farahmand2010error, chen2019information}.
This makes offline RL particularly challenging \citep{foster2022offline} and there have been efforts to understand the limits of this setting.

\textbf{Robust RL:}
The  robust Markov decision process framework \citep{nilim2005robust,iyengar2005robust} tackles the challenge of formulating a policy resilient to model discrepancies between training and testing environments. Robust reinforcement learning problem pursues this objective in the data-driven domain. Deploying simplistic RL policies \citep{isaac-sim2real-2021} can lead to catastrophic outcomes when faced with evident disparities in models. 
The optimization techniques and analyses in robust RL draw inspiration from the distributionally robust optimization (DRO) toolkit in supervised learning \citep{duchi2018learning, shapiro-2017-dro, gao-2022-distributionally, bertsimas2018dro, namkoong2016stochastic, blanchet2019dro}.
Many heuristic works \citep{xu2010distributionally, wiesemann2013robust, yu2015distributionally, mannor2016robust, russel2019beyond} show robust RL is valuable in such scenarios involving disparities of a simulator model with the real-world model. 
Many recent works address fundamental issues of RMDP giving concrete theoretical understanding in terms of sample complexity \citep{panaganti2021sample,panaganti22a, xu-panaganti-2023samplecomplexity, shi2022distributionally, shi2023curious}. Many works \citep{panaganti2020robust,wang2021online,panaganti22a} devise model-free online and offline robust RL algorithms employing general function approximation to handle potentially infinite state spaces. Recent work \citep{panaganti2023distributionally} introduces distributional robustness in the imitation learning setting.
There have been works \citep{panaganti2023thesis,panaganti2023bridging,wang2023distributionally} connecting robust RL with offline RL by linking notions of robustness and pessimism.

\section[Useful Technical Results]{Useful Technical Results \hfill \Coffeecup\Coffeecup}

We state the following result from the \textit{penalized distributionally robust optimization} literature \citep{levy2020large}.
\begin{lemma}[\text{\citealp[Section A.1.2]{levy2020large}}]
\label{lem:penalized-dro-inner-problem-solution}
Let $P^o$ be a distribution on the space $\cX$ and let $l: \cX \to \R$ be a loss function. For $\phi$-divergence \eqref{eq:phi-divergence-defn}, we have 
\begin{align*}
   \sup_{P\ll P^o} \E_{P}[l(X)-\lambda \dphi(P,P^o) ] = \inf_{\eta\in\R} ~~ \lambda \E_{P^o} \left[  \phi^* \left(\frac{l(X)-\eta}{\lambda}\right)\right] + \eta,
\end{align*}
where $\phi^*(s) = \sup_{t\geq 0} \{st - \phi(t)\}$ is the Fenchel conjugate function of $\phi$. Moreover, the optimization on the right hand side is convex in $\eta$.
\end{lemma}

We state a standard concentration inequality here.
\begin{lemma}[Bernstein's Inequality \hspace{-0.1cm} \text{\citep[Theorem 2.8.4]{vershynin2018high}}] \label{lem:bernstein-ineq}
Fix any $\delta\in(0,1)$. If $X_1,\cdots, X_T$ are independent and identically distributed random variables with finite second moment. Assume that $|X_t - \E[X_t]| \leq M$, for all $t$. Then we have with probability at least $1-\delta$:
\[ 
\Bigg|\E[X_1] - \frac{1}{T} \sum_{t=1}^T X_t \Bigg| \leq \sqrt{\frac{2 \E[X_1^2] \log(2/\delta)}{T}} +\frac{M\log(2/\delta)}{3T}.
\] 
\end{lemma}

We now state a useful concentration inequality when the samples are not necessarily i.i.d. but adapted to a stochastic process.
\begin{lemma}[Freedman's Inequality \hspace{-0.1cm} \text{\citep[Lemma 14]{song2023hybrid}}] 
\label{lem:freedman-ineq}
Let $X_1,\cdots, X_T$ be a sequence of $M>0$-bounded real valued random variables where $X_t \sim P_t$ from some stochastic process $P_t$ that depends on the history $X_1,\cdots, X_{t-1}$. Then, for any $\delta > 0$ and $\lambda \in [0, 1/2M]$, we have with probability at least $1 - \delta$:
\begin{align*}
\Bigg| \sum_{t=1}^T (X_t -\E[X_t \mid P_t]) \Bigg| \leq \lambda \sum_{t=1}^T (2  M |\E[X_t \mid P_t]| + \E[X_t^2 \mid P_t]) + \frac{\log(2/\delta)}{\lambda} .
\end{align*}
\end{lemma}

We now state a result
for the generalization bounds on empirical risk minimization (ERM) \citep{shalev2014understanding}.
\begin{lemma}[ERM Generalization Bound \hspace{-0.1cm} \text{\citep[Lemma 3]{panaganti-rfqi}}] \label{lem:erm-generalization-bound}
Let $P$ be the data generating distribution on the space $\cX$ and let $\cH$ be a given hypothesis class of functions. Assume that for all $x\in\cX$ and $h \in \cH$ for loss function $l$ we have that $|l(h,x)| \leq c_1$ for some positive constant $c_1>0$ and $l(h,x)$ is $c_3$-Lipschitz in $h$. Given a dataset $\cD=\{X_i\}_{i=1}^N$, generated independently from $P$, denote $\hat{h}$ as the ERM solution, i.e. $\hat{h}=\argmin_{h\in\cH} (1/N) \sum_{i=1}^N l(h, X_i)$. Furthermore, let $\cH$ be a finite hypothesis class, i.e. $|\cH|<\infty$,  with $|h\circ x| \leq c_2$ for all $h\in\cH$ and $x\in\cX$.
For any fixed $\delta\in(0,1)$ and $h^*\in\argmin_{h\in\cH} \E_{X\sim P}[l(h,X)]$, we have 
\begin{align*} 
    \E_{X\sim P}[l(\hat{h},X)] - \E_{X\sim P}[l(h^*,X)] \leq 2 c_2 c_3 \sqrt{\frac{2\log(|\cH|)}{N}} + 5 c_1 \sqrt{\frac{2\log(8/\delta)}{N}},
\end{align*}  
with probability at least $1-\delta$.
\end{lemma}

We now state a result from variational analysis literature \citep{rockafellar2009variational} that is useful to relate minimization of integrals and the integrals of pointwise minimization under decomposable spaces.
\begin{remark} \label{remark:decomposable}
A few examples of decomposable spaces are $L^p(\ScA,\Sigma(\ScA),\mu)$, for any $p\geq 1$, and $\cM(\ScA,\Sigma(\ScA))$, the space of all $\Sigma(\ScA)$-measurable functions. \end{remark}

\begin{lemma}[\text{\citealp[Theorem 14.60]{rockafellar2009variational}}] \label{lem:integral-min-exchange-rockafellar}
Let $\cX$ be a space of measurable functions from $\Omega$ to $\R$ that is decomposable relative to a $\sigma$-finite measure $\mu$ on the $\sigma$-algebra $\cA$. Let $f : \Omega \times \R \to \R$ (finite-valued) be a normal integrand. Then, we have 
\[ 
\inf_{x\in\cX} \int_{\omega\in\Omega} f(\omega,x(\omega)) \mu(\dd\omega) = \int_{\omega\in\Omega} \left( \inf_{x\in\R} f(\omega,x) \right) \mu(\dd\omega).
\]
Moreover, as long as the above infimum is finite, we have that
$ x' \in \argmin_{x\in\cX} \int_{\omega\in\Omega} f(\omega,x(\omega)) \mu(\dd\omega)$
if and only if $x'(\omega) \in \argmin_{x\in\R} f(\omega,x)$  for $\mu$-almost everywhere.
\end{lemma}

Now we state a few results that will be useful for the analysis of our finite-horizon results in this work.
The following result \citep[Lemma 6]{song2023hybrid} is useful under the use of bilinear model approximation. This result follows from the elliptical potential lemma \citep[Lemma 19.4]{lattimore2020bandit} for deterministic vectors.
\begin{lemma}[Elliptical Potential Lemma]
\label{lem:elliptical} 
Let $X_h(f^1), \cdots, X_h(f^T) \in \R^d$ be a sequence of vectors with $\norm{X_h(f^t)} \leq B_X < \infty$ for all $t \leq T$ and fix $\sigma \geq B^2_X$. Define $\Sigma_{t;h} = \sum_{\tau = 1}^t X_h(f^\tau) X_h(f^\tau)^\top + \sigma \indic_{d\times d}$ for $t \in [T]$. Then, the following holds: 
$\sum_{t=1}^T \|X_h(f^t)\|_{\Sigma_{t - 1;h}^{-1}} \leq \sqrt{2 d T \log(1 + ({T B_X^2}/{(\sigma d)}))}. $ 
\end{lemma}

We now state a result for the generalization bounds on the least-squares regression problem when the data are not necessarily i.i.d. but adapted to a stochastic process. We refer to \citet{van2015fast} for more statistical and online learning generalization bounds for a wider class of loss functions.
\begin{lemma}[Online Least-squares Generalization Bound \hspace{-0.1cm} \text{\citep[Lemma 3]{song2023hybrid}}]  
\label{lem:least-squares-bound-song} 
Let $L,M > 0$, $\delta \in (0, 1)$, and let $\cX$ be an input space  and $\cY$ be a target space . Let $\cH: \cX \mapsto [-M, M]$ be a given real-valued  hypothesis class of functions with $|\cH|<\infty$. Given a dataset $\cD=\{(x_i,y_i)\}_{i=1}^N$, denote $\widehat{h}$ as the least square solution, i.e. $\widehat{h}=\argmin_{h\in\cH} \sum_{i=1}^N (h(x_t) - y_t)^2$. 
The dataset $\cD$ is generated as $x_t \sim P_t$  from some stochastic process $P_t$ that depends on the history $\{(x_1, y_1), \dots, (x_{t-1}, y_{t-1})\}$, and $y_t$ is sampled via the conditional probability $p(\cdot \mid x_t)$ as
$y_t \sim  p(\cdot \mid x_t) = h^*(x_t) + \epsilon_t,$  where the function $h^*$ satisfies approximate realizability i.e. $\inf_{h \in \cH}  (1/N) \sum_{t=1}^N  \E_{x \sim P_t} (h^*(x) - h(x))^2 \leq \gamma,$ and ${\epsilon_t}_{t=1}^N$  are independent  random variables such that $\E[y_t \mid x_t] = h^*(x_t)$. Suppose it also holds $\max_t |y_t| \leq L$ and $\max_{x} |h^*(x)| \leq M$. Then, the least square solution satisfies with probability at least $1 - \delta$:
\begin{align*}
 \sum_{t=1}^N \E_{x \sim P_t} (\widehat{h}(x) - h^*(x))^2 &\leq 3 \gamma N + 64 (L+M)^2 \log(2 |\cH|/\delta). 
\end{align*}
\end{lemma} 

\newpage 
\section[Useful Foundational Results]{Useful Foundational Results \hfill \Coffeecup\Coffeecup\Coffeecup} \label{appen:foundation-results}

We provide the following result highlighting the necessary characteristics for specific examples of the Fenchel conjugate functions $\phi^*$.
\begin{proposition}[$\phi$-Divergence Bounds]\label{prop:divergence-loss-bounds}
    Let $V\in [0,\Vmax]^{\cardS}$ be any value function and fix a probability distribution $P^o\in\Delta(\cS)$. Define $h(y,\eta)=(\lambda \phi^*\left({(\eta-y)}/{\lambda}\right)  - \eta)$. Consider the following scalar convex optimization problem: $\inf_{\eta\in \Theta\subseteq\R} \E_{s \sim P^o} h(V(s),\eta)$. Let the maximum absolute value in $\Theta$ be less than or equal to $c_3$, let $|h(V(s),\eta)|\leq c_1$ for all $\eta\in\Theta$, and let $h(V(s),\eta)$ be $c_2$-Lipschitz in $\eta$; hold for some positive constants $c_1,c_2,c_3$.
    We have the following results for different forms of $\phi$:\\
    (i) Let \cref{fail-state-assump} hold. For TV distance i.e. $\phi(t) = |t-1|/2$, we have $\Theta\equiv [-\lambda/2,\lambda/2]$, hence $c_3=\lambda/2$, $c_1= 2\lambda + \Vmax$, and $c_2=2$.\\
    (ii) For chi-square divergence i.e. $\phi(t) = (t-1)^2$, we have $\Theta\equiv [-\lambda,2\Vmax+2\lambda]$, hence $c_3=2\Vmax+2\lambda$, $c_1= \lambda + (2\Vmax+4\lambda)(\frac{2\Vmax}{4\lambda}+2)$, and $c_2=(3+\frac{\Vmax}{\lambda})$.\\
    (iii) For KL divergence i.e. $\phi(t) = (t-1)^2$, we have $\Theta\equiv [\lambda,\Vmax+\lambda]$, hence $c_3=\Vmax+\lambda$, $c_1 = \lambda (\exp (\frac{\Vmax}{\lambda}) - 1)$, and $c_2=(\exp (\frac{\Vmax}{\lambda}) + 1)$.\\
    (iv) Fix $\alpha\in(0,1)$. For $\alpha$-CVaR i.e. $\phi(t) = \indic[0,1/\alpha)$, we have $\Theta\equiv [0,\Vmax/(1-\alpha)]$, hence $c_3=\Vmax/(1-\alpha)$, $c_1 = 2\Vmax/(\alpha(1-\alpha))$, and $c_2=1+\alpha^{-1}$.
\end{proposition}
\begin{proof}
    We first prove the statement for TV distance with $\phi(t) = |t-1|/2$. From $\phi$-divergence literature \citep{xu-panaganti-2023samplecomplexity}, we know \begin{equation*}
      \phi^*(s) =
        \begin{cases}
          -\frac{1}{2} & s\leq -\frac{1}{2},\\
          s & s\in [-\frac{1}{2},\frac{1}{2}]\\
          +\infty & s > \frac{1}{2}.
        \end{cases}.  
    \end{equation*}
    Thus, we have \begin{align*}
     \inf_{\eta\in \R} \E_{s \sim P^o} h(V(s),\eta)  &= \inf_{ \eta\in\R} ~ \E_{s\sim P^o} [\lambda \phi^*(\frac{\eta-V(s)}{\lambda})] - \eta  \\
    &\stackrel{(a)}{=} \inf_{ \eta\in\R, \frac{\eta-\min_s V(s)}{\lambda} \leq \frac{1}{2}} ~  \E_{s\sim P^o} [\lambda \max\{ \frac{\eta-V(s)}{\lambda}, -\frac{1}{2} \}] - \eta  \\
    &\stackrel{(b)}{=} \inf_{ \eta\in\R, \frac{\eta}{\lambda} \leq \frac{1}{2}} ~  \E_{s\sim P^o} [\lambda \max\{ \frac{\eta-V(s)}{\lambda}, -\frac{1}{2} \}] - \eta  \\
    &\stackrel{(c)}{=} \inf_{ \eta\in\R, \frac{\eta}{\lambda} \leq \frac{1}{2}} ~ \E_{s\sim P^o} [ (\eta-V(s) +\lambda/2)_+ ] - \lambda/2 - \eta  \\
    &\stackrel{(d)}{=} \inf_{ \eta'\in\R, \eta'\leq \lambda} ~ \E_{s\sim P^o} [ (\eta'-V(s))_+ ] - \eta' \\
    &\stackrel{(e)}{=} \inf_{ 0\leq \eta'\leq \lambda} ~ \E_{s\sim P^o} [ (\eta'-V(s))_+ ] - \eta', \numthis \label{eq:tv-inner-optimization-dual-form}
    \end{align*} 
    where $(a)$ follows by definition of $\phi^*$, $(b)$ by \cref{fail-state-assump}, $(c)$ by the fact $\max\{x,y\} = (x-y)_+ + y$ for any $x,y\in\R$, and $(d)$ follows by making the substitution $\eta=\eta'-\lambda/2$. Finally, for $(e)$, notice that since $V(s) \geq 0,$ $\E_{s\sim P^o} [ (\eta'-V(s))_+ ] - \eta'=-\eta' \geq 0$ holds when $\eta' \leq 0$. So $\inf_{\eta' \in (-\infty, 0]} \E_{s\sim P^o} [ (\eta'-V(s))_+ ] - \eta'=0$ is achieved at $\eta' = 0$.

    We immediately have $\Theta\equiv[-\lambda/2,\lambda/2]$ since $\eta=\eta'-\lambda/2$. Since $\eta \leq \lambda/2$ and $V(s)\leq\Vmax$, we further get $|h(V(s),\eta)|\leq 2\lambda + \Vmax$. For $\eta_1,\eta_2\in\Theta$, from the fact $|(x)_+ - (y)_+| \leq |(x-y)_+|\leq |x-y|$ we have $|h(V(s),\eta_1)-h(V(s),\eta_2)|\leq 2|\eta_1-\eta_2|$. This proves statement $(i)$.

    We now prove the statement for chi-square divergence with $\phi(t) = (t-1)^2$ following similar steps as before. From $\phi$-divergence literature \citep{xu-panaganti-2023samplecomplexity}, we know $\phi^*(s) = (s/2+1)_+^2 -1.  $
    Thus, we have \begin{align*}
     \inf_{\eta\in \R} \E_{s \sim P^o} h(V(s),\eta)  &= \inf_{ \eta\in\R} ~ \E_{s\sim P^o} [\lambda \phi^*(\frac{\eta-V(s)}{\lambda})] - \eta  \\
    &= \inf_{ \eta\in\R} ~  \E_{s\sim P^o} [\lambda (\frac{\eta-V(s)}{2\lambda} + 1)_+^2 ] - \lambda - \eta  \\
    &\stackrel{(f)}{=} \inf_{ \eta'\in\R} ~ \frac{1}{4\lambda}\E_{s\sim P^o} [ (\eta'-V(s))_+^2 ] + \lambda - \eta' \\
    &\stackrel{(g)}{=} \inf_{ \eta'\in[\lambda,2\Vmax+4\lambda]} ~ \frac{1}{4\lambda}\E_{s\sim P^o} [ (\eta'-V(s))_+^2 ] + \lambda - \eta' ,
    \end{align*} 
    where $(f)$ follows by making the substitution $\eta=\eta'-2\lambda$. Finally, for $(g)$, observe that the function $g(\eta')=\frac{1}{4\lambda}\E_{s\sim P^o} [ (\eta'-V(s))_+^2 ] + \lambda - \eta'$ is convex in the dual variable $\eta'$ and $\inf_{ \eta'\in\R} g(\eta') \leq 0$ since it is a Lagrangian dual variable.  Since $V(s) \geq 0,$ $\lambda-\eta'_* \leq 0$ where $\eta'_*$ is any solution of $\inf_{ \eta'\in\R} g(\eta') \leq 0$. When $\eta'\geq 2\Vmax + 4\lambda$, notice that $g(\eta')\geq \frac{1}{4\lambda}(\eta'^2 - 2(\Vmax + 2\lambda)\eta' + 4\lambda^2) \geq \lambda > 0,$ since $0\leq V(s) \leq \Vmax$.

    We immediately have $\Theta\equiv[-\lambda,2\Vmax+2\lambda]$ since $\eta=\eta'-2\lambda$. Since $\eta \leq 2\Vmax+2\lambda$ and $V(s)\geq 0$, we further get $|h(V(s),\eta)|\leq \lambda + (2\Vmax+4\lambda)(\frac{2\Vmax}{4\lambda}+2)$. For $\eta_1,\eta_2\in\Theta$, from the facts $|(x)_+ - (y)_+| \leq |(x-y)_+|\leq |x-y|$ and $|(x)_+^2 - (y)_+^2| = |(x)_+ - (y)_+| ((x)_+ + (y)_+)$,  we have $|h(V(s),\eta_1)-h(V(s),\eta_2)|\leq (3+(\Vmax))|\eta_1-\eta_2|$. This proves statement $(ii)$.

    We now prove the statement for KL divergence with $\phi(t) = t\log{t}$ following similar steps as before. From $\phi$-divergence literature \citep{xu-panaganti-2023samplecomplexity}, we know $\phi^*(s) = \exp(s-1).$
    Thus, we have \begin{align*}
     \inf_{\eta\in \R} \E_{s \sim P^o} h(V(s),\eta)  &= \inf_{ \eta\in\R} ~ \E_{s\sim P^o} [\lambda \phi^*(\frac{\eta-V(s)}{\lambda})] - \eta  \\
    &= \inf_{ \eta\in\R} ~  \E_{s\sim P^o} [\lambda \exp(\frac{\eta-V(s)}{\lambda} - 1) ]  - \eta  \\
    &\stackrel{(h)}{=} \inf_{ \eta'\in\R} ~ \lambda\E_{s\sim P^o} [ \exp(\frac{-\eta'-V(s)}{\lambda} - 1) ]  + \eta' \\
    &\stackrel{(j)}{=} \inf_{ \eta'\in[-\lambda-\Vmax,-\lambda]} ~ \lambda\E_{s\sim P^o} [ \exp(\frac{-\eta'-V(s)}{\lambda} - 1) ]  + \eta' ,
    \end{align*} 
    where $(h)$ follows by making the substitution $\eta=-\eta'$. Finally, for $(j)$, observe that the function $g(\eta')=\lambda\E_{s\sim P^o} [ \exp(\frac{-\eta'-V(s)}{\lambda} - 1) ]  + \eta'$ is convex in the dual variable $\eta'$ since it is a Lagrangian dual variable. From Calculus, the optimal $\eta' = -\lambda + \lambda\log\E_{s\sim P^o}\exp({-V(s)}/{\lambda})$. So $\eta'\in[-\lambda-\Vmax,-\lambda]$ since $0\leq V(s) \leq \Vmax$.

    We immediately have $\Theta\equiv[\lambda,\Vmax+\lambda]$ since $\eta=-\eta'$. Since $\eta \leq \Vmax+\lambda$ and $V(s)\geq 0$, we further get $|h(V(s),\eta)|\leq \lambda (\exp (\frac{\Vmax}{\lambda}) - 1)$. For $\eta_1,\eta_2\in\Theta$, from the fact $\exp(-x)$ is $1$-Lipschitz for $x\geq0$, we have $|h(V(s),\eta_1)-h(V(s),\eta_2)|\leq (\exp (\frac{\Vmax}{\lambda}) + 1)|\eta_1-\eta_2|$. This proves statement $(ii)$.

    We now prove the statement for $\alpha$-CVAR with $\phi(t) = \indic[0,1/\alpha)$. From $\phi$-divergence literature \citep{levy2020large}, we know 
     $ \phi^*(s) = (s)_+/\alpha .$ 
    Thus, we have \begin{align*}
     \inf_{\eta\in \R} \E_{s \sim P^o} h(V(s),\eta)  &= \inf_{ \eta\in\R} ~ \E_{s\sim P^o} [\lambda \phi^*(\frac{\eta-V(s)}{\lambda})] - \eta  \\
    &= \inf_{ \eta\in\R} ~ \frac{1}{\alpha}\E_{s\sim P^o} [ (\eta-V(s))_+ ] - \eta \\
    &\stackrel{(k)}{=} \inf_{ 0\leq \eta\leq \Vmax/(1-\alpha)} ~ \frac{1}{\alpha}\E_{s\sim P^o} [ (\eta-V(s))_+ ] - \eta. \numthis \label{eq:cvar-inner-optimization-dual-form}
    \end{align*} 
    For $(k)$, notice that since $V(s) \geq 0,$ $(1/\alpha)\E_{s\sim P^o} [ (\eta-V(s))_+ ] - \eta=-\eta \geq 0$ holds when $\eta \leq 0$. Also, since $V(s) \leq \Vmax,$ $(1/\alpha)\E_{s\sim P^o} [ (\eta-V(s))_+ ] - \eta \geq 0$ holds when $\eta \geq \Vmax/(1-\alpha)$. 

    We immediately have $\Theta\equiv[0,\Vmax/(1-\alpha)]$. We further get $|h(V(s),\eta)|\leq 2\Vmax/(\alpha(1-\alpha))$. For $\eta_1,\eta_2\in\Theta$, from the fact $|(x)_+ - (y)_+| \leq |(x-y)_+|\leq |x-y|$ we have $|h(V(s),\eta_1)-h(V(s),\eta_2)|\leq (1+\alpha^{-1})|\eta_1-\eta_2|$.
    This proves the final statement of this result.
\end{proof}

We now state and prove a generalization bound for empirical risk minimization when the data are not necessarily i.i.d. but adapted to a stochastic process. This result is of independent interest to more machine learning problems outside of the scope of this paper as well.
Furthermore, this result showcases better rate dependence on $N$, from $\cOtilde(1/\sqrt{N})$ to $\cOtilde(1/{N})$, than the classical result \cref{lem:erm-generalization-bound} \citep{shalev2014understanding}.
This result is not surprising and we refer to \citet[Theorems 7.6 \& 5.4]{van2015fast},  in the i.i.d. setting, for such $\cOtilde(1/{N})$ fast rates with bounded losses to empirical risk minimization and beyond.
\begin{proposition}[Online ERM Generalization Bound]
\label{prop:erm-dependent-data} 
Let $N>0$, $\delta \in (0, 1)$, let $\cX$ be an input space,  and let $\cY$ be the target functional space. Let $\cH\subseteq \cY$ be the given finite class of functions. Assume that for all $x\in\cX$ and $h \in \cH$ for loss function $l$ we have that $|l(h(x))| \leq c$ for some positive constant $c>0$. Given a dataset $\cD=\{x_i\}_{i=1}^N$, denote $\widehat{h}$ as the ERM solution, i.e. $\widehat{h}\leftarrow\argmin_{h\in\cH} \sum_{i=1}^N l(h(x_i))$. 
The dataset $\cD$ is generated as $x_t \sim P_t$  from some stochastic process $P_t$ that depends on the history $\{x_1, \dots, x_{t-1}\}$, where the function  $h^*_t\in \argmin_{f \in \cY} \E_{x\sim P_t}[l(f(x))]$ satisfies approximate realizability i.e. $$\inf_{h \in \cH}  \frac{1}{N} \sum_{t=1}^N  \E_{x_t \sim P_t} (l(h(x_t)) - l(h^*_t(x_t))) \leq \gamma,$$ 
and for all $x\in\cX$, $|l(h^*_t(x))| \leq c$.
Then, the ERM solution satisfies
\begin{align*}
 \sum_{t=1}^N \E_{x_t \sim P_t} l(\widehat{h}(x_t))-  \sum_{t=1}^N \E_{x_t \sim P_t} l(h^*_t(x_t)) &\leq 3 \gamma N + 48 c \log(2 \abs{\cH}/\delta)
\end{align*}  with probability at least $1 - \delta$.
\end{proposition} 
\begin{proof} We adapt the proof of least-squares generalization bound \citep[Lemma 3]{song2023hybrid} here for the empirical risk minimization generalization bound under online data collection.
Fix any function $h \in \cH$. We define the random variable $Z_t^h = l(h(x_t)) - l(h^*_t(x_t)).$ Immediately, we note $|Z_t^h| \leq 2c$ for all $t$. 
By definition of $h^*_t$, we have a non-negative first moment of $Z_t^h$:
\begin{align*}
\E_{P_t}[ Z^h_t] &= \E_{x_t \sim P_t} l(h(x_t)) - \E_{x_t \sim P_t} l(h^*_t(x_t)) . \numthis \label{eq:erm-first-moment}
\end{align*} 
By symmetrization, assuming $l(h^*_t(x_t))^2 \leq l(h(x_t))^2$, we have that 
\begin{align*}
0 \leq \E_{P_t} [(Z_t^h)^2] &\leq \E_{x_t \sim P_t}[ 2 l(h(x_t))^2 - 2\cdot l(h(x_t)) \cdot l(h^*_t(x_t))]\\
&\leq 2 |l(h(x_t))| \E_{x_t \sim P_t} (l(h(x_t))- l(h^*_t(x_t))) \\
&\leq 2 c \cdot \E_{x_t \sim P_t} (l(h(x_t))- l(h^*_t(x_t))).
\end{align*}
Similarly assuming $l(h^*_t(x_t))^2 \geq l(h(x_t))^2$, we get $0 \leq \E_{P_t} [(Z_t^h)^2] \leq 2 c \cdot \E_{x_t \sim P_t} (l(h(x_t))- l(h^*_t(x_t)))$. Thus, uniformly, we have 
\begin{align}
0 \leq \E_{P_t} [(Z_t^h)^2]
\leq 2 c \cdot \E_{x_t \sim P_t} (l(h(x_t))- l(h^*_t(x_t))). \numthis \label{eq:erm-second-moment}
\end{align}
We remark that \eqref{eq:erm-second-moment} is called \textit{Bernstein condition} \citep[Definition 5.1]{van2015fast} when all sampling distributions $P_t$'s are identical. This is one of the sufficient conditions on the loss functions to get $\cO(1/N)$-generalization bounds for empirical risk minimization.

Now, applying \cref{lem:freedman-ineq} with $\lambda \in [0, 1/4c]$ and $\delta > 0$, we have
\begin{align*}
\abs{\sum_{t=1}^N Z^h_t -\E_{P_t} [Z_t^h]} &\leq \lambda \sum_{t=1}^N (4c |\E_{P_t} [Z_t^h]|  + \E_{P_t} [(Z_t^h)^2]) + \frac{\log(2/\delta)}{\lambda}  \\
&\leq 6c \lambda \sum_{t=1}^N \E_{x_t \sim P_t} (l(h(x_t))- l(h^*_t(x_t))) + \frac{\log(2/\delta)}{\lambda}  
\end{align*}  with probability at least $1 - \delta$,  where the last inequality uses \eqref{eq:erm-first-moment} and \eqref{eq:erm-second-moment}. We set $\lambda = 1/12c$ in the above, we get for any $h \in \cH$, with probability at least $1 - \delta$:
\begin{align*}
\abs{\sum_{t=1}^N Z^h_t -\E_{P_t} [Z_t^h]}
\leq \frac{1}{2} \sum_{t=1}^N \E_{x_t \sim P_t} (l(h(x_t))- l(h^*_t(x_t))) + 12c\log(2 \abs{\cH}/\delta), 
\end{align*} by union bound over $h\in\cH$. Using \eqref{eq:erm-first-moment}, we rearrange the above to get: 
\begin{align}
\sum_{t=1}^N Z_t^h  \leq \frac{3}{2}  \sum_{t=1}^N \E_{x_t \sim P_t} (l(h(x_t))- l(h^*_t(x_t)))  + 12 c \log(2 \abs{\cH}/\delta) \label{eq:erm-rearrange-1}
\intertext{and} 
 \sum_{t=1}^N \E_{x_t \sim P_t} (l(h(x_t))- l(h^*_t(x_t))) \leq 2 \sum_{t=1}^N Z_t^h  + 24 c \log(2 \abs{\cH}/\delta). \label{eq:erm-rearrange-2}
\end{align}

Define the function $\widetilde{h} \in \argmin_{h \in \cH} \sum_{t=1}^N \E_{x_t \sim P_t} (l(h(x_t))- l(h^*_t(x_t)))$, which is independent of the dataset $\cD$. By \eqref{eq:erm-rearrange-1} for $\widetilde{h}$ and the approximate realizability assumption, we get 
\begin{align*}
\sum_{t=1}^N Z_t^{\widetilde{h}}  &\leq \frac{3}{2}  \sum_{t=1}^N \E_{x_t \sim P_t} (l(h(x_t))- l(h^*_t(x_t)))  + 12 c \log(2 \abs{\cH}/\delta)  \leq \frac{3}{2} \gamma N + 12 c \log(2 \abs{\cH}/\delta). 
\end{align*} By definitions  of $\widetilde{h}$ and the ERM function $\widehat{h}$, we have that
\begin{align*}
\sum_{t=1}^N Z_t^{\widehat{h}} &= \sum_{t=1}^N l(\widehat{h}(x_t))- l(h^*_t(x_t)) \leq \sum_{t=1}^N l(\widetilde{h}(x_t))- l(h^*_t(x_t)) = \sum_{t=1}^N Z_t^{\widetilde{h}}. 
\end{align*}
From the above two relations, we get 
\begin{align*}
\sum_{t=1}^N Z_t^{\widehat{h}} &\leq \frac{3}{2} \gamma N + 12 c \log(2  \abs{\cH}/\delta). 
\end{align*}
Now, using this and using \eqref{eq:erm-rearrange-2} for the function $\widehat{h}$, we get 
\begin{align*}
 \sum_{t=1}^N \E_{x_t \sim P_t} l(\widehat{h}(x_t))- l(h^*_t(x_t)) &\leq 2 \sum_{t=1}^N Z_t^{\widehat{h}}  + 24 c \log(2 \abs{\cH}/\delta) \leq 3 \gamma N + 48 c \log(2 \abs{\cH}/\delta) ,
\end{align*}
which holds with probability at least $1-\delta$. This completes the proof.
\end{proof}

We now state a useful result for an infinite-horizon discounted robust $\phi$-regularized Markov decision process $(\cS,\cA,r,P^o,\lambda,\gamma,\phi,d_0)$. This result helps our RPQ algorithm's policy search space to be the class of deterministic Markov policies.
\begin{proposition} \label{prop:infinite-horizon-opt-deterministic-policy}
    The robust regularized Bellman operator $\cT$ \eqref{eq:robust-regularized-bellman-eq-primal}  \begin{align*}
     &(\cT Q)(s,a) = r(s,a) +   \gamma \inf_{P_{s,a}\in\cP_{s,a}}\big(\E_{s'\sim P_{s,a}}[\max_{a'}Q(s',a')] + \lambda \dphi(P_{s,a},P^o_{s,a})\big),
    \end{align*} and the value function operator $(\cT_v V)(\cdot) = \max_a (\cT Q)(\cdot,a)$ are both $\gamma$-contraction operators w.r.t sup-norm. Moreover, their respective unique fixed points $Q^*_{\lambda}$ and $V^*_{\lambda}$, for optimal policy $\pi^*$,  achieve the optimal robust value $\max_{\pi} V^\pi_{\lambda}$. Furthermore, the robust regularized optimal  policy $\pi^*$ is a deterministic Markov policy satisfying $\pi^*(\cdot) = \argmax_a Q^*_{\lambda}(\cdot,a)$.
\end{proposition}
\begin{proof}
    The $\gamma$-contraction property of both operators directly follow from the fact $\inf_{x} p(x)-\inf_{x} q(x) \leq \sup_{x} (p(x) - q(x))$. Furthermore, this result is a direct corollary of \citep[Proposition 3.1]{yang2023avoiding} and \citep[Corollary 3.1]{iyengar2005robust}.
\end{proof}

We now state a similar result for a finite-horizon discounted robust $\phi$-regularized Markov decision process $(\cS,\cA,P^o=(P^o_h)_{h=0}^{H-1},r=(r_h)_{h=0}^{H-1},\lambda,H,\phi,d_0)$.
This result helps our HyTQ algorithm's policy search space to be the class of non-stationary deterministic Markov policies.

\begin{proposition} \label{prop:FH-opt-deterministic-policy}
    The robust regularized Bellman operator $\cT$ \eqref{eq:FH-robust-regularized-bellman-eq-primal} and the value function operator $\cT_v$ are as follows: \begin{align*}
     &(\cT Q_{h+1}) (s, a) = r_{h}(s, a) +   \inf_{P_{h,s,a}\in\cP_{h,s,a}}\big(\E_{s'\sim P_{h,s,a}}[\max_{a'}Q_{h+1}(s',a')] + \lambda \dphi(P_{h,s,a},P^o_{h,s,a})\big) \quad \text{and} \\
     &(\cT_v V_{h+1}) (s) = \max_a \bigg[r_{h}(s, a) +   \inf_{P_{h,s,a}\in\cP_{h,s,a}}\big(\E_{s'\sim P_{h,s,a}}[V_{h+1}(s')] + \lambda \dphi(P_{h,s,a},P^o_{h,s,a})\big) \bigg]. 
    \end{align*} 
    The optimal robust value $V^{*}_{h,\lambda}$ satisfies the following \textit{robust dynamic programming} procedure: Starting with $V^{*}_{H,\lambda}=0$, doing backward iteration of $\cT_v$, i.e.,  $V^*_{h,\lambda} = \cT_v V^*_{h+1,\lambda}$, we get $V^*_{h,\lambda}$ for all $h\in[H]$.
    Furthermore, the robust regularized optimal  policy $\pi^*$ is a non-stationary deterministic Markov policy satisfying $\pi^*_h(\cdot)=\argmax_{a}Q^*_{h,\lambda}(\cdot,a)$ for all $h\in[H]$ where \begin{align*}
        Q^*_{h,\lambda}(\cdot,a) = r_{h}(s, a) +   \inf_{P_{h,s,a}\in\cP_{h,s,a}}\big(\E_{s'\sim P_{h,s,a}}[V^{*}_{h+1}(s')] + \lambda \dphi(P_{h,s,a},P^o_{h,s,a})\big).
    \end{align*}
    Moreover, as $V^{*}_{H,\lambda}=0=Q^{*}_{H,\lambda}$, it suffices to backward iterate $\cT$, i.e., do  $Q^*_{h,\lambda} = \cT Q^*_{h+1,\lambda}$ to get $Q^*_{h,\lambda}$ for all $h\in[H]$.
\end{proposition}
\begin{proof}
    We start with the optimal robust value definition $V^{*}_{h,\lambda}=\max_{\pi} V^{\pi}_{h,\lambda} = \max_{\pi} \inf_{P\in\cP}V^{h,\pi}_{P,r^\lambda_{h}}$. The value function claims in this statement are direct consequences of \citep[Theorem 2.1 \& 2.2]{iyengar2005robust} and \citep[Theorem 2]{zhang2023regularized} with the reward function $r^\lambda_{h}$. 
    
    It remains to prove $Q^*$ dynamic programming with $\cT$. That is, we establish  $V^*_{h,\lambda}(\cdot) = \max_a Q^*_{h,\lambda}(\cdot,a)$ for all $h\in[H]$ with the dynamic programming of $\cT$.
    We use induction to prove this. The base case is trivially true since  $V^{*}_{H,\lambda}=0=Q^{*}_{H,\lambda}$. By $\cT$, we have \begin{align*}
        Q^*_{h,\lambda}(s,a) &= (\cT Q^*_{h+1,\lambda})(s,a) \\
        &=r_{h}(s, a) +   \inf_{P_{h,s,a}\in\cP_{h,s,a}}\big(\E_{s'\sim P_{h,s,a}}[\max_{a'}Q^*_{h+1}(s',a')] + \lambda \dphi(P_{h,s,a},P^o_{h,s,a})\big) \\
        &=r_{h}(s, a) +   \inf_{P_{h,s,a}\in\cP_{h,s,a}}\big(\E_{s'\sim P_{h,s,a}}[V^*_{h+1}(s')] + \lambda \dphi(P_{h,s,a},P^o_{h,s,a})\big),
    \end{align*} where the last equality follows by the induction hypothesis $V^*_{h+1,\lambda}(\cdot) = \max_a Q^*_{h+1,\lambda}(\cdot,a)$. Maximizing this both sides with action $a$ and by the dynamic program $V^*_{h,\lambda} = \cT_v V^*_{h+1,\lambda}$, we get $V^*_{h,\lambda}(\cdot) = \max_a Q^*_{h,\lambda}(\cdot,a)$. This completes the proof of this result.
\end{proof}

\newpage
\section[Offline Robust $\phi$-regularized RL Results]{Offline Robust $\phi$-regularized RL Results \hfill \Coffeecup\Coffeecup\Coffeecup}
\label{appendix:infinite-horizon}

In this section, we set $\Vmax=1/(1-\gamma)$ whenever we use results from \cref{prop:divergence-loss-bounds}. In the following, we use constants $c_1,c_2,c_3$ from \cref{prop:divergence-loss-bounds}.

We first prove \cref{prop:robust-bellman-dual} that directly follows from \cref{lem:penalized-dro-inner-problem-solution}.
\begin{proof}[\textbf{Proof of \cref{prop:robust-bellman-dual}}]
   For each $(s, a)$, consider the optimization problem in~\eqref{eq:robust-regularized-bellman-eq-primal} \begin{align*}
    \inf_{P_{s,a}\in\cP_{s,a}}&\big(\E_{s'\sim P_{s,a}}[V(s')] + \lambda \dphi(P_{s,a},P^o_{s,a})\big) = - \sup_{P_{s,a}\in\cP_{s,a}}\big(\E_{s'\sim P_{s,a}}[-V(s')] - \lambda \dphi(P_{s,a},P^o_{s,a})\big) \\
    &\stackrel{(a)}{=} -\inf_{\eta'\in \R} (\lambda \E_{s' \sim P^{o}_{s,a}}[\phi^*\left(\frac{-\eta'-V(s')}{\lambda}\right)]  + \eta' ) \\
    &\stackrel{(b)}{=} -\inf_{\eta\in \R} (\lambda \E_{s' \sim P^{o}_{s,a}}[\phi^*\left(\frac{\eta-V(s')}{\lambda}\right)]  - \eta ) \\
    &\stackrel{(c)}{=} -\inf_{\eta\in \Theta} (\lambda \E_{s' \sim P^{o}_{s,a}}[\phi^*\left(\frac{\eta-V(s')}{\lambda}\right)]  - \eta ),
 \end{align*} where $(a)$ follows from \cref{lem:penalized-dro-inner-problem-solution}, $(b)$ by setting $\eta=-\eta'$, and $(c)$ by \cref{prop:divergence-loss-bounds}. This completes the proof.
\end{proof}

We now prove \cref{prop:functional-opt-phi-divergence} which mainly follows from Lemma~\ref{lem:integral-min-exchange-rockafellar}.

\begin{proof}[Proof of \cref{prop:functional-opt-phi-divergence}]
Since the conjugate function $\phi^*(\cdot)$ is continuous, define a continuous function in $\eta$ for each $(s,a)\in\ScA$ $h((s,a),\eta)= (\lambda \E_{s'\sim P^o_{s,a}}\phi^*\left({(\eta-\max_{a'} f(s',a'))}/{\lambda}\right)  - \eta)$. We observe $h((s,a),\eta)$ in $(s,a)\in\ScA$ is $\Sigma(\ScA)$-measurable for each $\eta\in\Theta$, where $\Theta$ is a bounded real line. 
This lemma now directly follows by similar arguments in the proof of \citet[Lemma 1]{panaganti-rfqi}. 
\end{proof}

Now we state a result and provide its proof for the empirical risk minimization on the dual parameter.

\begin{proposition}[Dual Optimization Error Bound] \label{prop:erm-high-prob-bound}
Let $\widehat{g}_{f}$ be the dual optimization parameter from \cref{alg:infinite-horizon-RPQ-Algorithm} (Step 4) for the state-action value function $f$ and let $\cT_{g}$ be as defined in \eqref{eq:Tg}. With probability at least $1-\delta$, we have
\[
\sup_{f\in\cF} \|\cT f -  \cT_{\widehat{g}_{f}} f\|_{1,\mu} \leq 2\gamma c_2 c_3\sqrt{\frac{2\log(|\cG|)}{N}} +  5c_1  \sqrt{\frac{2\log(8|\cF|/\delta)}{N}}  + \gamma \epsilon_{\cG}. 
\]
\end{proposition}

\begin{proof}
We adapt the proof from \citet[Lemma 6]{panaganti-rfqi}.
We first fix $f\in\cF$. We will also invoke union bound for the supremum here. We recall from~\eqref{eq:L-dual-empirical-loss} that $\widehat{g}_f = \argmin_{g \in \cG} \widehat{L}_{\mathrm{dual}}(g; f)$. From  the robust Bellman equation, we directly obtain
\begingroup
\allowdisplaybreaks
\begin{align*}
    \|\cT_{\widehat{g}_{f}} f - \cT f \|_{1,\mu} &= \gamma (\E_{s,a\sim\mu} | \E_{s'\sim P^o_{s,a}} (\lambda \phi^*({(\widehat{g}_{f}(s,a)-\max_{a'} f(s',a'))}/{\lambda})  - \widehat{g}_{f}(s,a)) \\&\hspace{2cm}- \inf_{\eta \in\Theta}  (\lambda \E_{s'\sim P^o_{s,a}}\phi^*\left({(\eta-\max_{a'} f(s',a'))}/{\lambda}\right)  - \eta) | ) \\
    &\stackrel{(a)}{=} \gamma (\E_{s,a\sim\mu}\E_{s'\sim P^o_{s,a}} (\lambda \phi^*({(\widehat{g}_{f}(s,a)-\max_{a'} f(s',a'))}/{\lambda})  - \widehat{g}_{f}(s,a)) \\&\hspace{1cm}- \E_{s,a\sim\mu} [ \inf_{\eta \in\Theta}  (\lambda \E_{s'\sim P^o_{s,a}}\phi^*\left({(\eta-\max_{a'} f(s',a'))}/{\lambda}\right)  - \eta)  ) ]  ) \\
    &\stackrel{(b)}{=} \gamma (\E_{s,a\sim\mu,s'\sim P^o_{s,a}} (\lambda \phi^*({(\widehat{g}_{f}(s,a)-\max_{a'} f(s',a'))}/{\lambda})  - \widehat{g}_{f}(s,a)) \\&\hspace{1cm}- \inf_{g \in L^1(\mu)} \E_{s,a\sim\mu,s'\sim P^o_{s,a}} (\lambda \phi^*({(g(s,a)-\max_{a'} f(s',a'))}/{\lambda})  - g(s,a))  )  \\
    &= \gamma (\E_{s,a\sim\mu,s'\sim P^o_{s,a}} (\lambda \phi^*({(\widehat{g}_{f}(s,a)-\max_{a'} f(s',a'))}/{\lambda})  - \widehat{g}_{f}(s,a))\\&\hspace{1cm}- \inf_{g\in \cG} \E_{s,a\sim\mu,s'\sim P^o_{s,a}} (\lambda \phi^*({(g(s,a)-\max_{a'} f(s',a'))}/{\lambda})  - g(s,a))  )
    \\ &\hspace{1cm} + \gamma (\inf_{g\in \cG} \E_{s,a\sim\mu,s'\sim P^o_{s,a}} (\lambda \phi^*({(g(s,a)-\max_{a'} f(s',a'))}/{\lambda})  - g(s,a)) 
    \\&\hspace{1.5cm}- \inf_{g \in L^1(\mu)} \E_{s,a\sim\mu,s'\sim P^o_{s,a}} (\lambda \phi^*({(g(s,a)-\max_{a'} f(s',a'))}/{\lambda})  - g(s,a))  )\\
    &\stackrel{(c)}{\leq} \gamma (\E_{s,a\sim\mu,s'\sim P^o_{s,a}} (\lambda \phi^*({(\widehat{g}_{f}(s,a)-\max_{a'} f(s',a'))}/{\lambda})  - \widehat{g}_{f}(s,a)) \\&\hspace{1cm}- \inf_{g\in \cG} \E_{s,a\sim\mu,s'\sim P^o_{s,a}} (\lambda \phi^*({(g(s,a)-\max_{a'} f(s',a'))}/{\lambda})  - g(s,a))  )  + \gamma \epsilon_{\cG}
    \\&\stackrel{(d)}{\leq} 2\gamma c_2 c_3\sqrt{\frac{2\log(|\cG|)}{N}} +  5c_1  \sqrt{\frac{2\log(8/\delta)}{N}}  + \gamma \epsilon_{\cG}.
\end{align*}
\endgroup
$(a)$ follows since $\inf_{g} h(g) \leq h(\widehat{g}_{f})$. $(b)$ follows from \cref{prop:functional-opt-phi-divergence}. $(c)$ follows from the approximate dual realizability assumption (Assumption \ref{assum-dual-realizability}). 

For $(d)$, we consider the loss function $l(g,(s,a,s'))=\lambda \phi^*\left({(g(s,a)-\max_{a'}f(s',a'))}/{\lambda}\right)  - g(s,a) $ (for e.g. $l(g,(s,a,s'))=[(g(s,a) + 2\lambda -\max_{a'}f(s',a'))^2_+]/{4\lambda} -\lambda - g(s,a)$) and dataset $\cD=\{s_i,a_i,s_i'\}_{i=1}^N$. Since $f\in\cF$ and $g\in\cG$,  we note that $|l(g,(s,a,s'))|\leq c_1$, where the value of $c_1>0$ depend on specific forms of $\phi^*$ as demonstrated in \cref{prop:divergence-loss-bounds}. 
Furthermore, take $l(g,(s,a,s'))$ to be $c_2$-Lipschitz in $g$ and $|g(s,a)| \leq c_3$, since $g\in\cG$, for some positive constants $c_2$ and $c_3$. Again, these constants depend on specific forms of $\phi^*$ as demonstrated in \cref{prop:divergence-loss-bounds}.
With these insights, we can apply the empirical risk minimization result in  Lemma \ref{lem:erm-generalization-bound} to get $(d)$.

With union bound, with probability at least $1-\delta$, we finally get \begin{align*}
    &\sup_{f\in\cF}\|\cT f -  \cT_{\widehat{g}_{f}} f\|_{1,\mu} \leq 2\gamma c_2 c_3\sqrt{\frac{2\log(|\cG|)}{N}} +  5c_1  \sqrt{\frac{2\log(8|\cF|/\delta)}{N}}  + \gamma \epsilon_{\cG},
\end{align*}
which concludes the proof.
\end{proof}


We next prove the least-squares generalization bound for the RFQI algorithm. 

\begin{proposition}[Least squares generalization bound] \label{prop:least-squares-generalization-bound}
Let  $\widehat{f}_{g}$ be the least-squares solution from \cref{alg:infinite-horizon-RPQ-Algorithm} (Step 5) for the state-action value function $f$ and  dual variable function  $g$. Let $\cT_{g}$ be as defined in \eqref{eq:Tg}. Then, with probability at least $1-\delta$, we have 
\begin{align*}
    \sup_{f\in\cF} \sup_{g\in\cG} \|  \cT_{g} f - \widehat{f}_{g} \|_{2,\mu}
    &\leq \sqrt{6 \epsilon_{\cF}} + \sqrt{\frac{2}{(1-\gamma)^2} + 18(1+\gamma c_1)} \sqrt{\frac{18 \log(2|\cF||\cG|/\delta)}{N}}.
\end{align*}
\end{proposition}

\begin{proof}
We adapt the least-squares generalization bound given in~\citet[Lemma A.11]{agarwal2019reinforcement} to our setting. We recall from \eqref{eq:L-rob-Q-empirical-loss} that $\widehat{f}_{g}=\argmin_{Q\in\cF}\widehat{L}_{\mathrm{robQ}}(Q;f, g)$. We first fix functions $f\in\cF$ and $g\in\cG$. For any function $f'\in\cF$, we define random variables $z_i^{f'}$ as
\begin{align*}
    z_i^{f'} = \left( f'(s_i,a_i) - y_{i} \right)^2  - \left( (\cT_{g} f)(s_i,a_i) - y_{i}  \right)^2,
\end{align*} 
where $y_{i} =  r_i - \gamma \lambda \phi^*({(g(s_i,a_i)-\max_{a'} f(s_i',a'))}/{\lambda})  + \gamma g(s_i,a_i)$, and $(s_{i}, a_{i}, s'_{i}) \in \cD$ with $(s_{i}, a_{i}) \sim \mu, s'_{i} \sim P^{o}_{s_{i}, a_{i}}$. It is straightforward to note that for a given $(s_{i}, a_{i})$, we have $ \E_{s'_i \sim P^o_{s_i,a_i}} [y_{i}] = (\cT_{g} f)(s_{i}, a_{i})$. We note the randomness of $z_i^{f'}$ given $f,f'\in\cF$ and $g\in\cG$ is from the dataset pairs $(s_i,a_i,s_i')$.

Since $f, f'\in\cF$ and $g\in\cG$, from \cref{prop:divergence-loss-bounds}, we write both $(\cT_{g}f)(s_i,a_i),y_{i} \leq 1+\gamma c_1$, where the value of $c_1>0$ depend on specific forms of $\phi^*$.
Using this, we obtain the first moment and an upper-bound for the second moment of $z_i^{f'}$ as follows:
\begin{align*}
    \E_{s'_i \sim P^o_{s_i,a_i}} [z_i^{f'}] &= \E_{s'_i \sim P^o_{s_i,a_i}} [(   f'(s_i,a_i) - (\cT_{g} f)(s_i,a_i)) \cdot (f'(s_i,a_i) + (\cT_{g} f)(s_i,a_i)  -2 y_{i} ) ] \\
    &= (   f'(s_i,a_i) - (\cT_{g} f)(s_i,a_i))^{2}, \\
    \E_{s'_i \sim P^o_{s_i,a_i}}  [(z_i^{f'})^2] &= \E_{s'_i \sim P^o_{s_i,a_i}} [(   f'(s_i,a_i) - (\cT_{g} f)(s_i,a_i))^{2} \cdot (f'(s_i,a_i) + (\cT_{g} f)(s_i,a_i)  -2 y_{i} )^{2} ] \\
     &= ( f'(s_i,a_i) - (\cT_{g} f)(s_i,a_i))^{2} \cdot \E_{s'_i \sim P^o_{s_i,a_i}} [  (f'(s_i,a_i) + (\cT_{g} f)(s_i,a_i)  -2 y_{i} )^{2} ] \\
     &\leq C_{1} ( f'(s_i,a_i) - (\cT_{g} f)(s_i,a_i))^{2} ,
\end{align*} 
where $C_{1}=\frac{2}{(1-\gamma)^2} + 18(1+\gamma c_1)$. This immediately implies that
\begin{align*}
     \E_{s_i,a_i\sim \mu, s'_i \sim P^o_{s_i,a_i}} [z_i^{f'}] &= \norm{\cT_{g}f -f'}^{2}_{2,\mu}, \\ 
     \E_{s_i,a_i\sim \mu, s'_i \sim P^o_{s_i,a_i}} [(z_i^{f'})^2] &\leq C_{1} \norm{\cT_{g}f -f'}^{2}_{2,\mu}.
\end{align*}
From these calculations, it is also straightforward to see that $|z_i^{f'}-\E_{s_i,a_i\sim \mu, s'_i \sim P^o_{s_i,a_i}} [z_i^{f'}]| \leq 2 C_{1}$ almost surely.

Now, using  the Bernstein's inequality (Lemma \ref{lem:bernstein-ineq}), together with a union bound over all $f'\in \cF$, with probability at least $1-\delta$, we have
\begin{align} 
\label{eq:least-squares-bern-main}
    |\|  \cT_{g} f - f' \|_{2,\mu}^2 - \frac{1}{N} \sum_{i=1}^N z_i^{f'}| \leq \sqrt{\frac{2 C_{1}     \|  \cT_{g} f - f' \|_{2,\mu}^2 \log(2|\cF|/\delta)}{N}} + \frac{2C_{1} \log(2|\cF|/\delta)}{3N},
\end{align} 
for all $f'\in \cF$. 
This expression coincides with \citet[Eq.(15)]{panaganti-rfqi}. Thus, following the proof of \citet[Lemma 7]{panaganti-rfqi}, we finally get 
\begin{align}
     \|  \cT_{g} f - \widehat{f}_{g} \|_{2,\mu}^2
    &\leq 6 \epsilon_{\cF} + \frac{9 C_{1} \log(4|\cF|/\delta)}{N}. \label{eq:least-squares-square-bd}
\end{align}
We note a fact $\sqrt{x+y}\leq \sqrt{x}+\sqrt{y}$. Now, using union bound for $f\in\cF$ and $g\in\cG$, with probability at least $1-\delta$, we finally obtain 
\begin{align*}
    \sup_{f\in\cF} \sup_{g\in\cG} \|  \cT_{g} f - \widehat{f}_{g} \|_{2,\mu}
    &\leq \sqrt{6 \epsilon_{\cF}} + \sqrt{\frac{18 C_{1} \log(2|\cF||\cG|/\delta)}{N}}.
\end{align*}
This completes the least-squares generalization bound analysis for the robust regularized Bellman updates.
\end{proof}



We are now ready to prove the main theorem.

\subsection[Proof of \cref{thm:infinite-horizon-phi-divergence-guarantee}]{Proof of \cref{thm:infinite-horizon-phi-divergence-guarantee} \hfill \Coffeecup\Coffeecup\Coffeecup}
\label{appendix:thm:infinite-horizon-phi-divergence-guarantee}

\begin{theorem}[Restatement of \cref{thm:infinite-horizon-phi-divergence-guarantee}] \label{thm:restated:infinite-horizon-phi-divergence-guarantee}
Let \cref{assum-concentra-condition,assum-bellman-completion,assum-dual-realizability} hold. Let $\pi_{K}$ be the RPQ algorithm policy after $K$ iterations. Then, for any $\delta\in(0,1)$, with probability at least  $1 - \delta$,  we have 
\begin{align*}
	V^{\pi^*} - V^{\pi_{K}} \leq & \frac{2\gamma^K }{(1-\gamma)^2} + \frac{2\sqrt{C} }{(1-\gamma)^2} (2\gamma c_2 c_3\sqrt{\frac{2\log(|\cG|)}{N}} +  5c_1  \sqrt{\frac{2\log(8|\cF|/\delta)}{N}}  + \gamma \epsilon_{\cG})
    \\&+ \frac{2\sqrt{C} }{(1-\gamma)^2} (\sqrt{6 \epsilon_{\cF}} + \sqrt{\frac{2}{(1-\gamma)^2} + 18(1+\gamma c_1)} \sqrt{\frac{18 \log(2|\cF||\cG|/\delta)}{N}}). \end{align*}
\end{theorem}

\begin{proof}
We let $V_k(s) = Q_k(s,\pi_k(s))$ for every $s\in\cS$. Since $\pi_k$ is the greedy  policy w.r.t $Q_k$, we also have $V_k(s) = Q_k(s,\pi_k(s)) = \max_a Q_k(s,a)$. We recall that $V^*=V^{\pi^*}$ and $Q^*=Q^{\pi^*}$. We also recall from \cref{sec:infinite-horizon-offline-robust-rl} that $Q^{\pi^*}$ is a fixed-point of the robust Bellman operator $\cT$ defined in~\eqref{eq:robust-regularized-bellman-eq-primal}. We also note that the same holds true for any stationary deterministic policy $\pi$ from~\citet{yang2023avoiding} that $Q^{\pi}$ satisfies $Q^{\pi}(s, a) = r(s, a) + \gamma \min_{P_{s,a} \ll P^o_{s,a}} (\mathbb{E}_{s' \sim P_{s,a}} [ V^{\pi}(s')] + \lambda\dphi(P_{s,a}, P^o_{s,a})).$ 
We now adapt the proof of \citet[Theorem 1]{panaganti-rfqi} using the RRBE in its primal form \eqref{eq:robust-regularized-bellman-eq-primal} directly instead of its dual form \eqref{eq:robust-regularized-bellman-eq-dual}. 

We first characterize the performance decomposition between $V^{\pi^*}$ and ${V}^{\pi_K}$. We recall the initial state distribution $d_0$. Since $V^{\pi^*}(s) \geq V^{\pi_K}(s)$ for any $s\in\cS$, we observe that
\begin{align}
	0\leq&\E_{s_0\sim d_0}[V^{\pi^*}(s_0) - {V}^{\pi_K}(s_0)] = \E_{s_0\sim d_0}[ (V^{\pi^*}(s_0) -  V_K(s_0) )  - (V^{\pi_K}(s_0) -  V_K(s_0) )]\nn\\ 
	&=  \E_{s_0\sim d_0}[(Q^{\pi^*}(s_0,\pi^*(s_0)) -  Q_K(s_0,\pi_K(s_0)) ) - (Q^{\pi_K}(s_0,\pi_K(s_0)) -  Q_K(s_0,\pi_K(s_0)))] \nn\\ 
	&\stackrel{(a)}{\leq} \E_{s_0\sim d_0}[Q^{\pi^*}(s_0,\pi^*(s_0)) -  Q_K(s_0,\pi^*(s_0)) +  Q_K(s_0,\pi_K(s_0)) - Q^{\pi_K}(s_0,\pi_K(s_0)) ] \nn\\
	&= \E_{s_0\sim d_0}[ Q^{\pi^*}(s_0,\pi^*(s_0)) -  Q_K(s_0,\pi^*(s_0)) +  Q_K(s_0,\pi_K(s_0)) - Q^{\pi^*}(s_0,\pi_K(s_0))
	\nn\\&\hspace{5cm}+ Q^{\pi^*}(s_0,\pi_K(s_0)) - Q^{\pi_K}(s_0,\pi_K(s_0)) ]
	\nn\\
	&\stackrel{(b)}{\leq} \E_{s_0\sim d_0}[ Q^{\pi^*}(s_0,\pi^*(s_0)) -  Q_K(s_0,\pi^*(s_0)) +  Q_K(s_0,\pi_K(s_0)) - Q^{\pi^*}(s_0,\pi_K(s_0)) \nn\\&\hspace{2cm}+ 
    \gamma [\min_{P_{s_0,\pi_K(s_0)} \ll P^o_{s_0,\pi_K(s_0)}} (\mathbb{E}_{s_1 \sim P_{s_0,\pi_K(s_0)}} [ V^{\pi^*}(s_1)] + \lambda\dphi(P_{s_0,\pi_K(s_0)}, P^o_{s_0,\pi_K(s_0)})) \nn\\&\hspace{3cm}- \min_{P_{s_0,\pi_K(s_0)} \ll P^o_{s_0,\pi_K(s_0)}} (\mathbb{E}_{s_1 \sim P_{s_0,\pi_K(s_0)}} [ V^{\pi_K}(s_1)] + \lambda\dphi(P_{s_0,\pi_K(s_0)}, P^o_{s_0,\pi_K(s_0)}))] ] \nn\\
	&\stackrel{(c)}{\leq} \E_{s_0\sim d_0}[|Q^{\pi^*}(s_0,\pi^*(s_0)) -  Q_K(s_0,\pi^*(s_0))|] + \E_{s_0\sim d_0}[|Q^{\pi^*}(s_0,\pi_K(s_0)) -  Q_K(s_0,\pi_K(s_0))|  ]
	\nn\\&\hspace{4cm}+ \gamma \E_{s_0\sim d_0}\E_{s_1\sim P^{\pi_K,\min}_{s_0,\pi_{K}(s_0)}}(  |V^{\pi^*}(s_1)- V^{\pi_K}(s_1)| ) \nn\\
    &\stackrel{(d)}{\leq} \sum_{h=0}^\infty \gamma^h \cdot\bigg(\E_{s\sim d_{h,\pi_{K}}} [ |Q^{\pi^*}(s,\pi^*(s)) -  Q_{K}(s,\pi^*(s))| + |Q^{\pi^*}(s,\pi_K(s)) -  Q_{K}(s,\pi_K(s))| ] \bigg),  \label{eq:perf-diff-1} 
\end{align}
where $(a)$ follows from the fact that $\pi_{K}$ is the greedy policy with respect to $Q_{K}$, $(b)$ from the Bellman equations, and $(c)$ from the following definition \[P^{\pi_K,\min}_{s,\pi_{K}(s)}\in\argmin_{P_{s,\pi_K(s)} \ll P^o_{s,\pi_K(s)}} (\mathbb{E}_{s' \sim P_{s,\pi_K(s)}} [ V^{\pi_K}(s')] + \lambda\dphi(P_{s,\pi_K(s)}, P^o_{s,\pi_K(s)})).\]
We note that this worse-case model distribution can be non-unique and we just pick one by an arbitrary deterministic rule. We emphasize that this model distribution  is used only in analysis which is not required in the algorithm.
Finally, $(d)$ follows with telescoping over $|V^{\pi^*}-V^{\pi_K}|$ by defining a state distribution $d_{h,\pi_K}\in \Delta(\cS)$, for all natural numbers $h\geq 0$, as
\begin{equation*}
  d_{h,\pi_{K}} =
    \begin{cases}
      d_0 & \text{if $h=0$},\\
      P^{\pi_K,\min}_{s',\pi_{K}(s')} & \text{otherwise, with } s'\sim d_{h-1,\pi_{K}}.
    \end{cases}       
\end{equation*}
We note that such state distribution proof ideas are commonly used in the offline RL literature \citep{agarwal2019reinforcement,panaganti-rfqi,bruns2023robust,zhang2023regularized}.

For~\eqref{eq:perf-diff-1}, with the $\nu$-norm notation i.e. $\|f\|_{p,\nu}^2 = (\E_{s,a\sim\nu} |f(s,a)|^p)^{1/p}$ for any $\nu\in\Delta(\ScA)$, we have 
\begin{align}  
\label{eq:thm-bound-part-1}
&\E_{s_0\sim d_0}[{V}^{\pi^*}] - \E_{s_0\sim d_0}[V^{\pi_K}] \leq \sum_{h=0}^\infty \gamma^h \bigg(    \|Q^{\pi^*} -  Q_{K}\|_{1,d_{h,\pi_{K}}\circ \pi^* } + \|Q^{\pi^*} -  Q_{K}\|_{1,d_{h,\pi_{K}}\circ \pi_K } \bigg), 
\end{align} 
where the state-action distributions are $d_{h,\pi_{K}}\circ \pi^*(s,a)\propto d_{h,\pi_{K}}(s) \indic (a=\pi^*(s))$ and $d_{h,\pi_{K}}\circ \pi_K(s,a)\propto d_{h,\pi_{K}}(s) \indic (a=\pi_K(s))$.
We now analyze the above two terms treating either $d_{h,\pi_{K}}\circ \pi^*$ or $d_{h,\pi_{K}}\circ \pi_K$ as a state-action distribution $\nu$ satisfying Assumption \ref{assum-concentra-condition}. First, considering any $s,a\sim\nu$ satisfying  $Q^{\pi^*}(s,a) \geq  Q_{K}(s,a)$ we have
\begin{align}
    0\leq&Q^{\pi^*}(s,a) -  Q_{K}(s,a) 
    \leq Q^{\pi^*}(s,a) -  \cT Q_{K-1}(s,a) + |\cT Q_{K-1}(s,a) -  Q_{K}(s,a)| \nn
    \\&\leq Q^{\pi^*}(s,a) -  \cT Q_{K-1}(s,a) + \|\cT Q_{K-1} -  Q_{K}\|_{1,\nu} \nonumber \\
    &\stackrel{(e)}{\leq} Q^{\pi^*}(s,a) -  \cT Q_{K-1}(s,a) + \sqrt{C}\|\cT Q_{K-1} -  Q_{K}\|_{1,\mu} \nonumber \\
    &\stackrel{(f)}{=}    
    \gamma [\min_{P_{s,a} \ll P^o_{s,a}} (\mathbb{E}_{s' \sim P_{s,a}} [ \max_{a'} Q^{\pi^*}(s',a')] + \lambda\dphi(P_{s,a}, P^o_{s,a})) \nn\\&\hspace{3cm}- \min_{P_{s,a} \ll P^o_{s,a}} (\mathbb{E}_{s' \sim P_{s,a}} [ \max_{a'} Q_{K-1}(s',a')] + \lambda\dphi(P_{s,a}, P^o_{s,a}))] \nn\\&\hspace{6cm}+ \sqrt{C}\|\cT Q_{K-1} -  Q_{K}\|_{1,\mu} \nn \\
    &\stackrel{(g)}{\leq} \gamma( 
    \E_{s'\sim P^{Q_{K-1},\min}_{s,a}}   (\max_{a'} Q^{\pi^*}(s',a')-\max_{a'} Q_{K-1}(s',a'))) + \sqrt{C}\|\cT Q_{K-1} -  Q_{K}\|_{1,\mu} \nn \\
    &\stackrel{(h)}{\leq} \gamma ( \E_{s'\sim P^{Q_{K-1},\min}_{s,a}} \max_{a'} |Q^{\pi^*}(s',a') -  Q_{K-1}(s',a')| )     + \sqrt{C}\|\cT Q_{K-1} -  Q_{K}\|_{1,\mu} , \label{eq:qstar-qk-bound-pt1}
\end{align} 
where $(e)$ follows by the concentrability assumption (Assumption \ref{assum-concentra-condition}), $(f)$ from Bellman equation, operator $\cT$,
$(g)$ follows, similarly as step $(c)$, from the following definition \[P^{Q_{K-1},\min}_{s,a}\in\argmin_{P_{s,a} \ll P^o_{s,a}} (\mathbb{E}_{s' \sim P_{s,a}} [  \max_{a'} Q_{K-1}(s',a')] + \lambda\dphi(P_{s,a}, P^o_{s,a})).\] We again emphasize that this model distribution is analysis-specific and we just pick one by an arbitrary deterministic rule since it may not be unique. $(h)$ follows by the fact $|\sup_{x} p(x)-\sup_{x} q(x)| \leq \sup_{x} |p(x) - q(x)|$. Now, by replacing $P^{Q_{K-1},\min}$ with $P^{Q^{\pi^*},\min}$ in step $(g)$ and  repeating the steps for any $s,a\sim\nu$ satisfying  $Q^{\pi^*}(s,a) \leq  Q_{K}(s,a)$,  we get
\begin{align}
    0\leq Q_{K}(s,a) - Q^{\pi^*}(s,a)\leq \gamma ( \E_{s'\sim P^{Q^{\pi^*},\min}_{s,a}} \max_{a'} |Q^{\pi^*}(s',a') -  Q_{K-1}(s',a')| )     + \sqrt{C}\|\cT Q_{K-1} -  Q_{K}\|_{1,\mu} . \label{eq:qstar-qk-bound-pt2}
\end{align}
We immediately note that both $P^{Q_{K-1},\min}_{s,a}$ and $P^{Q^{\pi^*},\min}_{s,a}$ satisfies $\dphi(P^{Q_{K-1},\min}_{s,a},P^o_{s,a})\leq 1/(\lambda(1-\gamma))$ and $\dphi(P^{Q^{\pi^*},\min}_{s,a},P^o_{s,a})\leq 1/(\lambda(1-\gamma))$, which follows by their definition and the facts $Q_{K-1}\in\cF$, $\|Q^{\pi^*}\|_\infty\leq 1/(1-\gamma)$.
Define the state-action probability distribution $\nu'$ as, for any $s',a'$, \begin{align*}
    \nu'(s',a')&=\sum_{s,a} \nu(s,a) \indic\{ Q^{\pi^*}(s,a) >  Q_{K}(s,a) \}  P^{Q_{K-1},\min}_{s,a}(s') \indic \{ a'= \argmax_b |Q^{\pi^*}(s',b)- Q_{K-1}(s',b)| \} \\&\hspace{1cm}+ \sum_{s,a} \nu(s,a) \indic\{ Q^{\pi^*}(s,a) \leq  Q_{K}(s,a) \} P^{Q^{\pi^*},\min}_{s,a}(s') \indic\{ a'= \argmax_b |Q^{\pi^*}(s',b)- Q_{K-1}(s',b)| \} .
\end{align*}
Now, we can combine~\eqref{eq:qstar-qk-bound-pt1}-\eqref{eq:qstar-qk-bound-pt2} as follows
\begin{align*}
    \|Q^{\pi^*} -  Q_{K}\|_{1,\nu}&\leq \gamma \|Q^{\pi^*} -  Q_{K-1}\|_{1,\nu'} + \sqrt{C}\|\cT Q_{K-1} -  Q_{K}\|_{1,\mu} \nonumber  \\
    &\stackrel{(i)}{\leq} \gamma \|Q^{\pi^*} -  Q_{K-1}\|_{1,\nu'} + \sqrt{C}\|\cT_{g_{K-1}} Q_{K-1} -  Q_{K}\|_{2,\mu}+ \sqrt{C}\|\cT Q_{K-1} -  \cT_{g_{K-1}} Q_{K-1}\|_{1,\mu},
\end{align*}
where $(i)$ uses the fact $\|\cdot\|_{1,\mu}\leq \|\cdot\|_{2,\mu}$.

Now, by recursion until iteration 0, we get \begin{align} \|&Q^{\pi^*} -  Q_{K}\|_{1,\nu} \leq \gamma^K \sup_{\bar{\nu}} \|Q^{\pi^*} -  Q_{0}\|_{1,\bar{\nu}} + \sqrt{C} \sum_{t=0}^{K-1} \gamma^t \|\cT Q_{K-1-t} -  \cT_{g_{K-1-t}} Q_{K-1-t}\|_{1,\mu} \nonumber
\\&\hspace{4cm}+ \sqrt{C} \sum_{t=0}^{K-1} \gamma^t  \| \cT_{g_{K-1-t}} Q_{K-1-t} -  Q_{K-t}\|_{2,\mu}
\nonumber\\
&\stackrel{(j)}{\leq} \frac{\gamma^K }{1-\gamma} + \sqrt{C} \sum_{t=0}^{K-1} \gamma^t \|\cT Q_{K-1-t} -  \cT_{g_{K-1-t}} Q_{K-1-t}\|_{1,\mu}
\nonumber\\&\hspace{4cm}+ \sqrt{C} \sum_{t=0}^{K-1} \gamma^t  \| \cT_{g_{K-1-t}} Q_{K-1-t} -  Q_{K-t}\|_{2,\mu}
\nonumber\\&\stackrel{(k)}{\leq} \frac{\gamma^K }{1-\gamma} + \frac{\sqrt{C} }{1-\gamma} \sup_{f\in\cF} \|\cT f -  \cT_{\widehat{g}_{f}} f\|_{1,\mu}
+ \frac{\sqrt{C} }{1-\gamma} \sup_{f\in\cF} \| \cT_{\widehat{g}_{f}} f -  \widehat{f}_{\widehat{g}_{f}}\|_{2,\mu}
\nonumber\\&\leq \frac{\gamma^K }{1-\gamma} + \frac{\sqrt{C} }{1-\gamma} \sup_{f\in\cF} \|\cT f -  \cT_{\widehat{g}_{f}} f\|_{1,\mu}
+ \frac{\sqrt{C} }{1-\gamma} \sup_{f\in\cF} \sup_{g\in\cG} \| \cT_{g} f -  \widehat{f}_{g}\|_{2,\mu}.  \label{eq:thm-bound-part-2} \end{align} where $(j)$ follows since $|Q^{\pi^*}(s,a)|\leq 1/(1-\gamma), Q_{0}(s,a)= 0$, and $(k)$ follows since $\widehat{g}_{f}$ is the dual variable function from the algorithm for the state-action value function $f$ and $\widehat{f}_{g}$ as the least squares solution from the algorithm for the state-action value function $f$ and dual variable function $g$ pair.

Now, using Lemma \ref{prop:erm-high-prob-bound} and Lemma \ref{prop:least-squares-generalization-bound} to bound \eqref{eq:thm-bound-part-2}, and then combining it with \eqref{eq:thm-bound-part-1}, completes the proof of this theorem.
\end{proof}




\subsection[Specialized Result for TV $\phi$-divergence]{Specialized Result for TV $\phi$-divergence \hfill \Coffeecup\Coffeecup\Coffeecup}
\label{appen:infinite-horizon-tv-guarantee}

We now state and prove the improved (in terms of assumptions) result for TV $\phi$-divergence.

\begin{assumption}[Concentrability] 
\label{tv-assum-concentra-condition}
There exists a finite constant $C_{\mathrm{tv}}>0$ such that for any $\nu\in \{ d_{\pi,P^o}  \} \subseteq \Delta(\ScA)$ for any policy $\pi$ (can be non-stationary as well), we have $\norm{\nu/\mu}_{\infty} \leq \sqrt{C_{\mathrm{tv}}}$.
\end{assumption}

\begin{assumption}[Fail-state] \label{offline-fail-state-assump}
    There is a fail state $s_f$ such that $r(s_f, a) = 0$ and $P_{s_f, a}(s_f) = 1$, for all $a \in \cA$ and $P \in \cP$ satisfying $\dtv(P_{s',a'},P^o_{s',a'})\leq \max\{1,1/(\lambda(1-\gamma))\}$  for all  $s',a'$.
\end{assumption}

\begin{theorem} \label{thm:infinite-horizon-tv-guarantee}
Let \cref{tv-assum-concentra-condition,assum-bellman-completion,assum-dual-realizability,offline-fail-state-assump} hold.  Let $\pi_{K}$ be the RPQ algorithm policy after $K$ iterations. Then, for any $\delta\in(0,1)$, with probability at least  $1 - \delta$,  we have 
\begin{align*}
	V^{\pi^*} - V^{\pi_{K}} \leq & \frac{2\gamma^K }{(1-\gamma)^2} + \frac{2\sqrt{C_{\mathrm{tv}}} }{(1-\gamma)^2} (2\gamma c_2 c_3\sqrt{\frac{2\log(|\cG|)}{N}} +  5c_1  \sqrt{\frac{2\log(8|\cF|/\delta)}{N}}  + \gamma \epsilon_{\cG})
    \\&+ \frac{2\sqrt{C_{\mathrm{tv}}} }{(1-\gamma)^2} (\sqrt{6 \epsilon_{\cF}} + \sqrt{\frac{2}{(1-\gamma)^2} + 18(1+\gamma c_1)} \sqrt{\frac{18 \log(2|\cF||\cG|/\delta)}{N}}), \end{align*}
    with $c_1=2\lambda+ (1/(1-\gamma)), c_2=2, c_3=\lambda/2$.
\end{theorem}

\begin{proof}
We can now further use the dual form \eqref{eq:robust-regularized-bellman-eq-dual} under Assumption \ref{offline-fail-state-assump}. We again start by characterizing the performance decomposition between $V^{\pi^*}$ and ${V}^{\pi_K}$. This proof largely follows  the proofs of \cref{thm:infinite-horizon-phi-divergence-guarantee} and \citet[Theorem 1]{panaganti-rfqi}.
In particular, we use the total variation RRBE its dual form \eqref{eq:robust-regularized-bellman-eq-dual} under Assumption \ref{offline-fail-state-assump} in this proof. That is, for all $\pi$ and $Q\in\cF$, from \eqref{eq:tv-inner-optimization-dual-form} we have \begin{align} \label{eq:TV-robust-regularized-Bellman-eq-dual-form}
    Q^{\pi} (s, a) &= r(s, a) -    \inf_{ \eta\in[0, \lambda]} ~ (\E_{s' \sim P^o_{s,a}} [ (\eta- V^{\pi}(s'))_+ ] - \eta) \text{ and}\\
    (\cT Q) (s, a) &= r(s, a) -    \inf_{ \eta\in[0, \lambda]} ~ (\E_{s' \sim P^o_{s,a}} [ (\eta-\max_{a'}Q(s',a'))_+ ] - \eta). \nn
\end{align}
We recall the initial state distribution $d_0$. Since $V^{\pi^*}(s) \geq V^{\pi_K}(s)$ for any $s\in\cS$, we begin with step $(b)$ in \cref{thm:infinite-horizon-phi-divergence-guarantee}:
\begin{align}
	0\leq&\E_{s_0\sim d_0}[V^{\pi^*}(s_0) - {V}^{\pi_K}(s_0)] \nn\\ 
	&\leq \E_{s_0\sim d_0}[ Q^{\pi^*}(s_0,\pi^*(s_0)) -  Q_K(s_0,\pi^*(s_0)) +  Q_K(s_0,\pi_K(s_0)) - Q^{\pi^*}(s_0,\pi_K(s_0)) \nn\\&\hspace{2cm}+ 
    \gamma [\min_{P_{s_0,\pi_K(s_0)} \ll P^o_{s_0,\pi_K(s_0)}} (\mathbb{E}_{s_1 \sim P_{s_0,\pi_K(s_0)}} [ V^{\pi^*}(s_1)] + \lambda\dphi(P_{s_0,\pi_K(s_0)}, P^o_{s_0,\pi_K(s_0)})) \nn\\&\hspace{3cm}- \min_{P_{s_0,\pi_K(s_0)} \ll P^o_{s_0,\pi_K(s_0)}} (\mathbb{E}_{s_1 \sim P_{s_0,\pi_K(s_0)}} [ V^{\pi_K}(s_1)] + \lambda\dphi(P_{s_0,\pi_K(s_0)}, P^o_{s_0,\pi_K(s_0)}))] ] \nn\\
	&\stackrel{(a)}{\leq} \E_{s_0\sim d_0}[|Q^{\pi^*}(s_0,\pi^*(s_0)) -  Q_K(s_0,\pi^*(s_0))|] + \E_{s_0\sim d_0}[|Q^{\pi^*}(s_0,\pi_K(s_0)) -  Q_K(s_0,\pi_K(s_0))|  ]
	\nn\\&\hspace{4cm}+ \gamma \E_{s_0\sim d_0} \sup_{\eta} (\E_{s_1\sim P^o_{s_0,\pi_{K}(s_0)}}(  (\eta- V^{\pi_K}(s_1))_+ - (\eta-V^{\pi^*}(s_1))_+)) \nn\\
	&\stackrel{(b)}{\leq} \E_{s_0\sim d_0}[|Q^{\pi^*}(s_0,\pi^*(s_0)) -  Q_K(s_0,\pi^*(s_0))|] + \E_{s_0\sim d_0}[|Q^{\pi^*}(s_0,\pi_K(s_0)) -  Q_K(s_0,\pi_K(s_0))|  ]
	\nn\\&\hspace{4cm}+ \gamma \E_{s_0\sim d_0}\E_{s_1\sim P^{o}_{s_0,\pi_{K}(s_0)}}(  |V^{\pi^*}(s_1)- V^{\pi_K}(s_1)| ) \nn\\
    &\stackrel{(c)}{\leq} \sum_{h=0}^\infty \gamma^h \times\bigg(\E_{s\sim d_{h,\pi_{K}}} [ |Q^{\pi^*}(s,\pi^*(s)) -  Q_{K}(s,\pi^*(s))| + |Q^{\pi^*}(s,\pi_K(s)) -  Q_{K}(s,\pi_K(s))| ] \bigg),  \label{eq:TV-perf-diff-1} 
\end{align}
where $(a)$ follows from \eqref{eq:TV-robust-regularized-Bellman-eq-dual-form} and the fact $|\sup_{x} f(x)-\sup_{x} g(x)| \leq \sup_{x} |f(x) - g(x)|$, $(b)$ follows from the facts $(x)_+ - (y)_+ \leq (x-y)_+$ and $(x)_+ \leq |x|$ for any $x,y\in\R$.
We make an important note here in step $(b)$ regarding the dependence on the nominal model $P^o$ distribution unlike in step $(c)$ in the proof of \cref{thm:infinite-horizon-phi-divergence-guarantee}. This important step helps us improve the concentrability assumption in further analysis.
Finally, $(c)$ follows with telescoping over $|V^{\pi^*}-V^{\pi_K}|$ by defining a new state distribution $d_{h,\pi_K}\in \Delta(\cS)$, for all natural numbers $h\geq 0$, as
\begin{equation*}
  d_{h,\pi_{K}} =
    \begin{cases}
      d_0 & \text{if $h=0$},\\
      P^{o}_{s',\pi_{K}(s')} & \text{otherwise, with } s'\sim d_{h-1,\pi_{K}}.
    \end{cases}       
\end{equation*}

For~\eqref{eq:TV-perf-diff-1}, with the $\nu$-norm notation i.e. $\|f\|_{p,\nu}^2 = (\E_{s,a\sim\nu} |f(s,a)|^p)^{1/p}$ for any $\nu\in\Delta(\ScA)$, we have 
\begin{align}  
\label{eq:tv-bound-part-1}
\E_{s_0\sim d_0}[{V}^{\pi^*}] - \E_{s_0\sim d_0}[V^{\pi_K}] &\leq \sum_{h=0}^\infty \gamma^h \bigg(    \|Q^{\pi^*} -  Q_{K}\|_{1,d_{h,\pi_{K}}\circ \pi^* } + \|Q^{\pi^*} -  Q_{K}\|_{1,d_{h,\pi_{K}}\circ \pi_K } \bigg), \nn\\
&\leq  \sum_{h=0}^\infty \gamma^h     (2\sup_\nu \|Q^{\pi^*} -  Q_{K}\|_{1,\nu } ), 
\end{align} 
where the second inequality follows since both $d_{h,\pi_{K}}\circ \pi^*$ and $d_{h,\pi_{K}}\circ \pi_K$ satisfy Assumption \ref{tv-assum-concentra-condition}. We now analyze the summand in \eqref{eq:thm-bound-part-1}:
\begin{align}
    &\|Q^{\pi^*} -  Q_{K}\|_{1,\nu} \leq \|Q^{\pi^*} -  \cT Q_{K-1}\|_{1,\nu} + \|\cT Q_{K-1} -  Q_{K}\|_{1,\nu} \nonumber \\
    &\stackrel{(d)}{\leq} \|Q^{\pi^*} -  \cT Q_{K-1}\|_{1,\nu} + \sqrt{C_{\mathrm{tv}}}\|\cT Q_{K-1} -  Q_{K}\|_{1,\mu} \nonumber \\
    &= (\E_{s,a\sim\nu} |Q^{\pi^*}(s,a) -  \cT Q_{K-1}(s,a)|) + \sqrt{C_{\mathrm{tv}}}\|\cT Q_{K-1} -  Q_{K}\|_{1,\mu} \nonumber \\
    &\stackrel{(e)}{\leq} (\E_{s,a\sim\nu}   \gamma \sup_{\eta}|
    \E_{s'\sim P^o_{s,a}} (  (\eta-\max_{a'} Q_{K-1}(s',a'))_+ - (\eta-\max_{a'} Q^{\pi^*}(s',a'))_+) 
    |)
    \nonumber \\&\hspace{6cm}+ \sqrt{C_{\mathrm{tv}}}\|\cT Q_{K-1} -  Q_{K}\|_{1,\mu} \nonumber \\
    &\stackrel{(f)}{\leq} (\E_{s,a\sim\nu}  | 
    \E_{s'\sim P^o_{s,a}}   (\max_{a'} Q^{\pi^*}(s',a')-\max_{a'} Q_{K-1}(s',a'))_+|) + \sqrt{C_{\mathrm{tv}}}\|\cT Q_{K-1} -  Q_{K}\|_{1,\mu} \nonumber \\
    &\stackrel{(g)}{\leq} \gamma (\E_{s,a\sim\nu} \E_{s'\sim {P}^o_{s,a}} \max_{a'} |Q^{\pi^*}(s',a') -  Q_{K-1}(s',a')| )     + \sqrt{C_{\mathrm{tv}}}\|\cT Q_{K-1} -  Q_{K}\|_{1,\mu} \nonumber\\
    &\stackrel{(h)}{\leq} \gamma \|Q^{\pi^*} -  Q_{K-1}\|_{1,\nu'} + \sqrt{C_{\mathrm{tv}}}\|\cT Q_{K-1} -  Q_{K}\|_{1,\mu} \nonumber  \\
    &\stackrel{(i)}{\leq} \gamma \|Q^{\pi^*} -  Q_{K-1}\|_{1,\nu'} + \sqrt{C_{\mathrm{tv}}}\|\cT_{g_{K-1}} Q_{K-1} -  Q_{K}\|_{2,\mu}+ \sqrt{C_{\mathrm{tv}}}\|\cT Q_{K-1} -  \cT_{g_{K-1}} Q_{K-1}\|_{1,\mu}, \nn
\end{align}
where $(d)$ follows by Assumption \ref{tv-assum-concentra-condition}, $(e)$ from \cref{eq:TV-robust-regularized-Bellman-eq-dual-form} and the fact $|\sup_{x} p(x)-\sup_{x} q(x)| \leq \sup_{x} |p(x) - q(x)|$,
$(f)$ from the fact $|(x)_+ - (y)_+| \leq |(x-y)_+|$,
$(g)$ follows by Jensen's inequality and by the facts $|\sup_{x} p(x)-\sup_{x} q(x)| \leq \sup_{x} |p(x) - q(x)|$ and $(x)_+ \leq |x|$, $(h)$ follows by defining the distribution $\nu'$ as $\nu'(s',a')=\sum_{s,a} \nu(s,a) {P}^o_{s,a}(s') \indic \{ a'= \argmax_b |Q^{\pi^*}(s',b)- Q_{K-1}(s',b)| \}$, and $(i)$ using the fact that $\|\cdot\|_{1,\mu}\leq \|\cdot\|_{2,\mu}$.
The rest of the proof follows similarly as in the proof of \cref{thm:infinite-horizon-phi-divergence-guarantee}.
\end{proof}

\newpage

\section[Hybrid Robust $\phi$-regularized RL Results]{Hybrid Robust $\phi$-regularized RL Results \hfill \Coffeecup\Coffeecup\Coffeecup\Coffeecup}
\label{appendix:hybrid-robust-results}

In this section, we set $\Vmax=H$ whenever we use results from \cref{prop:divergence-loss-bounds}. We remark that we have attempted to optimize the absolute constants inside $\log$ factors of the performance guarantees.  In the following, we use constants $c_1,c_2,c_3$ from \cref{prop:divergence-loss-bounds}.

Now we provide an extension of \cref{prop:erm-high-prob-bound} using \cref{prop:erm-dependent-data} when the data comes from adaptive sampling.
\begin{proposition}[Online Dual Optimization Error Bound] \label{prop:erm-high-prob-bound-offline+online}
Fix $\delta \in (0,1)$. For $k \in \{0,1,\cdots,K-1\}$, $h \in \{0,1,\cdots,H-1\}$, let $g^{k}_{h}$ be the dual optimization function from \cref{alg:HyTQ-Algorithm} (Step 4) for the state-action value function  $Q^{k}_{h+1}$ using samples in the dataset $\{\cD^\mu_h, \cD^{0}_h, \cdots, \cD^{k-1}_h\} $.
Let $\cT_{g}$ be as defined in \eqref{eq:FH-Tg} and let $N=m_{\mathrm{off}} + K \cdot m_{\mathrm{on}}$. Then, with probability at least $1-\delta$, we have 
\begin{align*}
    &\|  \cT Q^{k}_{h+1} - \cT_{g^{k}_{h}} Q^{k}_{h+1} \|_{1,\mu_h}
    \leq \frac{1}{m_{\mathrm{off}}} \left( 3 \epsilon_{\cG} N + 48 c_1 \log(2 HK|\cG||\cF|/\delta) \right) = \Delta_{\mathrm{dual,off}} \quad\text{and}\\
    &\sum_{\tau=0}^{k-1} \|  \cT Q^{k}_{h+1} - \cT_{g^{k}_{h}} Q^{k}_{h+1} \|_{1,d_h^{\pi_\tau}} 
    \leq \frac{1}{m_{\mathrm{on}}} \left( 3 \epsilon_{\cG} N + 48 c_1 \log(2 HK|\cG||\cF|/\delta) \right) = \Delta_{\mathrm{dual,on}}.
\end{align*}
\end{proposition}
\begin{proof}
    Fix $k \in \{0,1,\cdots,K-1\}$, $h \in \{0,1,\cdots,H-1\}$, $Q^{k}_{h+1} \in \cF_{h+1}$. The algorithm solves for $g^{k}_{h}$ in the empirical risk minimization step as:
 \begin{align*}
g^{k}_{h} =
 \argmin_{g \in \cG_{h}} \widehat{L}_{\mathrm{dual}}(g; Q^{k}_{h+1}, \cD), 
\end{align*}
where dataset $\cD=\{(s^i_h, a^i_h, s^i_{h+1})\}_{i \leq N}$ with $N = m_{\mathrm{off}} + k \cdot m_{\mathrm{on}}$.
The first $m_{\mathrm{off}}$ samples in $\cD$ are $\{(s^i_h, a^i_h, s^i_{h+1})\}_{i \leq m_{\mathrm{off}}} = \cD^{\mu}_h$ (recall that these are generated by the offline state-action distribution $\mu_h$),  the next $m_{\mathrm{on}}$ samples are  $\{(s^i_h, a^i_h, s^i_{h+1})\}_{i = m_{\mathrm{off}} + 1}^{m_{\mathrm{off}} + m_{\mathrm{on}}} = \cD^0_h$ (recall that these are generated by the state-action distribution $d_h^{\pi_0}$), and so on where the samples $\{(s^i_h, a^i_h, s^i_{h+1})\}_{i = m_{\mathrm{off}} + \tau \cdot m_{\mathrm{on}} + 1}^{m_{\mathrm{off}} + (\tau+1) m_{\mathrm{on}}} = \cD^{\tau}_h$ (recall that these are generated by the state-action distribution $d_h^{\pi_\tau}$)  for all $\tau\leq k-1$.
We first have the following from step (b) in the proof of \cref{prop:erm-high-prob-bound}: 
\begin{align*}
    m_{\mathrm{off}} &\|  \cT Q^{k}_{h+1} - \cT_{g^{k}_{h}} Q^{k}_{h+1} \|_{1,\mu} + m_{\mathrm{on}} \sum_{\tau=0}^{k-1} \|  \cT Q^{k}_{h+1} - \cT_{g^{k}_{h}} Q^{k}_{h+1} \|_{1,d_h^{\pi_\tau}} 
    \\&= m_{\mathrm{off}} [ \E_{s,a\sim \mu_h,s'\sim P^o_{s,a}} (\lambda \phi^*({(g^{k}_{h}(s,a)-\max_{a'} Q^{k}_{h+1}(s',a'))}/{\lambda})  - g^{k}_{h}(s,a)) \\&\hspace{1cm}- \inf_{g \in L^1(\mu_h)} \E_{s,a\sim\mu_h,s'\sim P^o_{s,a}} (\lambda \phi^*({(g(s,a)-\max_{a'} Q^{k}_{h+1}(s',a'))}/{\lambda})  - g(s,a)) ]
    \\&\hspace{1cm} + m_{\mathrm{on}} \sum_{\tau=0}^{k-1} [ \E_{s,a\sim d_h^{\pi_\tau},s'\sim P^o_{s,a}} (\lambda \phi^*({(g^{k}_{h}(s,a)-\max_{a'} Q^{k}_{h+1}(s',a'))}/{\lambda})  - g^{k}_{h}(s,a)) \\&\hspace{1cm}- \inf_{g \in L^1(d_h^{\pi_\tau})} \E_{s,a\sim d_h^{\pi_\tau},s'\sim P^o_{s,a}} (\lambda \phi^*({(g(s,a)-\max_{a'} Q^{k}_{h+1}(s',a'))}/{\lambda})  - g(s,a)) ]
    \\&\stackrel{(a)}{=} m_{\mathrm{off}} [ \E_{s,a\sim \mu_h,s'\sim P^o_{s,a}} (\lambda \phi^*({(g^{k}_{h}(s,a)-\max_{a'} Q^{k}_{h+1}(s',a'))}/{\lambda})  - g^{k}_{h}(s,a)) \\&\hspace{1cm}- \E_{s,a\sim\mu_h,s'\sim P^o_{s,a}} (\lambda \phi^*({(g^*_{-1}(s,a)-\max_{a'} Q^{k}_{h+1}(s',a'))}/{\lambda})  - g^*_{-1}(s,a)) ]
    \\&\hspace{1cm} + m_{\mathrm{on}} \sum_{\tau=0}^{k-1} [ \E_{s,a\sim d_h^{\pi_\tau},s'\sim P^o_{s,a}} (\lambda \phi^*({(g^{k}_{h}(s,a)-\max_{a'} Q^{k}_{h+1}(s',a'))}/{\lambda})  - g^{k}_{h}(s,a)) \\&\hspace{1cm}-  \E_{s,a\sim d_h^{\pi_\tau},s'\sim P^o_{s,a}} (\lambda \phi^*({(g^*_{\tau}(s,a)-\max_{a'} Q^{k}_{h+1}(s',a'))}/{\lambda})  - g^*_{\tau}(s,a)) ]
    \\&= \sum_{i=1}^{m_{\mathrm{off}}} \E_{s^i_h,a^i_h\sim \mu_h,s^i_{h+1}\sim P^o_{s^i_h,a^i_h}} [ (\lambda \phi^*({(g^{k}_{h}(s^i_h,a^i_h)-\max_{a'} Q^{k}_{h+1}(s^i_{h+1},a'))}/{\lambda})  - g^{k}_{h}(s^i_h,a^i_h)) \\
    &\hspace{1cm} - (\lambda \phi^*({(g^*_{-1}(s^i_h,a^i_h)-\max_{a'} Q^{k}_{h+1}(s^i_{h+1},a'))}/{\lambda})  - g^*_{-1}(s^i_h,a^i_h)) ]
    \\&\hspace{1cm} + \sum_{i=m_{\mathrm{off}}+1}^{m_{\mathrm{off}} + m_{\mathrm{on}}} \E_{s^i_h,a^i_h\sim d_h^{\pi_0},s^i_{h+1}\sim P^o_{s^i_h,a^i_h}} [ (\lambda \phi^*({(g^{k}_{h}(s^i_h,a^i_h)-\max_{a'} Q^{k}_{h+1}(s^i_{h+1},a'))}/{\lambda})  - g^{k}_{h}(s^i_h,a^i_h)) \\
    &\hspace{2cm} - (\lambda \phi^*({(g^*_{0}(s^i_h,a^i_h)-\max_{a'} Q^{k}_{h+1}(s^i_{h+1},a'))}/{\lambda})  - g^*_{0}(s^i_h,a^i_h)) ] \\&\hspace{2cm}+ \cdots
    \\&\stackrel{(b)}{\leq} 3 \epsilon_{\cG} N + 48 c_1 \log(2 |\cG||\cF|/\delta),
\end{align*}
where  $(a)$ follows by defining the corresponding true solutions $g^*_{\tau}$ for all $\tau\in\{-1,0,1,\cdots,k-1\}$.
For $(b)$ with the empirical risk minimization solution $g^{k}_{h}$, we use \cref{prop:erm-dependent-data} by setting $c=c_1$ (with $c_1$, constant dependent on $H$ and $\lambda$, from \cref{prop:divergence-loss-bounds}) and
since $g^{k}_{h}\in\cG_{h}, Q^{k}_{h+1} \in\cF_{h+1}$ with sizes $|\cG_{h}|\leq |\cG|$ and $|\cF_{h+1}|\leq |\cF|$  under the union bound.
Taking a union bound over $k \in \{0,1,\cdots,K-1\}$, $h \in \{0,1,\cdots,H-1\}$, and bounding each term separately, completes the proof.
\end{proof}

Now we provide an extension of \cref{prop:least-squares-generalization-bound} using \cref{lem:least-squares-bound-song} when the data comes from adaptive sampling.
\begin{proposition}[Online Least-squares Generalization Bound] \label{prop:least-squares-generalization-bound-online+offline}
Fix $\delta \in (0,1)$. For $k \in \{0,1,\cdots,K-1\}$, $h \in \{0,1,\cdots,H-1\}$, let $Q^{k}_{h}$ be the least-squares solution from \cref{alg:HyTQ-Algorithm} (Step 5) for the state-action value function $Q^{k}_{h+1}$ and  dual variable function  $g^{k}_{h}$ using samples in the dataset $\{\cD^\mu_h, \cD^{0}_h, \cdots, \cD^{k-1}_h\} $.
Let $\cT_{g}$ be as defined in \eqref{eq:FH-Tg} and let $N=m_{\mathrm{off}} + K \cdot m_{\mathrm{on}}$. Then, with probability at least $1-\delta$, we have 
\begin{align*}
    &\|  \cT_{g^{k}_{h}} Q^{k}_{h+1} - Q^{k}_{h} \|_{2,\mu_h}
    \leq \frac{1}{\sqrt{m_{\mathrm{off}}}} \left( \sqrt{3 \epsilon_{\cF,\mathrm{r}} N} + 8 (1+c_1+H)  \sqrt{\log(2 HK |\cG||\cF|/\delta)} \right) = \Delta_{\mathrm{rQ,off}} \quad\text{and}\\
    &\sqrt{\sum_{\tau=0}^{k-1} \|  \cT_{g^{k}_{h}} Q^{k}_{h+1} - Q^{k}_{h} \|_{2,d_h^{\pi_\tau}}^2}
    \leq \frac{1}{\sqrt{m_{\mathrm{on}}}} \left( \sqrt{3 \epsilon_{\cF,\mathrm{r}} N} + 8 (1+c_1+H) \sqrt{\log(2 HK |\cG||\cF|/\delta)} \right) = \Delta_{\mathrm{rQ,on}}.
\end{align*}
\end{proposition}
\begin{proof}
We adapt the proof of \citet[Lemma 7]{song2023hybrid} here.
    Fix $k \in \{0,1,\cdots,K-1\}$, $h \in \{0,1,\cdots,H-1\}$, $g^{k}_{h} \in \cG_{h}$, and $Q^{k}_{h+1} \in \cF_{h+1}$. The algorithm solves for $Q^{k}_{h}$ in the least-squares regression step as:
 \begin{align*}
Q^{k}_{h} = \argmin_{Q\in\cF_{h}} \widehat{L}_{\mathrm{robQ}}(Q;Q^{k}_{h+1}, g^{k}_{h}, \cD), 
\end{align*}
where dataset $\cD=\{(x_i, y_i)\}_{i \leq N}$ with $N = m_{\mathrm{off}} + k \cdot m_{\mathrm{on}}$ and 
\begin{align*}
x_i = (s^i_h, a^i_h) \qquad \text{and} \qquad y_i = r_h(s^i_h, a^i_h) - \lambda \phi^*({(g^{k}_{h}(s^i_h, a^i_h)-\max_{a'} Q^{k}_{h+1}(s^i_{h+1},a'))}/{\lambda})+ g^{k}_{h}(s^i_h, a^i_h). 
\end{align*}
The first $m_{\mathrm{off}}$ samples in $\cD$ are $\{(x_i, y_i)\}_{i \leq m_{\mathrm{off}}} = \cD^{\mu}_h$ (recall that these are generated by the offline state-action distribution $\mu_h$),  the next $m_{\mathrm{on}}$ samples are  $\{(x_i, y_i)\}_{i = m_{\mathrm{off}} + 1}^{m_{\mathrm{off}} + m_{\mathrm{on}}} = \cD^0_h$ (recall that these are generated by the state-action distribution $d_h^{\pi_0}$), and so on where the samples $\{(x_i, y_i)\}_{i = m_{\mathrm{off}} + \tau \cdot m_{\mathrm{on}} + 1}^{m_{\mathrm{off}} + (\tau+1) m_{\mathrm{on}}} = \cD^{\tau}_h$ (recall that these are generated by the state-action distribution $d_h^{\pi_\tau}$) for all $\tau\leq k-1$.

For using \cref{lem:least-squares-bound-song}, we first note  for any sample $(x,y)$ in $\cD$ with $x = (s_h, a_h)$ and $ y = (r_h(s_h, a_h) - \lambda \phi^*({(g^{k}_{h}(s_h, a_h)-\max_{a'\in\cA_{h+1}} Q^{k}_{h+1}(s_{h+1},a'))}/{\lambda})+ g^{k}_{h}(s_h, a_h))$, there exists some $f_{h+1}\in\cF_{h+1}$ by \cref{assum-realizability-with-bellman-completion} such that the following holds: 
\begin{align*}
\E [y \mid x] &= \E_{s_{h+1} \sim P^o_{h,s_h, a_h}} (r_h(s_h, a_h) - \lambda \phi^*({(g^{k}_{h}(s_h, a_h)-\max_{a'\in\cA_{h+1}} Q^{k}_{h+1}(s_{h+1},a'))}/{\lambda})+ g^{k}_{h}(s_h, a_h))\\
&= \cT_{g^{k}_{h}} Q^{k}_{h+1}(s_h, a_h) \leq f_{h+1}(s_h, a_h).
\end{align*} 
We also note for any sample in $\cD$, $|y| \leq 1+c_1$ (with $c_1$, constant dependent on $H$ and $\lambda$, from \cref{prop:divergence-loss-bounds}) and $f_{h+1}(s, a) \leq H$ for all $s, a$.
With these notes, applying \cref{lem:least-squares-bound-song}, we get that the least square regression solution $Q^{k}_{h}$ satisfies 
\begin{align*}
    \sum_{i=1}^N \E [(\cT_{g^{k}_{h}} Q^{k}_{h+1}(x_i) - Q^{k}_{h}(x_i) )^2\mid \cD]
    &\leq 3 \epsilon_{\cF,\mathrm{r}} N + 64 (1+c_1+H)^2 \log(2 |\cG||\cF|/\delta)
\end{align*} 
with probability at least $1-\delta$, since $g^{k}_{h}\in\cG_{h}$ and $Q^{k}_{h+1} \in\cF_{h+1}$ with sizes $|\cG_{h}|\leq |\cG|$ and $|\cF_{h+1}|\leq |\cF|$ under the union bound.
Recall the  samples in $\cD^\mu_h$ are independently and identically drawn from the offline distribution $\mu_h$, and the samples in $\cD^\tau_{h}$ are independently and identically drawn from the state-action distribution $d_h^{\pi_\tau}$. Thus we can further write as
\begin{align*}
    m_{\mathrm{off}} \|  \cT_{g^{k}_{h}} Q^{k}_{h+1} - Q^{k}_{h} \|_{2,\mu}^2 + m_{\mathrm{on}} \sum_{\tau=0}^{k-1} \|  \cT_{g^{k}_{h}} Q^{k}_{h+1} - Q^{k}_{h} \|_{2,d_h^{\pi_\tau}}^2
    &\leq 3 \epsilon_{\cF,\mathrm{r}} N + 64 (1+c_1+H)^2 \log(2 |\cG||\cF|/\delta).
\end{align*}
Taking a union bound over $k \in \{0,1,\cdots,K-1\}$, $h \in \{0,1,\cdots,H-1\}$, bounding each term separately, and using the fact $\sqrt{x+y}\leq \sqrt{x}+\sqrt{y}$, completes the proof.
\end{proof}



We are now ready to prove the main theorem.

\subsection[Proof of \cref{thm:hybrid-tv-guarantee}]{Proof of \cref{thm:hybrid-tv-guarantee} \hfill \Coffeecup\Coffeecup\Coffeecup\Coffeecup}
\label{appendix:thm:hybrid-tv-guarantee}

\begin{theorem}[Restatement of \cref{thm:hybrid-tv-guarantee}] \label{thm:restated:hybrid-tv-guarantee}
Let \cref{assum-bellman-transfer-coefficient,assum-realizability-with-bellman-completion,assum-dual-realizability-hybrid,assum-bilinear-model,fail-state-assump} hold and fix any $\delta\in(0,1)$. Then, HyTQ algorithm policies $\{\pi_k\}_{k\in[K]}$ satisfy
\begin{align*}
     \sum_{k = 0}^{K-1} (V^{\pi^*} - V^{\pi_{k}}) \leq &\cO((\sqrt{\epsilon_{\cF,\mathrm{r}}} + \epsilon_{\cG})K^{5/2}H) \\&+ \cOtilde ( \max\{C(\pi^*), 1\} \sqrt{d K H^2} (\lambda+H)  \log(HK|\cF||\cG|/\delta) \sqrt{\log(1+(K/d))} )
\end{align*} 
with probability at least  $1 - \delta$.
\end{theorem}

\begin{proof}
We let $V^k_h(s) = Q^k_h(s,\pi_k(s))$ for every $s,h$. Since $\pi_k$ is the greedy  policy w.r.t $Q^k$, we also have $V^k_h(s) = Q^k_h(s,\pi_k(s)) = \max_a Q^k_h(s,a)$. We recall that $V^*=V^{\pi^*}$ and $Q^*=Q^{\pi^*}$. We also note that the same holds true for any stationary Markov policy $\pi$  from \citep{zhang2023regularized} that $Q^{\pi}$ satisfies $Q^{\pi}_{h}(s, a) = r_{h}(s, a) + \gamma \min_{P_{h,s,a} \ll P^o_{h,s,a}} (\mathbb{E}_{s' \sim P_{h,s,a}} [ V^{\pi}_{h}(s')] + \lambda\dphi(P_{h,s,a}, P^o_{h,s,a})).$ We can now further use the dual form \eqref{eq:robust-regularized-bellman-eq-dual} under Assumption \ref{fail-state-assump}, that is, for all $\pi$ and $f_{h+1}\in\cF_{h+1}$, 
\begin{align} \label{eq:FH-tv-dual-bellman-eq}
    Q^{\pi}_{h} (s, a) &= r_{h}(s, a) -    \inf_{ \eta\in[0, \lambda]} ~ (\E_{s' \sim P^o_{h,s,a}} [ (\eta- V^{\pi}_{h+1}(s'))_+ ] - \eta), \text{ and}\\
    (\cT f_{h+1}) (s, a) &= r_{h}(s, a) -    \inf_{ \eta\in[0, \lambda]} ~ (\E_{s' \sim P^o_{h,s,a}} [ (\eta-\max_{a'}f_{h+1}(s',a'))_+ ] - \eta) \nn \\
    (\cT_{g_h} f_{h+1}) (s, a) &= r_h(s, a) -  \E_{s' \sim P^{o}_{h,s,a}}[(g_h(s,a)-\max_{a'}f_{h+1}(s',a'))_+]  + g_h(s,a)  . \nn
\end{align} We first characterize the performance decomposition between $V^{\pi^*}_{0}$ and ${V}^{\pi_k}_{0}$. We recall the initial state distribution $d_0$. Since $V^{\pi^*}(s) \geq V^{\pi_k}(s)$ for any $s\in\cS$, we observe that
\begin{align}
	0\leq&\sum_{k=0}^{K-1}\E_{s_0\sim d_0}[V^{\pi^*}_{0}(s_0) - V^{\pi_k}_{0}(s_0)] = \sum_{k=0}^{K-1}\E_{s_0\sim d_0}[ (V^{\pi^*}_{0}(s_0) -  V^k_{0}(s_0) )  - (V^{\pi_k}_{0}(s_0) -  V^k_{0}(s_0) )]\nn\\ 
	&= \sum_{k=0}^{K-1} \E_{s_0\sim d_0}[(Q^{\pi^*}_{0}(s_0,\pi^*(s_0)) -  Q^{k}_{0}(s_0,\pi_k(s_0)) ) - (Q^{\pi_k}_{0}(s_0,\pi_k(s_0)) -  Q^{k}_{0}(s_0,\pi_k(s_0)))] \nn\\ 
	&\leq \underbrace{\sum_{k=0}^{K-1} \E_{s_0\sim d_0}[(Q^{\pi^*}_{0}(s_0,\pi^*(s_0)) -  Q^{k}_{0}(s_0,\pi_k(s_0)))_+] }_{(I)} + \underbrace{\sum_{k=0}^{K-1} \E_{s_0\sim d_0}  [(Q^{k}_{0}(s_0,\pi_k(s_0)) - Q^{\pi_k}_{0}(s_0,\pi_k(s_0)))_+ ]}_{(II)}. \label{eq:tv-oo-main}
\end{align}

We rewrite the state-action distribution $d^{h,\pi}_{P^o}$, dropping $P^o$, as $d^{\pi}_{h}$ for simplicity. Letting  $d^{\pi}_{h}$ also denote a state  distribution ($\Delta(\cS)$), we can write it as, for all $h$,
\begin{equation} \label{eq:hytq-state-dist-Po}
  d^{\pi}_{h} =
    \begin{cases}
      d_0 & \text{if $h=0$},\\
      P^{o}_{h,s',a'} & \text{otherwise, with } s'\sim d^{\pi}_{h-1}, a'\sim \pi_{h}(s').
    \end{cases}       
\end{equation}

Analyzing one term in $(I)$ of \eqref{eq:tv-oo-main} starting with the facts that $\pi_k$ is the greedy policy with respect to $Q^{k}$ and function $(x)_+$ is non-decreasing in $x\in\R$:
\begin{align}
    &\E_{s_0\sim d_{0}}[(Q^{\pi^*}_{0}(s_0,\pi^*(s_0)) -  Q^{k}_{0}(s_0,\pi_k(s_0)))_+] \leq \E_{s_0,a_0\sim d^{\pi^*}_{0}}[(Q^{\pi^*}_{0}(s_0,a_0) -  Q^{k}_{0}(s_0,a_0))_+]  \nn \\
    &\stackrel{(a)}{\leq} 
    \E_{s_0,a_0\sim d^{\pi^*}_{0}}[(Q^{\pi^*}_{0}(s_0,a_0) -  \cT Q^{k}_{1}(s_0,a_0))_+] + \E_{s_0,a_0\sim d^{\pi^*}_{0}}[(\cT Q^{k}_{1}(s_0,a_0) -  Q^{k}_{0}(s_0,a_0))_+] \nn \\
    &\stackrel{(b)}{\leq} \E_{s_0,a_0\sim d^{\pi^*}_{0}}    (  \sup_{\eta} (
    \E_{s_1\sim P^o_{0,s_0,a_0}}   [(\eta-\max_{a'} Q^{k}_{1}(s_1,a'))_+ - (\eta-\max_{a'} Q^{\pi^*}_{1}(s_1,a'))_+] ) )_+ 
    \nn \\&\hspace{3cm}+  \E_{s_0,a_0\sim d^{\pi^*}_{0}}[(\cT Q^{k}_{1}(s_0,a_0) -  Q^{k}_{0}(s_0,a_0))_+] \nn \\
    &\stackrel{(c)}{\leq}  \E_{s_0,a_0\sim d^{\pi^*}_{0}}   
    (\E_{s_1\sim P^o_{0,s_0,a_0}}   (\max_{a'} Q^{\pi^*}_{1}(s_1,a')-\max_{a'} Q^{k}_{1}(s_1,a'))_+)_+ +  \E_{s_0,a_0\sim d^{\pi^*}_{0}}[(\cT Q^{k}_{1}(s_0,a_0) -  Q^{k}_{0}(s_0,a_0))_+] \nonumber \\
    &\stackrel{(d)}{\leq}  \E_{s_0,a_0\sim d^{\pi^*}_{0}} \E_{s_1\sim P^o_{0,s_0,a_0}} ( Q^{\pi^*}_{1}(s_1,\pi^*(s_1)) -  Q^{k}_{1}(s_1,\pi_k(s_1)))_+      + \E_{s_0,a_0\sim d^{\pi^*}_{0}}[(\cT Q^{k}_{1}(s_0,a_0) -  Q^{k}_{0}(s_0,a_0))_+] \nn \\
    &=  \E_{s_0\sim d^{\pi^*}_{1}} [ ( Q^{\pi^*}_{1}(s_1,\pi^*(s_1)) -  Q^{k}_{1}(s_1,\pi_k(s_1)))_+]      + \E_{s_0,a_0\sim d^{\pi^*}_{0}}[(\cT Q^{k}_{1}(s_0,a_0) -  Q^{k}_{0}(s_0,a_0))_+], \label{eq:tv-part1-recurse-0}
\end{align} 
where $(a)$ follows by triangle inequality for $(\cdot)_+$ operation, $(b)$ from Bellman equation, operator $\cT$, and the fact $\inf_{x} p(x)-\inf_{x} q(x) \leq \sup_{x} (p(x) - q(x))$,
$(c)$ from the fact $(x)_+ - (y)_+ \leq (x-y)_+$ for any $x,y\in\R$,
$(d)$ follows by Jensen's inequality and by definitions of policies $\pi^*$ and $\pi_k$. Now, recursively applying this method for first term over horizon in \eqref{eq:tv-part1-recurse-0} we get
\begin{align}
    &\E_{s_0\sim d_{0}}[(Q^{\pi^*}_{0}(s_0,\pi^*(s_0)) -  Q^{k}_{0}(s_0,\pi_k(s_0)))_+] \nn\\&\leq  \E_{s_H\sim d_{H}}[(Q^{\pi^*}_{H}(s_H,\pi^*(s_H)) -  Q^{k}_{H}(s_H,\pi_k(s_H)))_+]     + \sum_{h=0}^{H-1} \E_{s,a\sim d^{\pi^*}_{h}}[(\cT Q^{k}_{h+1}(s,a) -  Q^{k}_{h}(s,a))_+] \nn \\
    &\leq  \sum_{h=0}^{H-1} \E_{s,a\sim d^{\pi^*}_{h}}[(\cT Q^{k}_{h+1}(s,a) -  Q^{k}_{h}(s,a))_+], \label{eq:tv-part1}
\end{align} where the last inequality holds since $ V^{\pi}_{H}(s_H) = 0 $ for all $\pi$ and $  Q^{k}_{H}(s_H,\pi_k(s_H)) = 0$.

Recall \[ C(\pi^*) = \max_{f\in\cF} \frac{\sum_{h=0}^{H-1} \E_{s,a\sim d^{\pi^*}_{h}}[(\cT f_{h+1}(s,a) -  f_{h}(s,a))_+]}{ \sum_{h=0}^{H-1} \E_{s,a\sim \mu_{h}}[|\cT f_{h+1}(s,a) -  f_{h}(s,a)|] }. \]
Now, using \eqref{eq:tv-part1} in $(I)$ of \eqref{eq:tv-oo-main}, the following holds with probability at least $1-\delta/2$:
\begin{align}
    \sum_{k=0}^{K-1} \E_{s_0\sim d_0}[(Q^{\pi^*}(s_0,\pi^*(s_0)) &-  Q^{k}(s_0,\pi_k(s_0)))_+] \leq  \sum_{k=0}^{K-1} \sum_{h=0}^{H-1} \E_{s,a\sim d^{\pi^*}_{h}}[(\cT Q^{k}_{h+1}(s,a) -  Q^{k}_{h}(s,a))_+] \nn \\
    &\stackrel{(e)}{\leq}  \sum_{k=0}^{K-1} C(\pi^*) \sum_{h=0}^{H-1} \|\cT Q^{k}_{h+1} -  Q^{k}_{h}\|_{1,\mu_{h}}   \nn \\
    &\stackrel{(f)}{\leq}  \sum_{k=0}^{K-1} C(\pi^*) \sum_{h=0}^{H-1} (\|\cT Q^{k}_{h+1} -  \cT_{g^{k}_{h}} Q^{k}_{h+1}\|_{1,\mu_{h}} + \| \cT_{g^{k}_{h}} Q^{k}_{h+1} -  Q^{k}_{h}\|_{2,\mu_{h}} )   \nn \\
    &\stackrel{(g)}{\leq} K H C(\pi^*)  (\Delta_{\mathrm{dual,off}}+\Delta_{\mathrm{rQ,off}}), \label{eq:tv-oo-main-2}
\end{align}
where $(e)$ follows from definition of $C(\pi^*)$ in \cref{assum-bellman-transfer-coefficient}, $(f)$ from triangle inequality and the fact $\|\cdot\|_{1,\mu}\leq \|\cdot\|_{2,\mu}$, and $(g)$ follows from \cref{prop:erm-high-prob-bound-offline+online,prop:least-squares-generalization-bound-online+offline}.

For $(II)$, firstly we note $\E_{s_0\sim d_0}  [(Q^{k}(s_0,\pi_k(s_0)) - Q^{\pi_k}(s_0,\pi_k(s_0)))_+ ] = \E_{s_0,a_0\sim d^{\pi_k}_0}  [(Q^{k}(s_0,a_0) - Q^{\pi_k}(s_0,a_0))_+ ] $. So, following the same analysis as in $(I)$, we get 
\begin{align}
    &\E_{s_0\sim d_{0}}[(Q^{k}(s_0,\pi_k(s_0)) - Q^{\pi_k}(s_0,\pi_k(s_0)))_+ ] \leq  \sum_{h=0}^{H-1} \E_{s,a\sim d^{\pi_k}_{h}}[( Q^{k}_{h}(s,a) - \cT Q^{k}_{h+1}(s,a) )_+] \nn \\
    &\leq  \sum_{h=0}^{H-1} \E_{s,a\sim d^{\pi_k}_{h}}[( Q^{k}_{h}(s,a) - (\cT_{g^{k}_{h}} Q^{k}_{h+1})(s,a))_+ + ((\cT_{g^{k}_{h}} Q^{k}_{h+1}) (s,a) - (\cT Q^{k}_{h+1})(s,a) )_+], \label{eq:tv-part2}
\end{align} where the last inequality follows by triangle inequality for $(\cdot)_+$ operation.

Now, using \eqref{eq:tv-part2} in $(II)$ of \eqref{eq:tv-oo-main}, we have
\begin{align*}
    &\sum_{k=0}^{K-1} \E_{s_0\sim d_0}  [(Q^{k}(s_0,\pi_k(s_0)) - Q^{\pi_k}(s_0,\pi_k(s_0)))_+ ] \leq \\ & \sum_{k=0}^{K-1} \sum_{h=0}^{H-1} \E_{s,a\sim d^{\pi_k}_{h}}[( Q^{k}_{h}(s,a) - (\cT_{g^{k}_{h}} Q^{k}_{h+1})(s,a))_+] + \sum_{k=0}^{K-1} \sum_{h=0}^{H-1} \E_{s,a\sim d^{\pi_k}_{h}}[ ((\cT_{g^{k}_{h}} Q^{k}_{h+1}) (s,a) - \cT Q^{k}_{h+1}(s,a) )_+] 
    .  \label{eq:tv-oo-main-3} \numthis
\end{align*}

Recall bilinear model from \cref{assum-bilinear-model}: $\E_{d^{\pi^f}_{h}}[( f_{h} - \cT_{g_{h}} f_{h+1})_+]
=  \abs{\tri{X_h(f), W^{\mathrm{q}}_h(f,g)}}$.

Analyzing the first part of \eqref{eq:tv-oo-main-3},
the following holds with probability at least $1-\delta/2$:
\begin{align*}
    \sum_{k=0}^{K-1} \sum_{h=0}^{H-1} &\E_{ d^{\pi_k}_{h}}[( Q^{k}_{h} - \cT_{g^{k}_{h}} Q^{k}_{h+1})_+]
    \stackrel{(h)}{=}  \sum_{k=0}^{K-1} \sum_{h=0}^{H-1}  \abs{\tri{X_h(Q^{k}), W^{\mathrm{q}}_h(Q^{k},g^{k})}}   \nn \\
    &\stackrel{(i)}{\leq} \sum_{k=0}^{K-1} \sum_{h=0}^{H-1}  \|X_h(Q^{k})\|_{\Sigma_{k-1; h}^{-1}}\|W^{\mathrm{q}}_h(Q^{k},g^{k}) \|_{\Sigma_{k-1;h}} \\
    &= \sum_{k=0}^{K-1} \sum_{h=0}^{H-1} \|X_h(Q^{k})\|_{\Sigma_{k-1; h}^{-1}} \sqrt{(W^{\mathrm{q}}_h(Q^{k},g^{k}))^\top \Sigma_{k-1; h} W^{\mathrm{q}}_h(Q^{k},g^{k})} \\ 
    &= \sum_{k=0}^{K-1} \sum_{h=0}^{H-1} \|X_h(Q^{k})\|_{\Sigma_{k-1; h}^{-1}} \sqrt{(W^{\mathrm{q}}_h(Q^{k},g^{k}))^\top  (\sum_{i = 0}^{k-1} X_h(Q^{i}) X_h(Q^{i})^\top + \sigma \indic)  W^{\mathrm{q}}_h(Q^{k},g^{k})}  \\ 
    &= \sum_{k=0}^{K-1} \sum_{h=0}^{H-1} \|X_h(Q^{k})\|_{\Sigma_{k-1; h}^{-1}} \sqrt{\sum_{i = 0}^{k-1}  \abs{\tri{W^{\mathrm{q}}_h(Q^{k},g^{k}), X_h(Q^{i})}}^2 + \sigma   \norm{W^{\mathrm{q}}_h(Q^{k},g^{k})}^2}  \\ 
    &\stackrel{(j)}{\leq} \sum_{k=0}^{K-1} \sum_{h=0}^{H-1} \|X_h(Q^{k})\|_{\Sigma_{k-1; h}^{-1}} \sqrt{\sum_{i = 0}^{k-1}  \abs{\tri{W^{\mathrm{q}}_h(Q^{k},g^{k}), X_h(Q^{i})}}^2 + \sigma   B_W^2}  \\ 
    &\stackrel{(k)}{\leq} \sum_{k=0}^{K-1} \sum_{h=0}^{H-1} \|X_h(Q^{k})\|_{\Sigma_{k-1; h}^{-1}} \sqrt{\sum_{i = 0}^{k-1}  \| \cT_{g^{k}_{h}} Q^{k}_{h+1} -  Q^{k}_{h}\|_{2,d^{\pi_i}_{h}}^2 + \sigma B_W^2} \\
    &\stackrel{(l)}{\leq}  \sum_{k=0}^{K-1} \sum_{h=0}^{H-1} \|X_h(Q^{k})\|_{\Sigma_{k-1; h}^{-1}} ( \sqrt{\sum_{i = 0}^{k-1}\| \cT_{g^{k}_{h}} Q^{k}_{h+1} -  Q^{k}_{h}\|_{2,d^{\pi_i}_{h}}^2 } + \sqrt{\sigma B_W^2}) \\
    &\stackrel{(m)}{\leq}  (\Delta_{\mathrm{rQ,on}} + \sqrt{\sigma B_W^2}) \sum_{k=0}^{K-1} \sum_{h=0}^{H-1} \|X_h(Q^{k})\|_{\Sigma_{k-1; h}^{-1}} \\
    &\stackrel{(n)}{\leq}  (\Delta_{\mathrm{rQ,on}} + B_X B_W)  \sqrt{2 d H ^2 \log(1 + \frac{K}{d}) K}  ,  \label{eq:tv-oo-main-3-1} \numthis
\end{align*}
where $(h)$ follows from \cref{assum-bilinear-model}, $(i)$ from matrix Cauchy-Schwarz inequality, $(j)$ from \cref{assum-bilinear-model}, and $(k)$ by \cref{assum-bilinear-model} with $\|\cdot\|_{1,d^{\pi_i}_{h}}\leq \|\cdot\|_{2,d^{\pi_i}_{h}}$:
\begin{align*}
    |\tri{W^{\mathrm{q}}_h(Q^{k},g^{k}), X_h(Q^{i})}| &=  \E_{s,a\sim d^{\pi_i}_{h}}[( Q^{k}_{h}(s,a) - (\cT Q^{k}_{h+1})(s,a) )_+] \leq \| \cT_{g^{k}_{h}} Q^{k}_{h+1} -  Q^{k}_{h}\|_{2,d^{\pi_i}_{h}}. 
\end{align*}

Finally, $(l)$ follows by the fact $\sqrt{x+y}\leq \sqrt{x} + \sqrt{y}$, $(m)$ follows from \cref{prop:least-squares-generalization-bound-online+offline}, and $(n)$ follows from \cref{lem:elliptical}. 

Now recall bilinear model from \cref{assum-bilinear-model}: $\E_{d^{\pi^f}_{h}} [ (\cT_{g_{h}} f_{h+1} - \cT f_{h+1} )_+] =  \abs{\tri{X_h(f), W^{\mathrm{d}}_h(f,g)}}$.
Following analysis above in \eqref{eq:tv-oo-main-3-1} for the second part of \eqref{eq:tv-oo-main-3} using \cref{assum-bilinear-model} and 
\cref{prop:erm-high-prob-bound-offline+online},
the following holds with probability at least $1-\delta/2$: 
\begin{align*}
    \sum_{k=0}^{K-1} \sum_{h=0}^{H-1} &\E_{s,a\sim d^{\pi_k}_{h}} [ (\cT_{g^{k}_{h}} Q^{k}_{h+1} - \cT Q^{k}_{h+1} )_+] \leq
 (\Delta_{\mathrm{dual,on}} + B_X B_W)  \sqrt{2 d H ^2 \log(1 + \frac{K}{d}) K}  .  \label{eq:tv-oo-main-3-2} \numthis
\end{align*}

Now combining \cref{eq:tv-oo-main-3-1,eq:tv-oo-main-3-2} with \eqref{eq:tv-oo-main-3} we have 
\begin{align*}
    \sum_{k=0}^{K-1} \sum_{h=0}^{H-1} &\E_{s,a\sim d^{\pi_k}_{h}} [(Q^{k}(s_0,\pi_k(s_0)) - Q^{\pi_k}(s_0,\pi_k(s_0)))_+ ] \\&\leq
 (\Delta_{\mathrm{dual,on}} + \Delta_{\mathrm{rQ,on}} + 2B_X B_W)  \sqrt{2 d H ^2 \log(1 + \frac{K}{d}) K}  ,
\end{align*}
with probability at least $1-\delta$.
Finally, we combine this and \eqref{eq:tv-oo-main-2} with \eqref{eq:tv-oo-main}:
\begin{align*}
	0&\leq\sum_{k=0}^{K-1}\E_{s_0\sim d_0}[V^{\pi^*}_{0}(s_0) - V^{\pi_k}_{0}(s_0)] \leq K H C(\pi^*)  (\Delta_{\mathrm{dual,off}}+\Delta_{\mathrm{rQ,off}}) +\\&\hspace{3cm}(\Delta_{\mathrm{dual,on}} + \Delta_{\mathrm{rQ,on}} + 2B_X B_W)  \sqrt{2 d H ^2 \log(1 + \frac{K}{d}) K}.
\end{align*}
Let $N=m_{\mathrm{off}} + K \cdot m_{\mathrm{on}}$. Using offline bounds from \cref{prop:erm-high-prob-bound-offline+online,prop:least-squares-generalization-bound-online+offline} with $c_1=2\lambda+H$ from \cref{prop:divergence-loss-bounds}, we have:
\begin{align*}
	0&\leq\sum_{k=0}^{K-1}\E_{s_0\sim d_0}[V^{\pi^*}_{0}(s_0) - V^{\pi_k}_{0}(s_0)] \leq  K H C(\pi^*) \cdot \\& (\frac{1}{m_{\mathrm{off}}} \left( 3 \epsilon_{\cG} N + 48 (2\lambda+H) \log(2 HK|\cG||\cF|/\delta) \right) + \frac{1}{\sqrt{m_{\mathrm{off}}}} \left( \sqrt{3 \epsilon_{\cF,\mathrm{r}} N} + 8 (1+2\lambda+2H)  \sqrt{\log(2 HK |\cG||\cF|/\delta)} \right) ) \\&+(\Delta_{\mathrm{dual,on}} + \Delta_{\mathrm{rQ,on}} + 2B_X B_W)  \sqrt{2 d H ^2 \log(1 + \frac{K}{d}) K}.
\end{align*}
Now using on-policy bounds from \cref{prop:erm-high-prob-bound-offline+online,prop:least-squares-generalization-bound-online+offline} with $c_1=2\lambda+H$ from \cref{prop:divergence-loss-bounds}, we have:
\begin{align*}
	0&\leq\sum_{k=0}^{K-1}\E_{s_0\sim d_0}[V^{\pi^*}_{0}(s_0) - V^{\pi_k}_{0}(s_0)] \leq  K H C(\pi^*) \cdot \\& (\frac{1}{m_{\mathrm{off}}} \left( 3 \epsilon_{\cG} N + 48 (2\lambda+H) \log(2 HK|\cG||\cF|/\delta) \right) + \frac{1}{\sqrt{m_{\mathrm{off}}}} \left( \sqrt{3 \epsilon_{\cF,\mathrm{r}} N} + 8 (1+2\lambda+2H)  \sqrt{\log(2 HK |\cG||\cF|/\delta)} \right) ) \\&+( \frac{1}{m_{\mathrm{on}}} \left( 3 \epsilon_{\cG} N + 48 (2\lambda+H) \log(2 HK|\cG||\cF|/\delta) \right) \\&\hspace{1.5cm}+ \frac{1}{\sqrt{m_{\mathrm{on}}}} \left( \sqrt{3 \epsilon_{\cF,\mathrm{r}} N} + 8 (1+2\lambda+2H)  \sqrt{\log(2 HK |\cG||\cF|/\delta)} \right) + 2B_X B_W) \cdot   \sqrt{2 d H ^2 \log(1 + \frac{K}{d}) K}
 \end{align*}
Finally, choosing higher order terms by setting $m_{\mathrm{on}}=1$ and $m_{\mathrm{off}}=K$, we get
\begin{align*}
	0\leq&\sum_{k=0}^{K-1}\E_{s_0\sim d_0}[V^{\pi^*}_{0}(s_0) - V^{\pi_k}_{0}(s_0)] \\&\leq  \sqrt{K} H C(\pi^*) (  6 (\epsilon_{\cG}+\sqrt{\epsilon_{\cF,\mathrm{r}}}) K^2 + (8+112\lambda+64H) \log(2 HK|\cG||\cF|/\delta) ) \\
    & + (  6 (\epsilon_{\cG}+\sqrt{\epsilon_{\cF,\mathrm{r}}}) K^2 + 8+112\lambda+64H \log(2 HK|\cG||\cF|/\delta) + 2B_X B_W ) \cdot   \sqrt{2 d H ^2 \log(1 + \frac{K}{d}) K}
 \\&\leq \cO((\sqrt{\epsilon_{\cF,\mathrm{r}}} + \epsilon_{\cG})K^{5/2}H) + \cOtilde ( \max\{C(\pi^*), 1\} \sqrt{d K H^2} (\lambda+H)  \log(HK|\cF||\cG|/\delta) \sqrt{\log(1+(K/d))} ).
\end{align*}
The proof is now complete.
\end{proof}

\subsection[HyTQ Algorithm Specialized Results]{HyTQ Algorithm Specialized Results \hfill \Coffeecup\Coffeecup\Coffeecup}
\label{appen:hytq-specialized-results}

In this section we specialize our main result \cref{thm:hybrid-tv-guarantee} for different bilinear model classes and also provide an equivalent sample complexity guarantee in the offline robust RL setting.

Before we move ahead, we showcase an important property of our robust transfer coefficient $C(\pi)$ for any fixed policy. Fixing a nominal model $P^o$, the transfer coefficient considers the distribution shift w.r.t the data-generating distribution along the general function class which the algorithm uses. It is in fact smaller than the existing density ratio based concentrability assumption (\cref{tv-assum-concentra-condition}). We state this result in the following lemma.

\begin{lemma}\label{lem:transfer-coeff-concentrability}For any policy $\pi$ and offline distribution $\mu$, we have
$C(\pi) \leq  \sup_{h, s,a} {d^{\pi}_{h}(s, a)}/{\mu_{h}(s, a)}.$ 
\end{lemma}
\begin{proof} By definition in \cref{assum-bellman-transfer-coefficient}, we get that 
\begin{align*}
C(\pi) &= \max_{f\in\cF} \frac{\sum_{h=0}^{H-1} \E_{s,a\sim d^{\pi}_{h}}[(\cT f_{h+1}(s,a) -  f_{h}(s,a))_+]}{ \sum_{h=0}^{H-1} \E_{s,a\sim \mu_{h}}[|\cT f_{h+1}(s,a) -  f_{h}(s,a)|] }\\
&\leq \max_{f\in\cF} \frac{\sum_{h=0}^{H-1} \E_{s,a\sim d^{\pi}_{h}}[|\cT f_{h+1}(s,a) -  f_{h}(s,a)|]}{ \sum_{h=0}^{H-1} \E_{s,a\sim \mu_{h}}[|\cT f_{h+1}(s,a) -  f_{h}(s,a)|] } \\
&\stackrel{(a)}{\leq}  \max_{f\in\cF, h\in[H]} \frac{ \E_{s,a\sim d^{\pi}_{h}}[|\cT f_{h+1}(s,a) -  f_{h}(s,a)|]}{ \E_{s,a\sim \mu_{h}}[|\cT f_{h+1}(s,a) -  f_{h}(s,a)|] }
\leq  \sup_{h, s,a} \frac{d^{\pi}_{h}(s, a)}{\mu_{h}(s, a)} ,
\end{align*} where $(a)$ follows from the Mediant inequality. 
\end{proof}

\begin{remark}
The concentrability assumption (\cref{tv-assum-concentra-condition}) is in fact the same non-robust RL concentrability assumption  \citep{munos08a, chen2019information}. 
We make two important points here. Firstly, our transfer coefficient is larger than the  transfer coefficient \citep[Definition 1]{song2023hybrid} using the fact $\|\cdot\|_{1,\mu}\leq \|\cdot\|_{2,\mu}$. Secondly, our transfer coefficient is not directly comparable with the l2-norm version  transfer coefficient  \citep[Definition 1]{xie2021bellman}. It is an interesting open question for future research to investigate about minimax lower bound guarantees w.r.t different transfer coefficients for both non-robust and robust RL problems.
\end{remark}

We now define a bilinear model called
\textbf{Low Occupancy Complexity} \citep[Definition 4.7]{du2021bilinear}.
The nominal model $P^o$ and realizable function class $\cF$ has \textit{low occupancy complexity} w.r.t., for each $h\in[H]$, a (possibly unknown  to the learner) feature map   $\psi=(\psi_h : \ScA \rightarrow \cY)$, where $\cY$ is a Hilbert space, and w.r.t. to a (possibly unknown to the learner) map $\nu_h:\cF\mapsto \cY$ such that for all $f\in\cF$, with greedy policy $\pi^f$ w.r.t. $f$,  and $(s,a)$ we have
\begin{align}
\label{eq:Low-Occupancy-Complexity}
d^{h,\pi^f}_{P^o}(s,a) = \langle \psi_h(s, a), \nu_h(f)  \rangle.
\end{align}

We make the following assumption on the offline data-generating distribution (or policy by slight notational override for convenience).
\begin{assumption}\label{assum:low-occupancy-model-data-policy}
    Consider the Low Occupancy Complexity model (bilinear model) on $\cY=\R^d$. Let the offline data distribution $\mu = \{\mu_h\}_{h\in[H]}$ satisfy a low rank structure, i.e. 
    $\mu_{h}(s,a)  = \langle \psi_h(s, a), \nu_h(f^{\mathrm{off}})  \rangle = \sum_{i\in[d]} \psi_{h,i}(s, a) \nu_{h,i}(f^{\mathrm{off}})$, for some $f^{\mathrm{off}}\in\cF$.
\end{assumption}

Now we extend our main result \cref{thm:hybrid-tv-guarantee} in this next result specializing to the \textit{Low Occupancy Complexity} \eqref{eq:Low-Occupancy-Complexity} bilinear model.

\begin{corollary}[Cumulative Suboptimality of \cref{thm:hybrid-tv-guarantee} in Low Occupancy Complexity \eqref{eq:Low-Occupancy-Complexity} bilinear model]
Consider the Low Occupancy Complexity \eqref{eq:Low-Occupancy-Complexity} bilinear model. Let \cref{assum-bellman-transfer-coefficient,assum-realizability-with-bellman-completion,assum-dual-realizability-hybrid,fail-state-assump} hold and fix any $\delta\in(0,1)$. Then, HyTQ algorithm policies $\{\pi_k\}_{k\in[K]}$ satisfy
\begin{align*}
     \sum_{k = 0}^{K-1} (V^{\pi^*} - V^{\pi_{k}}) \leq &\cO((\sqrt{\epsilon_{\cF,\mathrm{r}}} + \epsilon_{\cG})K^{5/2}H) \\&+ \cOtilde ( \max\{C(\pi^*), 1\} \sqrt{d K H^2} (\lambda+H)  \log(HK|\cF||\cG|/\delta) \sqrt{\log(1+(K/d))} ) \\&+ \cOtilde ( \sqrt{d K H^4} \max_{f\in\cF}\| \nu_h(f)\|_2 \|\sum_{s,a} \psi_h(s, a)\|_2 \sqrt{\log(1+(K/d))} )
\end{align*} 
with probability at least  $1 - \delta$. Now, consider the offline data distribution as in \cref{assum:low-occupancy-model-data-policy} with perfect robust Bellman completeness, i.e. $\epsilon_{\cF,\mathrm{r}}=0=\epsilon_{\cG}$. We have 
$C(\pi^*) \leq \sup_{h, i\in[d]} ({\nu_{h,i}^*}/{\nu_{h,i}(f^{\mathrm{off}})}).$
\end{corollary}
\begin{proof}
Using the Low Occupancy Complexity \eqref{eq:Low-Occupancy-Complexity} bilinear model, we have $\E_{ d^{h,\pi^f}_{P^o}} [ (\cT_{g_{h}} f_{h+1} - \cT f_{h+1} )_+] =  \tri{X_h(f), W^{\mathrm{d}}_h(f,g)}$, where \[ X_h(f) = \nu_h(f), \qquad W^{\mathrm{d}}_h(f,g) = \sum_{(s,a)\in\ScA} \psi_h(s, a) ((\cT_{g_{h}} f_{h+1})(s,a) - (\cT f_{h+1})(s,a) )_+. \]
We also have $\E_{d^{h,\pi^f}_{P^o}}[( f_{h} - \cT_{g_{h}} f_{h+1})_+]
=  {\tri{X_h(f), W^{\mathrm{q}}_h(f,g)}}$, where \[ \qquad W^{\mathrm{q}}_h(f,g) = \sum_{(s,a)\in\ScA} \psi_h(s, a) ( f_{h}(s,a) - (\cT_{g_{h}} f_{h+1})(s,a))_+. \]

Furthermore, we set $B_X = \max_{f\in\cF}\| \nu_h(f)\|_2$. Since $\cF$ is realizable and $\cT_g$ is complete, we set $B_W = H \|\sum_{s,a} \psi_h(s, a)\|_2$. Then the result directly follows by \cref{thm:hybrid-tv-guarantee}.

For the second statement, first note that the occupancy $d^{\pi^*}_{h}$ is low-rank as well since we assume perfect Bellman completeness.
Following the proof of \cref{lem:transfer-coeff-concentrability} we get
\begin{align*}
C(\pi^*) &= \max_{f\in\cF} \frac{\sum_{h=0}^{H-1} \E_{s,a\sim d^{\pi^*}_{h}}[(\cT f_{h+1}(s,a) -  f_{h}(s,a))_+]}{ \sum_{h=0}^{H-1} \E_{s,a\sim \mu_{h}}[|\cT f_{h+1}(s,a) -  f_{h}(s,a)|] }\\
&\leq \max_{f\in\cF} \frac{\sum_{h=0}^{H-1} \E_{s,a\sim d^{\pi^*}_{h}}[|\cT f_{h+1}(s,a) -  f_{h}(s,a)|]}{ \sum_{h=0}^{H-1} \E_{s,a\sim \mu_{h}}[|\cT f_{h+1}(s,a) -  f_{h}(s,a)|] } \\
&\stackrel{(a)}{\leq}  \max_{f\in\cF, h\in[H]} \frac{ \E_{s,a\sim d^{\pi}_{h}}[|\cT f_{h+1}(s,a) -  f_{h}(s,a)|]}{ \E_{s,a\sim \mu_{h}}[|\cT f_{h+1}(s,a) -  f_{h}(s,a)|] }\\
&\leq  \sup_{h, s,a} \frac{d^{\pi}_{h}(s, a)}{\mu_{h}(s, a)} \stackrel{(b)}{\leq}  \sup_{h, i\in[d]} \frac{\nu_{h,i}^*}{\nu_{h,i}(f^{\mathrm{off}})},
\end{align*} where $(a),(b)$ follows from the Mediant inequality. 
This completes the proof.
\end{proof}

We now define a bilinear model called
\textbf{Low-rank Feature Selection Model} \citep[Definition A.1]{du2021bilinear}. 
The nominal model $P^o$ is a \textit{low-rank feature selection model} if it satisfies $P^o_{h,s,a}(s') = \langle \theta_h(s,a), \psi_h(s') \rangle$, for each $h\in[H]$ and all $(s,a,s')$, with a (possibly unknown to the learner) map   $\theta=(\theta_h : \ScA \rightarrow \cY)$ and a (possibly unknown  to the learner) map $\psi_h:\cS\mapsto \cY$, where $\cY$ is a Hilbert space.

This model specializes to the \textit{kernel MDP model} when the map $\theta$ is known to the learner \citep[Definition 30]{jin2021bellman}. 
This model also specializes to the \textit{low-rank MDP model} when $\cY=\R^d$ \citep[Assumption 1]{huang2023reinforcement} and furthermore to \textit{linear MDP model} when the map $\theta$ is also known to the learner \citep[Definition A.4]{du2021bilinear}.

We make the following assumption on the offline data-generating distribution (or policy by slight notational override for convenience).
\begin{assumption}\label{assum:low-rank-model-data-policy}
    Consider the Low-rank MDP Model (bilinear model). Let the offline data distribution $\mu = \{\mu_h\}_{h\in[H]}$ satisfy $\max_{h,s,a} {\pi^*_{h}(a|s)}/{\mu_h(a|s)} \leq \alpha$ and  suppose that $\mu$ is induced by the nominal model, i.e. $\mu_0(s) = d_0(s)$ (starting state distribution)  and $\mu_h(s) = \E_{s', a' \sim \mu_{h-1}} P^o_{h-1, s', a'}(s)$ for any $h \geq 1$. Furthermore, suppose that $\mu$ satisfies that the feature covariance matrix  $\Sigma_{\mu_{h-1}, \theta} = \E_{s,a \sim \mu_{h-1}}[\theta_{h}(s,a)\theta_{h}(s,a)^\top]$ is invertible  for all $h\in[H]$ and $\E_{s,a \sim \mu_{h}}[|\cT f_{h+1}(s,a) -  f_{h}(s,a)|] \geq 1$ for at least one $h\in[H]$ and all $f\in\cF$.
\end{assumption}

Now we extend our main result \cref{thm:hybrid-tv-guarantee} in this next result specializing to the \textit{Low-rank Feature Selection Model} bilinear model.

\begin{corollary}[Cumulative Suboptimality of \cref{thm:hybrid-tv-guarantee} in Low-rank Feature Selection Model (bilinear model)]
Consider the Low-rank Feature Selection Model (bilinear model). Let \cref{assum-bellman-transfer-coefficient,assum-realizability-with-bellman-completion,assum-dual-realizability-hybrid,fail-state-assump} hold and fix any $\delta\in(0,1)$. 
Then, HyTQ algorithm policies $\{\pi_k\}_{k\in[K]}$ satisfy
\begin{align*}
     \sum_{k = 0}^{K-1} (V^{\pi^*} - V^{\pi_{k}}) \leq &\cO((\sqrt{\epsilon_{\cF,\mathrm{r}}} + \epsilon_{\cG})K^{5/2}H) \\&+ \cOtilde ( \max\{C(\pi^*), 1\} \sqrt{d K H^2} (\lambda+H)  \log(HK|\cF||\cG|/\delta) \sqrt{\log(1+(K/d))} ) \\&+ \cOtilde ( \sqrt{d K H^4} \| \sum_{s,a} \theta_h(s,a) \|_2 \|\sum_{s} \psi_h(s)\|_2 \sqrt{\log(1+(K/d))} )
\end{align*} 
with probability at least  $1 - \delta$. Now, consider the offline data distribution as in \cref{assum:low-rank-model-data-policy} with a low-rank MDP model. We have \begin{align*}
C(\pi^*) &\leq \sqrt{2 \alpha H} \sum_{h=1}^H \E_{s, a \sim d^{h-1,\pi^*}_{P^o}} \norm{\theta_h(s, a)}_{\Sigma_{\mu_{h-1}, \theta}^{-1}} + \sqrt{\alpha}. 
\end{align*} 
\end{corollary}
\begin{proof}
We first begin with establishing a Q-value-dependent linearity property for the state-action-visitation measure $d^{h,\pi^f}_{P^o}(s,a)$. To do this, we adapt the proof of \citet[Lemma 17]{huang2023reinforcement} here.
We start by writing the state-visitation measure by recalling \cref{eq:hytq-state-dist-Po} here:
\begin{align*}
    d^{h,\pi^f}_{P^o}(s_h) &= \sum_{(s,a)\in\ScA}  P^o_{h,s,a}(s_h) \pi^f_{h-1}(a|s) d^{h-1,\pi^f}_{P^o}(s) \\
    &\stackrel{(a)}{=} \sum_{(s,a)\in\ScA}  \langle \theta_h(s,a), \psi_h(s_h) \rangle \pi^f_{h-1}(a|s) d^{h-1,\pi^f}_{P^o}(s) \\
    &=   \langle \sum_{(s,a)\in\ScA} \theta_h(s,a) \pi^f_{h-1}(a|s) d^{h-1,\pi^f}_{P^o}(s), \psi_h(s_h) \rangle  = \langle \psi_h(s_h), \nu_{h,\pi^f}(f)  \rangle,
\end{align*}
where $(a)$ follows by the low-rank feature selection model definition, and the last equality follows by taking a functional $\nu_{h,\pi^f}(f) = \sum_{s,a} \theta_h(s,a) \pi^f_{h-1}(a|s) d^{h-1,\pi^f}_{P^o}(s)$.
Since we consider the finite action space with possibly large state space setting for our results, the state-action visitation measure for the deterministic non-stationary  policy $\pi^f$ is now given by $d^{h,\pi^f}_{P^o}(s_h,a_h)  = \langle \psi'_{h,\pi^f}(s_h,a_h), \nu_{h,\pi^f}(f)  \rangle$ with $\psi'_{h,\pi^f}(s_h,a_h)=C\psi_h(s_h) 1\{a_h=\pi^f_h(s)\}$ for features $\psi'_{h,\pi^f}:\ScA \to \cY$. Here $C>0$ is a normalizing constant such that the state-action visitation  measure is a probability measure.

We now have $\E_{ d^{h,\pi^f}_{P^o}} [ (\cT_{g_{h}} f_{h+1} - \cT f_{h+1} )_+] =  \tri{X_h(f), W^{\mathrm{d}}_h(f,g)}$, where \[ X_h(f) = \nu_{h,\pi^f}(f), \qquad W^{\mathrm{d}}_h(f,g) = \sum_{(s,a)\in\ScA} \psi'_{h,\pi^f}(s, a) ((\cT_{g_{h}} f_{h+1})(s,a) - (\cT f_{h+1})(s,a) )_+. \]
We also have $\E_{d^{h,\pi^f}_{P^o}}[( f_{h} - \cT_{g_{h}} f_{h+1})_+]
=  {\tri{X_h(f), W^{\mathrm{q}}_h(f,g)}}$, where \[ \qquad W^{\mathrm{q}}_h(f,g) = \sum_{(s,a)\in\ScA} \psi'_{h,\pi^f}(s, a) ( f_{h}(s,a) - (\cT_{g_{h}} f_{h+1})(s,a))_+. \]

Furthermore, we set \[\max_{f\in\cF}\| \nu_h(f)\|_2 = \max_{f\in\cF} \| \sum_{s,a} \theta_h(s,a) \pi^f(a|s) d^{h-1,\pi^f}_{P^o}(s) \|_2 \leq \| \sum_{s,a} \theta_h(s,a) \|_2 = B_X. \] Since $\cF$ is realizable and $\cT_g$ is complete for all $g\in\cG$, we set \[ H \|\sum_{s,a} \psi'_{h,\pi^f}(s,a)\|_2 = H C \|\sum_{s,a} \psi_h(s) 1\{a=\pi^f_h(s)\}\|_2 \leq H C \|\sum_{s} \psi_h(s)\|_2 =  B_W . \] Then the first result directly follows by \cref{thm:hybrid-tv-guarantee}. Following the proof of \citet[Lemma 13]{song2023hybrid} for our transfer coefficient $C(\pi^*)$, with the facts $(x-y)^2\leq |x-y| |x+y|$ for $x,y\geq 0$ and $\|f_h\|_\infty\leq H$ for all $h\in[H]$, the last statement for $C(\pi^*)$ follows. This completes the proof.
\end{proof}

Now we extend our main result \cref{thm:hybrid-tv-guarantee} in this next result to showcase sample complexity for comparisons with offline+online RL setting. 

\begin{corollary}[Offline+Online RL Sample Complexity of the HyTQ algorithm] \label{cor:hytq-sample-complexity-rl}
Let \cref{assum-bellman-transfer-coefficient,assum-realizability-with-bellman-completion,assum-dual-realizability-hybrid,assum-bilinear-model,fail-state-assump} hold. Fix any $\delta\in(0,1)$ and any $\epsilon>0$, and let $N_{\mathrm{tot}} $ be the total number of sample tuples used in HyTQ algorithm. Then, the uniform policy $\widehat{\pi}$ (uniform convex combination) of HyTQ algorithm policies $\{\pi_k\}_{k\in[K]}$ satisfy, 
with probability at least  $1 - \delta$,
\begin{align*}
     V^{\pi^*} - V^{\widehat{\pi}} \leq \epsilon, \quad\text{ if } N \geq N_{\mathrm{tot}} = \cOtilde ( \frac{\max\{(C(\pi^*))^2, 1\} d H^3 (\lambda+H)^2}{\epsilon^2}  \log^2(H|\cF||\cG|/\delta) ).
\end{align*}
\end{corollary}

\begin{proof} This proof is straightforward from the \cref{thm:hybrid-tv-guarantee} using a standard online-to-batch conversion \citep[Theorem 14.8 \& Chapter 21]{shalev2014understanding}.  Define the policy $\widehat{\pi} = \text{Uniform}\{\pi_0, \dots, \pi_{K-1}\}$. From \cref{thm:hybrid-tv-guarantee}, we get 
\begin{align*}
&0\leq\E_{s_0\sim d_0}[V^{\pi^*}_{0}(s_0) - V^{\widehat{\pi}}_{0}(s_0)] = \frac{1}{K} \sum_{k=0}^{K-1} \E_{s_0\sim d_0}[V^{\pi^*}_{0}(s_0) - V^{\pi_k}_{0}(s_0)] \\
&\leq \cO((\sqrt{\epsilon_{\cF,\mathrm{r}}} + \epsilon_{\cG})K^{3/2}H) + \cOtilde ( \max\{C(\pi^*), 1\} \sqrt{d H^2/ K} (\lambda+H)  \log(HK|\cF||\cG|/\delta) \sqrt{\log(1+(K/d))} ). 
\end{align*} 

We recall that our algorithm uses $m_{\mathrm{off}} H $ number of offline samples and $m_{\mathrm{on}} H K $ number of on-policy samples in the datasets $\{\cD^\mu_h, \cD^{0}_h, \cdots, \cD^{K-1}_h\} $ for all $h\in[H]$. Since we set   $m_{\mathrm{on}}=1$ and $m_{\mathrm{off}}=K$, the total number of offline and on-policy samples is $2HK$.

Fix any $\epsilon>0$.
For approximations $\epsilon_{\cF,\mathrm{r}}, \epsilon_{\cG}$, we first assume there exists $K_1=\cOtilde ( H^4)$ such that $\cO((\sqrt{\epsilon_{\cF,\mathrm{r}}} + \epsilon_{\cG})K^{3/2}H) \leq \epsilon/2$ for all $K\geq K_1$.
Let  \[ K_2 = \cOtilde ( \frac{\max\{(C(\pi^*))^2, 1\} d H^2 (\lambda+H)^2}{\epsilon^2}  \log^2(H|\cF||\cG|/\delta) ). \]
Then, for $K\geq K_1+K_2$, we have $\E_{s_0\sim d_0}[V^{\pi^*}_{0}(s_0) - V^{\widehat{\pi}}_{0}(s_0)] \leq \epsilon $ with probability at least $1-\delta$. So, the total number of samples is at least $N_{\mathrm{tot}}$:
$$N_{\mathrm{tot}} = 2H(K_1 + K_2) = \cOtilde ( \frac{\max\{(C(\pi^*))^2, 1\} d H^3 (\lambda+H)^2}{\epsilon^2}  \log^2(H|\cF||\cG|/\delta) ). $$
This completes the proof.
\end{proof}